\documentclass[lettersize,journal]{IEEEtran}
\usepackage{amsmath,amsfonts}
\usepackage{algorithmic}
\usepackage{algorithm}
\usepackage{array}
\usepackage{textcomp}
\usepackage{stfloats}
\usepackage{url}
\usepackage{verbatim}
\usepackage{graphicx}
\usepackage{cite}
\usepackage{color}
\usepackage{microtype}
\usepackage{graphicx}
\usepackage{booktabs}
\usepackage{hyperref}
\numberwithin{equation}{section}
\usepackage{amssymb}
\usepackage{mathtools}
\usepackage{amsthm}
\usepackage{caption}
\usepackage{subcaption}
\usepackage[capitalize,noabbrev]{cleveref}

\theoremstyle{plain}
\newtheorem{theorem}{Theorem}

\newtheorem{lemma}{Lemma}

\theoremstyle{definition}

\theoremstyle{remark}

\DeclareMathOperator*{\argmin}{arg\,min}

\usepackage[textsize=tiny]{todonotes}

\begin{document}

%\title{Purification Of One-Hidden-Layer Contaminated Convolutional Neural Networks Via Robust Recovery}

%\title{Robust Recovery for Contaminated Convolutional Neural Networks: An Approach with Theoretical Recovery Guarantee in One-Hidden-Layer Case}

\title{Purification Of Contaminated Convolutional Neural Networks Via Robust Recovery: An Approach with Theoretical Guarantee in One-Hidden-Layer Case}

\author{Hanxiao~Lu,~\IEEEmembership{Student Member,~IEEE,}
        Zeyu~Huang,~\IEEEmembership{Student Member,~IEEE,}
        and~Ren~Wang,~\IEEEmembership{Member,~IEEE}% <-this % stops a space
\IEEEcompsocitemizethanks{\IEEEcompsocthanksitem Ren Wang is with the Department
of Electrical and Computer Engineering, Illinois Institute of Technology, Chicago,
IL, 60616.%\protect\\
%E-mail: rwang74@iit.edu
\IEEEcompsocthanksitem Hanxiao Lu and Zeyu Huang are research interns at the Trustworthy and Intelligent Machine Learning Research Lab in the Department
of Electrical and Computer Engineering, Illinois Institute of Technology, Chicago,
IL, 60616.}% <-this % stops an unwanted space
\thanks{The first two authors contributed equally to this paper.}
\thanks{Corresponding author: Ren Wang. E-mail: rwang74@iit.edu}
\thanks{Partial and preliminary results appeared in \cite{lu2023enhancing}.}
%\thanks{This work was supported by the National Science Foundation (NSF) under Grant 2246157.}
%\thanks{Under review at IEEE TPAMI.}
}

% %\thanks{The first two authors contributed equally to this work. This work was done when Hanxiao Lu and Zeyu Huang were research interns in the Trustworthy and Intelligent Machine Learning Research Group under the supervision of Ren Wang. (Corresponding author: Ren Wang)}
% }

\markboth{Journal of \LaTeX\ Class Files,~Vol.~14, No.~8, August~2021}%
{Shell \MakeLowercase{\textit{et al.}}: A Sample Article Using IEEEtran.cls for IEEE Journals}


\maketitle

\begin{abstract}
Convolutional neural networks (CNNs), one of the key architectures of deep learning models, have achieved superior performance on many machine learning tasks such as image classification, video recognition, and power systems. Despite their success, CNNs can be easily contaminated by natural noises and artificially injected noises such as backdoor attacks. In this paper, we propose a robust recovery method to remove the noise from the potentially contaminated CNNs and provide an exact recovery guarantee on one-hidden-layer non-overlapping CNNs with the rectified linear unit (ReLU) activation function. Our theoretical results show that both CNNs' weights and biases can be exactly recovered under the overparameterization setting with some mild assumptions. The experimental results demonstrate the correctness of the proofs and the effectiveness of the method in both the synthetic environment and the practical neural network setting. Our results also indicate that the proposed method can be extended to multiple-layer CNNs and potentially serve as a defense strategy against backdoor attacks. %We further extend the method to the backdoor attack elimination application and show that the proposed method can serve as a defense strategy against malicious model poisoning. 
\end{abstract}


\section{Introduction}\label{sec: intro}
Deep neural networks (DNNs), models with thousands or millions of parameters, are used in deep learning (DL) to learn patterns from inputs, outperforming traditional techniques that use human-crafted models. Among all types of DNNs, convolutional neural networks (CNNs) achieve state-of-the-art performances over other types of architectures on tasks such as image classification \cite{zhang2022unsupervised}, action recognition \cite{cheron2015p}, and fault detection in power systems \cite{li2023physics}. CNNs also require fewer coefficients than fully connected neural networks due to shared weights, and they can better extract local features with convolution operations. %However, CNN models run the risk of being trained in untrusted environments, which can make them vulnerable to easy contamination. Such risk can come from different real-world applications. For example, model updates and transmissions commonly happen in the collaborative training process (e.g., federal learning) \cite{ma2022communication}, which can bring additional noises. %model transmissions commonly happen in the collaborative training process (e.g., federal learning) \cite{ma2022communication} and model download/upload \cite{jankowski2022airnet}. 
%In order to improve CNNs' running/transmission/storage efficiency and fit the finite system precision, CNNs' parameter resolutions are often reduced by quantization or truncation operations, which themselves can be viewed as injected noises \cite{young2021transform}. 
However, CNN models are susceptible to contamination when trained in untrusted environments, a risk exacerbated by various real-world applications. For instance, during the collaborative training process, such as federated learning, model updates and transmissions frequently introduce additional noise \cite{ma2022communication}. To enhance efficiency in running, transmitting, and storing CNNs within the constraints of system precision, their parameter resolutions are often lowered through quantization or truncation, effectively injecting noise \cite{young2021transform}. Recently, studies on training-phase poisoning attacks, like backdoor attacks, have shown that contaminating just a small fraction of the training data is enough to lead to ``noisy'' CNNs with inaccurate predictions in downstream tasks \cite{gu2019badnets,wang2020practical}. Therefore, we need techniques to purify CNNs.

%There are artificially injected noises by model trainers such as backdoor attack. 

%Nowadays, models are often transmitted over noisy wireless channels, which will bring additional noise to CNNs' parameters \cite{gunduz2019machine}. Model transmissions commonly happen in the collaborative training process (e.g., federal learning) \cite{ma2022communication} and model download/upload \cite{jankowski2022airnet}. Another real-world case where noise appears on CNN models is analog computing, which has been highly pushed for accelerating neural networks \cite{elbtity2021memory}. According to \cite{shen2017deep}, the accuracy drops 15\% when implementing a shallow neural network in an analog device. In addition to these naturally emerged noises, 

There are many works focusing on robust data recovery \cite{zhang2019robust,wang2018data,wang2020tensor} and robust regression for linear models \cite{dalalyan2019outlier,suggala2019adaptive}. Few studies have explored how to purify neural networks to reduce the negative impact of unexpected noises. To remove Gaussian noises from noisy neural networks, a Bayesian estimation-based denoiser is proposed \cite{shao2021denoising}. The recovery error discussed in this work is only valid when the inputs are uniform, and the weights and noises follow Gaussian distributions. Recent model purification work has only considered the recovery of a one-hidden-layer fully connected neural network \cite{gao2020model}. In this paper, we consider the theoretical recovery of a one-hidden-layer convolutional neural network contaminated by noises from arbitrary distributions, including backdoor pollutions, and empirically extend it to multi-layer scenarios. Noting that existing training-phase poisoning defenses are mainly based on detection \cite{wang2020practical,pal2024backdoor} and fine-tuning \cite{pal2023towards,zhu2021clear}, our proposed method can detoxify CNNs under training-phase poisoning attacks. Our approach can directly eliminate the impact of poisoning from the model's parameters and requires only a limited amount of benign data without any label information.

%we further extend the proposed recovery method to detoxify CNNs under training-phase poisoning attacks. Our approach can directly eliminate the impact of poisoning from the model's parameters and only necessitates a limited amount of benign data without any label information.

The contributions of this paper can be summarized as follows:
\begin{itemize}
    \item The paper introduces a robust recovery method designed to cleanse CNNs of both natural and artificially injected noises. This method provides theoretical recovery guarantees for one-hidden-layer CNNs using the ReLU activation function under under an overparameterization scenario.
    \item It demonstrates the practical application of the method on CNNs trained on poisoned data, offering a direct technique to purify the networks.
    \item The method is empirically tested on standard datasets like MNIST and CIFAR-10, showing that it maintains the same level of accuracy as clean CNNs while reducing the success rate of attacks. The method requirs minimal clean data, potentially from limited amount of unlabeled benign data outside the training set.
    \item We empirically show that the method works on CNNs with more than one hidden layer.
\end{itemize}

% \subsection{Our contributions} 
% This work theoretically proves that by properly selecting design matrices in the proposed robust recovery method, all CNN parameters can be purified to ground-truth parameters. In addition, we quantify the relationship between learning accuracy and noise level. Experiments on synthetic and MNIST data demonstrate the theoretical correctness and method effectiveness, in addition to the novel contributions to the theoretical analysis of CNNs. Furthermore, we also leverage the proposed method to purify CNNs trained on poisoned data. Unlike previous works that primarily focused on detection and fine-tuning, this work aims to directly remove the noisy weights corresponding to the poisoning effect. We only require a small amount of clean data, which could be selected from resources other than the training set. Our experimental results show that CNNs can remain at the same level of standard accuracy with low attack success rates after the recovery.

The remainder of the paper is structured as follows. We first provide all the notations used in this paper. The problem formulation is introduced in Section~\ref{sec: prob_form}. Section~\ref{sec: alg} presents the algorithm, and Section~\ref{sec: theory} summarizes the major theoretical results. Section~\ref{sec: experiments} presents all the experimental results, and Section~\ref{sec: conclusion} concludes the paper. The supplementary materials contain all of the proofs.

%Recent research on training-phase poisoning attacks demonstrates that even a modest amount of training data contamination is sufficient to produce noisy CNNs with erroneous predictions for downstream tasks.

\begin{figure}[ht]
\vskip 0.2in
\begin{center}
\centerline{\includegraphics[trim=0 0 0 0,clip,width=.49\textwidth]{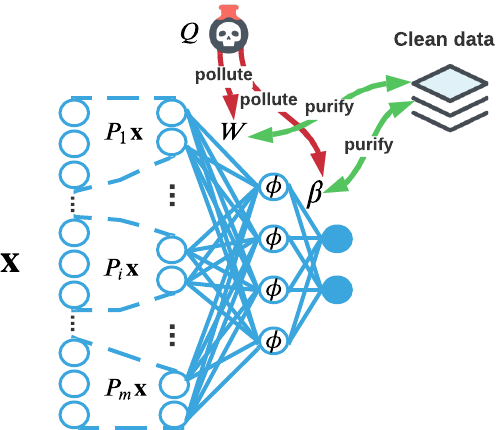}}
\caption{Conceptual diagram illustrating the proposed framework. Hidden-layer weights $W$ and output layer weights $\beta$ of a convolutional neural network (CNN) are contaminated by noises. The proposed CNN purification method can remove noises from contaminated weights.}
\label{fig: overview}
\end{center}
\vskip -0.2in
\end{figure}

\section{Problem formulation}
\label{sec: prob_form}
In this section, we first give an overview of the CNN purification problem and then provide details of the CNN architecture, and the contamination model studied in this work. We consider the general scenario that a CNN is trained on $n$ inputs $\{\mathbf{x}_s\}_{s=1}^n \in \mathbb{R}^d$ with corresponding ground truth $y_s$, and the parameters are contaminated by some random noise $z$. $z$ is assumed to be independent of input data and is generated from an arbitrary distribution, which can result from either post-training phase perturbations or poisoned inputs. Our goal is to purify contaminated CNN parameters by leveraging the proposed robust recovery method, which avoids retraining the model from scratch.

\subsection{CNN model }
As illustrated in Figure~\ref{fig: overview}, this work studies the one-hidden-layer CNN architecture:
\begin{align}\tag{1}
    \begin{array}{ll}
\hat{y_{s}} = \sum_{j=1}^p\sum_{i=1}^m \beta_j \psi (W_j^T P_i \mathbf x_{s})
    \end{array},
    \label{eq: one_layer_cnn}
\end{align}
where $\mathbf x_{s} \in \mathbb{R}^d$ is the input and the scalar $\hat{y_{s}}$ is its prediction. Following the same setting in previous theoretical works on CNNs \cite{zhang2020improved,zhong2017learning}, we consider CNN with $m$ non-overlapping input patches. $P_i \mathbf{x_{s}} \in \mathbb{R}^k$ is the $i$-th patch $(i=1,2, \cdots, m)$ of input $\mathbf{x_{s}}$, which is separated by $m$ matrices $\{P_i\}_{i=1}^m \in \mathbb{R}^{k \times d}$ defined as follows.
$$
P_i=[\underbrace{\mathbf{0}_{k\times k(i-1)}}_\text{All zero matrix}~~ \underbrace{I_{k}}_\text{Identity matrix $\in \mathbb{R}^{k\times k}$}~~ \underbrace{\mathbf{0}_{k\times k(m-i)}}_\text{All zero matrix}]
$$
Note that the non-overlapping setting forces $\{P_i\}_{i=1}^m$ independent of each other and therefore simplifies our proofs. $W = [W_1, W_2, \cdots, W_p] \in \mathbb{R}^{k \times p}$ denotes the hidden layer weights with each column $W_j \in \mathbb{R}^{k}$ representing the $j$-th kernel weights. The Rectified Linear Unit (ReLU) operation $\psi$ is the most commonly used activation function that transforms data $t$ into $\text{ReLU}(\cdot) = \max(0,\cdot)$. $\beta \in \mathbb{R}^{p}$ denotes the output layer weights and $\beta_j$ is its $j$-th entry. In this paper, we consider an overparameterization setting, where $p, k \gg n$.

\subsection{Corrupted model}
Here we define the contamination model for $W$ and $\beta$.\begin{align}\tag{2}
    \begin{array}{ll}
\Theta_{j} = W_{j} +  z_{W_j}
    \end{array},
    \label{eq: corrputed model hidden layer}
\end{align}
\begin{align}\tag{3}
    \begin{array}{ll}
\eta = \beta + z_{\beta}
    \end{array},
    \label{eq: corrputed model output layer}
\end{align}
where $\Theta$ and $\eta$ are contaminated parameters of CNN's hidden layer and output layer, respectively. The vectors $z_{W_j} \in \mathbb{R}^{k}, z_{\beta} \in \mathbb{R}^{p}$ are noise vectors with each entry $[z_{W_j}]_i$ ($[z_{\beta}]_i$) generated from an arbitrary distribution $Q_i$ with fixed probability $\epsilon$, which is between 0 and 1. 

In the post-training phase poisoning scenario discussed in Section~\ref{sec: intro}, our contamination model describes the additional noises added to clean weights $W$ and $\beta$. In the training phase poisoning scenario we considered in this work, additional noises are injected through manipulated training data. For example, a backdoor attack targeting neural networks is an adversarial strategy designed to undermine the reliability of a machine learning model by secretly embedding a malicious pattern or trigger during its training phase \cite{gu2019badnets,wang2020practical}. This subtle trigger is often undetectable to humans but can force the model to generate erroneous or manipulated outputs when encountered in subsequent inputs. Attackers typically execute a backdoor attack by contaminating the training data with a particular pattern or feature and a corresponding target label. As the model trains, it learns to associate the pattern with the target label, effectively embedding the backdoor. Once the model is in use, the attacker can exploit this backdoor by incorporating the trigger into the input data, leading the model to produce the intended, manipulated output. Besides poisoning from training data, attackers can even directly manipulate CNN parameters to inject backdoors \cite{hong2022handcrafted}. In all the above attack settings, contaminated CNNs can be viewed as benign models with additional poisoning parameters. 

%Existing poisoning mitigation defenses are mainly based on detection \cite{wang2020practical} and fine-tuning \cite{pal2023towards,zhu2021clear}. To the best of our knowledge, no work has considered directly removing the poisoning effect from the model parameter. Our method only requires a small amount of benign data without label information.

%Here we apply our method on poisoning attack mitigation. We consider the backdoor attack, which is the most harmful attack category in poisoning attacks \cite{gu2019badnets,wang2020practical}. 

%According to recent research \cite{wang2020practical}, some CNN weights contain a portion of poisoning information, which our contamination model can also characterize.

In the following sections, we introduce a method that can recover parameters $W$ and $\beta$ from $\Theta$ and $\eta$ with theoretical guarantees. 

\section{Purification of One-hidden-Layer CNN Algorithm }\label{sec: alg}
\subsection{CNN model training} 

Before introducing the CNN recovery optimization and algorithm, we need to specify the process of obtaining $W$ and $\beta$. In our setting, the one-hidden-layer CNN is trained by the traditional gradient descent algorithm, which is shown in Algorithm~\ref{alg: gradent descent of CNN}. $X \in \mathbb{R}^{d \times n}$ is the matrix format of the training examples. $W(0),\beta(0)$ are initializations of hidden and output layers' weights. They are initialized randomly following Gaussian distributions $\mathcal{N}(0,k^{-1}I_{k})$ and $\mathcal{N}(0,1)$ respectively. $\gamma$ and $\frac{\gamma}{k}$ are learning rates indicating step sizes of gradient descents. With the purpose of easier computation of the partial derivative of loss function $\mathcal{L}$ with respect to $\beta$ and $W$, we use the squared error empirical risk
$$\mathcal{L}(\beta,W) = \frac{1}{2}\frac{1}{n}\sum_{s=1}^{n}(y_{s}-\frac{1}{\sqrt{p}}\sum_{j=1}^{p}\sum_{i=1}^{m}\beta_{j}\psi(W_{j}^{T}P_{i}\mathbf{x_{s}}))^2$$       
that quantifies the prediction errors of the learned CNN. $\frac{1}{\sqrt{p}}$ is used for simplifying our proofs. Note that in the post-training phase poisoning scenario, $W(t_{max})$ and $\beta(t_{max})$ are the ground truth we want to extract from observations $\Theta$ and $\eta$. We will introduce the details of the training phase poisoning scenario in Section~\ref{sec: experiments}. We use the following $\ell_1$ norm-based robust recovery optimization method to achieve accurate estimations.

\begin{algorithm}[ht]
   \caption{CNN Model Training}
   \label{alg: gradent descent of CNN}
\begin{algorithmic}
   \STATE {\bfseries Input:} Data $ (y,X) $, maximum number of iterations $t_{max}$
   \STATE {\bfseries Output:} $W(t_{max})$ and $\beta(t_{max})$
   \STATE Initialize $ W_{j}(0)  \sim \mathcal{N}(0,k^{-1}I_{k}) $ and $\beta_{j}(0) \sim \mathcal{N}(0,1)$ independently for all $j \in [p]$.
   \FOR{$t=0$ {\bfseries to} $t_{max}$}
   \FOR{$j=1$ {\bfseries to} $p$}
   \STATE $\beta_{j}(t) = \beta_{j}(t-1) - \gamma \frac{\partial \mathcal{L}(\beta(t-1),W(t-1))}{\partial \beta_{j}(t-1)}$
   \ENDFOR
   \FOR{$j=1$ {\bfseries to} $p$}
   \STATE $W_{j}(t) = W_{j}(t-1) - \frac{\gamma}{k} \frac{\partial \mathcal{L}(\beta(t),W(t-1))}{\partial W_{j}(t-1)}$
   \ENDFOR
   \ENDFOR
   \STATE {\bfseries Output:} $\beta(t_{max})$ and $W(t_{max})$
\end{algorithmic}
\end{algorithm}

\begin{algorithm}[ht]
   \caption{Purification of One-hidden-Layer CNN}
   \label{alg: model repair}
\begin{algorithmic}
   \STATE {\bfseries Input:} Contaminated model $ (\eta, \Theta) $, design matrix $A_{W},A_{\beta}$, and parameter initialization $ \beta(0),W(0)$.
   \STATE {\bfseries Output:}The purified parameters $\widetilde{\beta}$ and $\widetilde{W}$
   \FOR{$j=1$ {\bfseries to} $p$}
   \STATE $\widetilde{u}_{j} = \argmin\limits_{u} \|\Theta_{j} - W_{j}(0)- A_{W}^{T}u_{j}\|_{1}$
   \STATE $\widetilde{W}_{j} = W_{j}(0) + A_{W}^{T}\widetilde{u_{j}} $   
   \ENDFOR
   \STATE $ \widetilde{v} = \argmin\limits_{v} \|\eta - \beta(0)- A_{\beta}^{T}v\|_{1} $
   \STATE $\widetilde{\beta}= \beta(0) + A_{\beta}^{T}\widetilde{v}$
   \STATE {\bfseries Output:} $\widetilde{W}$ and $\widetilde{\beta}$
\end{algorithmic}
\end{algorithm}

\subsection{Robust recovery for CNN purification}

The $\ell_1$ norm-based recovery optimizations for $W$ and $\beta$ are defined as 
 \begin{align}\tag{4}
    \begin{array}{ll}
\widetilde{u}_{j} = \argmin\limits_{u} \|\Theta_{j} - W_{j}(0)- A_{W}^{T}u_{j}\|_{1}
    \end{array},
    \label{eq: recovery optimization W}
\end{align}
\begin{align}\tag{5}
    \begin{array}{ll}
\widetilde{v} = \argmin\limits_{v} \|\eta - \beta(0)- A_{\beta}^{T}v\|_{1}
    \end{array},
    \label{eq: recovery optimization beta}
\end{align}
where $\widetilde{u}_{j}, j \in [p], \widetilde{v}$ are the optimal estimations of the models' coefficients of the two optimization problems. $A_{W}$ is the design matrix for purifying $W$: 
\begin{align}\tag{6}
    \begin{array}{ll}
A_{W}=[P_{1}X,P_{2}X...,P_{m}X]
    \end{array},
    \label{eq: design matrix A_W}
\end{align}

$A_{\beta}$ is the design matrix for recovering $\beta$: 
\begin{align}\tag{7}
    \begin{array}{ll}
A_{\beta}=\left[\sum_{i=1}^{m} \psi(W^{T}P_{i}x_1),...,\sum_{i=1}^{m} \psi(W^{T}P_{i}x_n)\right]
    \end{array},
    \label{eq: design matrix A_beta}
\end{align}

%By making use of Ordinary Least Square (OLS), we deduce that the ground truth model parameter lies in the row space of the input matrix. Then by projecting contaminated parameter onto subspace of input matrix, the estimated model parameter could be recovered to the ground truth model parameter with high probability.

%Therefore, the key for successful recovery of $W_{j}$ out of $\Theta_{j}$ is that $W_j(t_{max}) - W_j(0)$ lies in the subspaces spanned by $A_W$. Similarly, we can recover $\beta$ out of $\eta$ because $\beta_j(t_{max}) - \beta_j(0)$ lies in the subspace spanned by $A_{\beta}$. Further conditions which are necessary for successful recovery of $W_{j},\beta$ are theoretically analyzed in theorem \ref{thm: main1} and theorem \ref{thm: main2}. Based on \eqref{eq: recovery optimization W} and \eqref{eq: recovery optimization beta}, the purification of contaminated one-hidden-layer CNN is given in Algorithm~\ref{alg: model repair}. By properly selecting the design matrix of the hidden layer recovery $A_W$ and the design matrix of the output layer $A_\beta$, one can make a successful recovery.

The key to successfully recovering the ground truth model parameter is that it lies in the subspace spanned by the proposed design matrices. In other words, we can recover $W_{j}$ from $\Theta_{j}$ due to the fact that $W_j(t_{max}) - W_j(0)$ lies in the subspaces spanned by $A_W$. Similarly, we can recover $\beta$ from $\eta$ because $\beta_j(t_{max}) - \beta_j(0)$ lies in the subspace spanned by $A_{\beta}$. By projecting the contaminated parameter onto the subspace of design matrices, the estimated model parameter can be recovered to the ground truth model parameter with high probability. The detailed analysis is provided in the following subsections. Further conditions necessary for the successful recovery of $W_{j},\beta$ are theoretically analyzed in Theorem \ref{thm: main1} and Theorem \ref{thm: main2} in Section~\ref{sec: theory}. Based on the above analysis, the purification of the contaminated one-hidden-layer CNN is presented in Algorithm~\ref{alg: model repair}. By properly selecting the design matrices for all layers of CNNs, one can achieve successful recovery.

\subsection{Design Matrix of hidden layer $A_W$} 

We now explain in detail why we choose $A_W$ in the format of \eqref{eq: design matrix A_W}. We define the mapping from input to output as $ f(\mathbf{x_s})=\frac{1}{\sqrt{p}} \sum_{j=1}^{p}\sum_{i=1}^{m} \beta_{j}\psi(W_{j}^{T}P_{i}\mathbf{x_s})$. For weights update in each iteration of the Algorithm~\ref{alg: gradent descent of CNN}, the partial derivative of the loss function with respect to $W_j$ is represented by
\begin{align*}
    &\frac{\partial \mathcal{L}(\beta,W)}{\partial W_{j}}  \bigg|_{(\beta, W)=(\beta(t), W(t-1))} = \frac{\partial \mathcal{L}}{\partial f} \frac{\partial f}{\partial W_{j}} \\
     &= \delta_j \sum_{s=1}^n [(\frac{1}{\sqrt{p}} \sum_{j=1}^p \sum_{i=1}^{m}\beta_j(t) \psi(W_j(t-1)^T P_i \mathbf{x_s})-y_s)]\\
    &\cdot [\beta_i (t)\sum_{i=1}^{m} \psi^{\prime}({W}_j^T (t-1)  P_i \mathbf{x_s})P_i \mathbf{x_s} ]\\
    &= \sum_{s=1}^{n}\sum_{i=1}^{m}\alpha_{i}P_{i}\mathbf{x_{s}}
    \end{align*}
 where $\delta_j$ is a constant and $\alpha_{i}$ sums up all other remaining terms.  
\begin{align*}
   W_j(t_{max}) - W_j(0)  & = \sum_{t=1 }^{t_{max}} W_j(t) - W_j(t-1)\\
   &= \sum_{t=1 }^{t_{max}}  - \frac{\gamma}{k} \frac{\partial \mathcal{L}(\beta(t),W(t-1))}{\partial W_{j}(t-1)}\\
    &= \sum_{t=1 }^{t_{max}} \sum_{s=1}^{n}\sum_{i=1}^{m}\alpha_{i}^{'}P_{i}\mathbf{x_{s}}\\
\end{align*}

One can easily observe that the gradient $\frac{\partial \mathcal{L}(\beta,W)}{\partial W_{j}} $ lies in the subspace spanned by $P_{i}\mathbf{x_{s}}$. And this indicates that vector $W_{j}(t_{max}) - W_{j}(0)$ also lies in the same subspace. Therefore, we can use the design matrix $A_{W}$ in the format of \eqref{eq: design matrix A_W} to purify CNNs' weights.

\subsection{Design Matrix of output layer $A_{\beta}$}
We then introduce how we select $A_{\beta}$ in the form of \eqref{eq: design matrix A_beta} and how it helps the recovery. For weights update in each iteration of the Algorithm~\ref{alg: gradent descent of CNN}, the partial derivative of the loss function with respect to $\beta$ is shown below. %\Ren{why beta has this fixed point version while W does not have ($(\beta,W)=(\beta(t-1),W(t-1))$)? why in the equation you use $W(t)$ when the fixed point is $W(t-1)$?}

%\Ren{We probably need to include the fixed time point.}

\begin{align*}
    &\frac{\partial \mathcal{L}(\beta,W)}{\partial \beta_{j}} \bigg|_{(\beta, W)=(\beta(t-1), W(t-1))} = \frac{\partial \mathcal{L}}{\partial f} \frac{\partial f}{\partial \beta_{j}} \\
    &= \frac{1}{\sqrt{p}} \sum_{s=1}^n\left(\frac{1}{\sqrt{p}} \sum_{j=1}^p \sum_{i=1}^{m}\beta_j (t-1)\psi\left(W_j(t-1)^T P_i \mathbf{x_s}\right)-y_s\right)\\
    &\cdot \sum_{i=1}^{m} \psi\left({W}_j^T(t-1) P_i \mathbf{x_s}\right)\\
    &= \sum_{s=1}^{n} \delta_{s} \sum_{i=1}^{m}  \psi(W_{j}^{T}(t-1)P_{i}\mathbf{x_s}) 
\end{align*}
where $\delta_{s}$ sum ups all other remaining terms. 

\begin{align*}
   \beta_j(t_{max}) - \beta_j(0)  & = \sum_{t=1 }^{t_{max}} \beta_j(t) - \beta_j(t-1)\\
   &= \sum_{t=1 }^{t_{max}}  - \gamma\frac{\partial \mathcal{L}(\beta(t-1),W(t-1))}{\partial W_{j}(t-1)}\\
    &= \sum_{t=1 }^{t_{max}} \sum_{s=1}^{n} \delta_{s} \sum_{i=1}^{m}  \psi(W_{j}^{T}(t-1)P_{i}\mathbf{x_s})\\
\end{align*}

Since the derivative of $\mathcal{L}$ with respect to the $j$-th entry $\beta_j$ is represented by combinations of $\sum_{i=1}^{m}  \psi(W_{j}^{T}(t-1)P_{i} \mathbf{x}_s)$ %and $\delta_s$ only depends on $\mathbf{x}_s,$
we get the conclusion that $\frac{\partial \mathcal{L}(\beta,W)}{\partial \beta}$ lies in the subspace that is spanned by $\sum_{i=1}^{m}  \psi(W^{T}(t-1)P_{i} \mathbf{x}_s)$. %\Ren{check if this is correct. Remember that we want to say something about $\beta$ here, not $\beta_j$. My question here is that should we use $\sum_{i=1}^{m}  \psi(W(0)^{T}P_{i} \mathbf{x}_s)$ or $\sum_{i=1}^{m}  \psi(W(t_{max})^{T}P_{i} \mathbf{x}_s)$ or it doesn't matter?} 
Further notice that $\beta(t_{max}) -\beta(0)$ is an accumulation of $\frac{\partial \mathcal{L}(\beta,W)}{\partial \beta}$ in each iteration. Unlike the subspace spanned by $P_i \mathbf{x_s}$ which is used for hidden layer recovery remains constant, the subspace spanned by $\sum_{i=1}^{m}  \psi(W^{T}(t-1)P_{i} \mathbf{x}_s)$ which is used for output layer recovery keeps changing over $t$. However, thanks to overparametrization assumption of CNN, one could show $W(t)$ obtained by Algorithm \ref{alg: gradent descent of CNN} is close to initialization $W(0)$ for all $t \geq 0$. Theorem~\ref{thm: bound_iteration} in the next section shows that $W(t)$s are all not far away from each other. Thus, $\beta(t_{max}) -\beta(0)$ approximately lies in the same spanned subspace, resulting in the proposed design matrix $A_{\beta}$.

\section{Theoretical Recovery Guarantee}\label{sec: theory}
In the previous section, we introduced our CNN purification algorithm and went over how to build design matrices for recovering the hidden and the output layers. In this section, we demonstrate theoretically that the proposed algorithm's estimation is accurate. Before providing the main theoretical results, we first illustrate the reliability of the design matrices in Lemma~\ref{lemma: conditions} that $A_W$ and $A_\beta$ satisfy certain conditions.

\begin{lemma}\label{lemma: conditions}
Assume that $\frac{mn}{k}$ ($\frac{mn}{\sqrt{p}},\frac{nlog(mn)}{k}$) is sufficiently small, following upper and lower bounds hold for $A=A_{W}$ ($A=A_{\beta}$) with some constants $\sigma^2$, $\underline{\lambda}$, and $\bar{\lambda}$. 
\begin{align}
\Vert \frac{1}{|A|}\sum_{i=1}^{|A|} c_iA_i \Vert^2 \leq \frac{\sigma^2 D_A}{|A|},\label{sublemma: condition_a}\tag{8}\\
\underset{||\Delta||=1}{inf}\frac{1}{|A|}\sum_{i=1}^{|A|}|A_i^T\Delta|\geq \underline{\lambda},\label{sublemma: condition_b}\tag{9}\\
\underset{||\Delta||=1}{sup}\frac{1}{k}\sum_{i=1}^k|A_i^T\Delta|^2\leq \bar{\lambda}^2, \label{sublemma: condition_c}\tag{10}
\end{align}
where $|A|$ is the column number of $A$, and $D_{A}$ is the dimension of $A_i$ . $c_1,\cdots,c_{|A|}$ are fixed values satisfying $\max_s|c_i|\leq 1$ . $A$ is either $A_{W}$ or $A_{\beta}$ . And  $A_{W}$ and $A_{\beta}$ can be the design matrices for recovering the contaminated parameters.
\end{lemma}

\begin{proof}[\emph {Proof of Lemma \ref{lemma: conditions}}]
We firstly prove that lemma \ref{lemma: conditions} holds for hidden layer design matrix $A_W$ by lemma 2, lemma 3 and lemma 4 in the Appendix. Then We prove that lemma \ref{lemma: conditions} holds for output layer design matrix $A_{\beta}$ by lemma 5, lemma 6 and lemma 7 in the Appendix.
Combining Lemma~\ref{lemma: conditions} with Lemmas~4 and 7 in the Appendix, we can conclude that using $A_W$ and $ A_\beta$ can purify contaminated parameters to their ground truth with high probability.

\end{proof}

\iffalse 
\textbf{Condition A}. There exists some $\sigma^2$, such that for any fixed (not random) $c_1,...,c_k$ satisfying $\max_s|c_i|\leq 1$
\begin{align}
\Vert \frac{1}{k}\sum_{i=1}^k c_ia_i \Vert^2 \leq \frac{\sigma^2 mn}{k}
    \label{eq: condition A}
\end{align}
with high probability.

\textbf{Condition B}. There exists $\underline{\lambda}$ and $\bar{\lambda}$ such that
\begin{align}
\underset{||\Delta||=1}{inf}\frac{1}{k}\sum_{i=1}^k|a_i^T\Delta|\geq \underline{\lambda}
    \label{eq: condition B.1}
\end{align}
\begin{align}
    \underset{||\Delta||=1}{sup}\frac{1}{k}\sum_{i=1}^k|a_i^T\Delta|^2\leq \bar{\lambda}^2
    \label{eq: condition B.2}
\end{align}

with high probability.

\textbf{Condition A}. There exists some $\sigma^2$, such that for any fixed (not random) $c_1,...,c_m$ satisfying $\max_s|c_i|\leq 1$
\begin{align}
\Vert \frac{1}{m}\sum_{i=1}^m c_ia_i \Vert^2 \leq \frac{\sigma^2 k}{m}
    \label{eq: condition A}
\end{align}
with high probability.

\textbf{Condition B}. There exists $\underline{\lambda}$ and $\bar{\lambda}$ such that
\begin{align}
\underset{||\Delta||=1}{inf}\frac{1}{m}\sum_{i=1}^m|a_i^T\Delta|\geq \underline{\lambda}
    \label{eq: condition B.1}
\end{align}
\begin{align}
    \underset{||\Delta||=1}{sup}\frac{1}{m}\sum_{i=1}^m|a_i^T\Delta|^2\leq \bar{\lambda}^2
    \label{eq: condition B.2}
\end{align}

\fi

We remark that \eqref{sublemma: condition_a} is critical for the universality of the distribution $Q_i$ where $\mathbf{z}$ is generated from. In the design of our theoretical analysis, $\Delta$ refers to the difference between the ground truth model coefficient and the estimated model coefficient in the optimization. \eqref{sublemma: condition_b} and \eqref{sublemma: condition_c} ensure that the product of a design matrix's column with a bounded model coefficients' difference is always bounded. %Combining Lemma~\ref{lemma: conditions} with Lemmas~\ref{4.1} and \ref{4.5} in the Appendix, we can conclude that using $A_W$ and $ A_\beta$ can purify contaminated CNNs to their ground truth with high probability.

We assume $\mathbf x_s$ follows Gaussian distribution $\mathcal{N}(0,I_d)$ for $\forall s\in[n]$ with $|y_s|\leq 1$. Let $f_s(t)$ be $f(\mathbf{x_s})$ with weights $W_j(t)$ and $\beta_j(t)$. We then have the following conclusion.

\begin{theorem}\label{thm: bound_iteration}
If $\frac{mnlog(mn)}{k}$, $\frac{(mn)^3log(p)^4}{p}$ and $mn\gamma$ are all sufficiently small, then
\begin{align}
&\underset{1\leq j \leq p}{max}||\beta_j(t)-\beta_j(0)||\leq  32\sqrt{\frac{(mn)^2log(p)}{p}}= R_{\beta}\label{eq:R_b}\tag{11}\\
&\underset{1\leq j \leq p}{max}||W_j(t)-W_j(0)||\leq \frac{100mnlog(p)}{\sqrt{pk}} = R_W \label{eq:R_W}\tag{12}\\
&||y-f_s(t)||^2\leq\left(1-\frac{\gamma}{8}\right)^t||y-f_s(0)||^2 \label{eq:R_y}\tag{13}
\end{align}
for all $t\geq 1$ with high probability.
\end{theorem}  

\begin{proof}[\emph {Proof of Theorem \ref{thm: bound_iteration}}]
   We introduce the function
$$v_s(t)=\frac{1}{\sqrt{p}} \sum_{j=1}^p \sum_{a=1}^m\beta_j(t) \psi\left(W_j(t-1)^T P_a x_s\right)$$
\begin{align}
    \|y-v(t)\|^2 \leq\left(1-\frac{\gamma}{8}\right)^t\|y-v(0)\|^2 \label{VI.20}\tag{14}
\end{align}

First, we can prove that for any integer: $k \geq 1$, as long as \eqref{eq:R_b},\eqref{eq:R_W}, \eqref{eq:R_y} and \eqref{VI.20} hold  for all $t \leq k$, then \eqref{eq:R_b}  holds for $t=k+1$ with high probability.

By triangle inequality and the gradient formula,

\begin{align*}
& \left|\beta_j(k+1)-\beta_j(0)\right| \leq \sum_{t=0}^k\left|\beta_j(t+1)-\beta_j(t)\right| \\
& \leq \frac{\gamma}{\sqrt{p}} \sum_{t=0}^k \sum_{s=1}^n\sum_{a=1}^m\left|y_s-f_s(t)\right|\left|W_j(t)^TP_a x_s\right|  \\
&  \leq \frac{\gamma}{\sqrt{p}} \sum_{t=0}^k\|y-f_s(t)\|(R_W \sqrt{\sum_{s=1}^n\sum_{a=1}^m\left\|P_ax_s\right\|^2}\\
&+\sqrt{\sum_{s=1}^n\sum_{a=1}^m\left|W_j(0)^T P_ax_s\right|^2}) \\
& \stackrel{(\rm{a})} \leq \gamma \sqrt{\frac{7 mn+18 \log p}{p}} \sum_{t=0}^k\|y-f_s(t)\|  \\
& \leq 16 \sqrt{\frac{7 mn+18 \log p}{p}}\|y-f_s(0)\| \\
& \stackrel{(\rm{b})} \leq 32 \sqrt{\frac{(mn)^2 \log p}{p}}=R_{\beta}, \\
\end{align*}
where we have used $\max _{1 \leq s \leq n}\left\|Px_s\right\| \lesssim \sqrt{mk} $ and $  \max _{1 \leq j \leq p} \sum_{s=1}^n\left|W_j(0)^T Px_s\right|^2 \leq 6 mn+18 \log p$ in (a), $ \|f_s(0)\| \leq \sqrt{mn}(\log p)^{1 / 4} $ in (b). Hence, \eqref{eq:R_b} holds for $t=k+1$, and the claim for \eqref{eq:R_b} is true.

Secondly, for any integer $k \geq 1$, as long as \eqref{eq:R_W}  and \eqref{eq:R_y} hold for all $t \leq k$, and \eqref{eq:R_b}  and \eqref{VI.20}  hold for all $t \leq k+1$, then \eqref{eq:R_W}  holds for $t=k+1$  with high probability is also obvious. \\
We bound $\left\|W_j(k+1)-W_j(0)\right\|$ by 
%\begin{align*}
%&\left\|W_j(k+1)-W_j(0)\right\|\\ & \leq \sum_{t=0}^k\left\|W_j(t+1)-W_j(t)\right\|\\
%& \leq \frac{\gamma}{k \sqrt{p}} \sum_{t=0}^k\left|\beta_j(t+1)\right| \sum_{s=1}^n\sum_{a=1}^m\left|y_s-v_i(t+1)\right|\left\|P_ax_s\right\| \\
%& \leq \frac{16}{k \sqrt{p}}\left(\left|\beta_j(0)\right|+R_{\beta}\right) \sqrt{\sum_{s=1}^n\sum_{a=1}^m\left\|P_a x_s\right\|^2}\|y-v(0)\| \\
%& \leq \frac{100 mn \log p}{\sqrt{p k}}=R_W,
%\end{align*}

\begin{align*}
&\left\|W_j(k+1)-W_j(0)\right\|\\ & \leq \sum_{t=0}^k\left\|W_j(t+1)-W_j(t)\right\|\\
& \leq \frac{\gamma}{k \sqrt{p}} \sum_{t=0}^k\| \beta_j(t+1) \sum_{s=1}^n \sum_{a=1}^m\left(v_s(t+1)-y_s\right) \\
&\psi^{\prime}\left(W_j(t)^T P_ax_s\right)P_a x_s\| \\
& \leq \frac{\gamma}{k \sqrt{p}} \sum_{t=0}^k\left|\beta_j(t+1)\right| \sum_{s=1}^n\sum_{a=1}^m\left|y_s-v_s(t+1)\right|\left\|P_ax_s\right\| \\
& \leq \frac{\gamma}{k \sqrt{p}}\left(\left|\beta_j(0)\right|+R_{\beta}\right) \sqrt{\sum_{s=1}^n\sum_{a=1}^m\left\|P_ax_s\right\|^2} \sum_{t=0}^k\|y-v(t+1)\| \\
& \leq \frac{16}{k \sqrt{p}}\left(\left|\beta_j(0)\right|+R_{\beta}\right) \sqrt{\sum_{s=1}^n\sum_{a=1}^m\left\|P_a x_s\right\|^2}\|y-v(0)\| \\
& \leq \frac{100 mn \log p}{\sqrt{p k}}=R_W,
\end{align*}
where we have used $\max _{1 \leq j \leq p}\left|\beta_j(0)\right| \leq 2 \sqrt{\log p}$, $\sum_{s=1}^n\left\|Px_s\right\|^2 \leq 2 mn k$ and $\|f_s(0)\| \leq \sqrt{mn}(\log p)^{1 / 4}$ in the last inequality. Thus, the claim for \eqref{eq:R_W} is true.\\

Next, we just need to prove that for any integer $k \geq 1$, as long as \eqref{eq:R_y} holds for all $t \leq k$, and \eqref{eq:R_b}, \eqref{eq:R_W}  and \eqref{VI.20} hold for all $t \leq k+1$, then \eqref{eq:R_y} holds for $t=k+1$ with high probability.

We  define the matrices $G(k),H(k) \in \mathbb{R}^{n\times n}$ with entries

\begin{align*}
G_{s l}(k) & =\frac{1}{p} \sum_{j=1}^p \sum_{a=1}^{m}\sum_{b=1}^{m} \psi\left(W_j(k)^T P_a x_s\right) \psi\left(W_j(k)^T P_b x_l\right) \\
H_{s l}(k) & =\frac{(Px_s)^T Px_l}{k} \frac{1}{p} \sum_{j=1}^p \sum_{a=1}^{m}\sum_{b=1}^{m}\beta_j(k+1)^2\times\\ &\psi^{\prime}\left(W_j(k)^TP_a x_s\right) \psi^{\prime}\left(W_j(k)^T P_b x_l\right)
\end{align*}
and vector $r(k)$ by
\begin{align*}
  & r_{s}(k) =\frac{1}{\sqrt{p}}\sum_{j=1}^p \sum_{a=1}^{m}\beta_{j}(k+1)(\psi(W_{j}(k+1)^{T}P_a x_s)\\
  &-\psi(W_{j}(k)^{T}P_a x_{s})) - \frac{1}{\sqrt{p}}\sum_{j=1}^p\sum_{a=1}^{m}\beta_{j}(k+1)(W_{j}(k+1)\\
  &-W_{j}(k))^{T}Px_{s}\psi^{'}(W_{j}(k)^{T}P_a x_{s})
\end{align*}
%then analyze $f_s(k+1)-f_s(k)$. For each $i \in[n]$, we have

\iffalse 

$$
\begin{aligned}
& f_s(k+1)-f_s(k) = \frac{1}{\sqrt{p}} \sum_{j=1}^p \beta_j(k+1)\times\\
&\sum_{a=1}^m\left(\psi\left(W_j(k+1)^T P_ax_s\right)-\psi\left(W_j(k)^T P_ax_s\right)\right) \\
& +\frac{1}{\sqrt{p}} \sum_{j=1}^p\sum_{a=1}^m\left(\beta_j(k+1)-\beta_j(k)\right) \psi\left(W_j(k)^T P_a x_s\right) \\
&=  \gamma \sum_{l=1}^n\left(H_{s l}(k)+G_{sl}(k)\right)\left(y_l-u_l(k)\right)+r_s(k),
\end{aligned}
$$

\fi

To bound $G(k)$ , $H(k)$ and $r(k)$, we have
\begin{align}
0\le \lambda_{min}(G(k)) \le \lambda_{max}(G(k)) \lesssim m n.\label{VI.28}\tag{15}
\end{align}

\begin{align}
\frac{1}{6} \leq \lambda_{min}(H(k)) \leq \lambda_{max}(H(k)) \lesssim 1 \label{VI.33}\tag{16}
\end{align}

\begin{align}
&\Vert r(k) \Vert=\sqrt{\sum_{s}\big| r_{s}(k)\big|^2}\notag\\
&\lesssim \gamma mn logp(\sqrt{mk}R_{W}+\sqrt{\frac{log mn}{p}})\Vert y-f_s(k) \Vert \label{VI.34}\tag{17}
\end{align}

Please refer to the Appendix for specific analysis of \eqref{VI.28}, \eqref{VI.33}, and \eqref{VI.34}, respectively.

Now we are ready to analyze $\Vert y-f_s(k+1)\Vert^2$. Given the relation 
$$f_s(k+1)-f_s(k)=\gamma(H(k)+G(k))(y-f_s(k))+r(k) ,$$

we have
\begin{align*}
&\Vert y-f_s(k+1)\Vert^2 = \Vert y-f_s(k)\Vert^2-2\langle y-f_s(k),f_s(k+1)\\
&-f_s(k) \rangle +\Vert f_s(k) - f_s(k+1) \Vert^2\\
&=\Vert y-f_s(k)\Vert^2-2\gamma(y-f_s(k))^T\times (H(k)+G(k))(y-f_s(k))\\
&-2\langle y-f_s(k),r(k) \rangle+\Vert f_s(k)-f_s(k+1)\Vert^2.\\
\end{align*}

By \eqref{VI.28} and \eqref{VI.33}, we have

\begin{align}
-2\gamma(y-f_s(k))^T(H(k)+G(k))(y-f_s(k))\le -\frac{\gamma}{6}\Vert y-f_s(k) \Vert^2.\label{VI.35}\tag{18}
\end{align}

The bound \eqref{VI.34} implies
\begin{align*}
    &-2\langle y-f_s(k),r(k) \rangle \le 2\Vert y-f_s(k)\Vert \Vert r(k)\Vert\\& \lesssim \gamma mn logp(R_{W}+\sqrt{\frac{log mn}{p}})\Vert y-f_s(k) \Vert^2
\end{align*}

By \eqref{VI.28}, \eqref{VI.33} and \eqref{VI.34}, we also have
$$
\Vert f_s(k)-f_s(k+1) \Vert^2 \le 2\gamma^2\Vert (H(k)+G(k))(y-f_s(k)) \Vert^2 + 2\Vert r(k \Vert^2)
$$
$$
\lesssim \gamma^2 mn \Vert y-f_s(k) \Vert^2 + (\gamma mn logp)^2(R_{W}+\sqrt{\frac{log mn}{p}})^2\Vert y-f_s(k) \Vert^2
$$
Therefore, as long as $\frac{mnlog mn}{k},\frac{(mn)^3(logp)^4}{p}$ and $\gamma mn$ are all sufficiently small, we have
$$
-2\langle y-f_s(k),r(k) \rangle + \Vert f_s(k)-f_s(k+1) \Vert^2 \le \frac{\gamma}{24} \Vert y-f_s(k) \Vert^2
$$

Together with the bound \eqref{VI.35}, we have
$$
\Vert y-f_s(k+1) \Vert^2 \le (1-\frac{\gamma}{8})\Vert y-f_s(k) \Vert^2\le (1-\frac{\gamma}{8})^{k+1}\Vert y-f_s(0) \Vert^2,
$$
and the claim for \eqref{VI.20} is true.

Finally we can prove that for any integer $k \geq 1$, as long as \eqref{eq:R_W}, \eqref{eq:R_y} and \eqref{VI.20} hold for all $t \leq k$, and \eqref{eq:R_b} holds for all $t \leq k+1$, then \eqref{VI.20} holds for $t=k+1$ with high probability. We will not expand the proof here because the analysis uses the same argument as that of the proof of \eqref{eq:R_y}.

With all the claims above being true, we can then deduce \eqref{eq:R_b}, \eqref{eq:R_W}, \eqref{eq:R_y} and \eqref{VI.20} for all $t \geq 1$ by mathematical induction.

\end{proof}

Although weights $W(t)$ and $\beta(t)$ are updated over iterations $t$, Theorem~\ref{thm: bound_iteration} tells us that the post-trained parameter $W$ and $\beta$ via Algorithm~\ref{alg: gradent descent of CNN} are not too far away from their initializations. Due to the bounded distance, we can show that $\beta(t_{max}) -\beta(0)$ approximately lies in the subspace spanned by $A_\beta$. Moreover, the distance between the ground truth $y$ and the final prediction is bounded by the distance between $y$ and the model's initial prediction, indicating a global convergence of Algorithm~\ref{alg: gradent descent of CNN} despite the nonconvexity of the loss.

Assisting by  Theorem~\ref{thm: bound_iteration}, we propose two main theorems below to demonstrate that Algorithm~\ref{alg: model repair} can effectively purify CNN under two different training situations. Under Algorithm~\ref{alg: gradent descent of CNN}, the following conclusion holds.

\begin{theorem}\label{thm: main1}
Under condition of theorem \ref{thm: bound_iteration} with additional assumption that $\frac{log(p)}{k}$ and $\epsilon\sqrt{mn}$ are sufficiently small, then $\widetilde{W} = W(t_{max})$ and $\frac{1}{p}||\widetilde{\beta}-\beta(t_{max})||^2\lesssim \frac{(mn)^3log(p)}{p}$ with high probability, where $W(t_{max})$ and $\beta(t_{max})$ are obtained by gradient descent algorithm and $\widetilde{W}$ and $\widetilde{\beta}$ are results of model purification of convolution neural network.
\end{theorem}

\begin{proof}[\emph {Proof of Theorem \ref{thm: main1}}]
Consider $\eta = b + Av^* + z \in \mathbb{R}^k$, where the noise vector $z$ satisfies
$z_i \sim(1-\varepsilon) \delta_0+\varepsilon Q_i,$ independently for all $i \in[\mathrm{m}]$. And $b \in \mathbb{R}^k$ is an arbitrary bias vector. Then, the estimator $\hat{v} = \mathop{argmin}\limits_{v \in \mathbb{R}^{n}}|| \eta - Av||_{1}$ satisfies the theoretical guarantee lemma 9 in the Appendix.

We first analyze $\hat{u_1},..,\hat{u_p}$. The idea is to apply the result of lemma 2 in the Appendix to each of the $p$ robust regression problems. Thus, it suffices to check if the conditions of  lemma 2 in the Appendix hold for the $p$ regression problems simultaneously. Then, by the same argument that leads to lemma 4 in the Appendix, we have $\widetilde{W_j} =\widehat{W_j}$ for all $j \in [p]$  with high probability.

%Since the $p$ regression problems share the same Gaussian design matrix, Lemma\ref{4.2} implies that Conditions A and B hold for all the $p$ regressionproblems. Next, by scrutinizing the proof of Lemma \ref{4.1}, the randomness of the conclusion is from the noise vector $Z_j$ through the empirical process bound. With anadditional union bound argument applied to (A.2) in its proof, Lemma A.6 can be extended to $Z_j$ simultaneously for all $j\in [p]$ with an additional assumption that $\frac{log p}{k}$ is sufficiently small. Then, by the same argument that leads to Corollary 4.1, we have $\widetilde{W_j} =\hat{W_j}$ for all $j \in [p]$  with highprobability.

%To analyze $\hat{f_s}$, we apply the lemma 9 in the Appendix. 
Note that 
\begin{align*}
&\eta_j-\beta_j(0)\\
&= \sum_{t=0}^{t_{\max }-1}\left(\beta_j(t+1)-\beta_j(t)\right)+z_j \\
&=  \frac{\gamma}{\sqrt{p}} \sum_{t=0}^{t_{\max }-1} \sum_{s=1}^n\sum_{a=1}^m\left(y_s-f_s(t)\right) \psi\left(W_j(t)^T P_ax_s\right)+z_j \\
&=  \frac{\gamma}{\sqrt{p}} \sum_{t=0}^{t_{\max }-1} \sum_{s=1}^n\sum_{a=1}^m\left(y_s-f_s(t)\right)(\psi\left(W_j(t)^T P_a x_s\right)\\
&-\psi\left(W_j(0)^TP_a x_s\right)) \\
& +\frac{\gamma}{\sqrt{p}} \sum_{t=0}^{t_{\max }-1} \sum_{s=1}^n\sum_{a=1}^m\left(y_s-f_s(t)\right) \psi\left(W_j(0)^TP_a x_s\right)+z_j .
\end{align*}

Thus, in the framework of lemma 9 in the Appendix, we can view $\eta - \beta(0)$ as the response, $\psi( W(0)^TPX)$ as the design, z as the noise, and $b_j = \frac{\gamma}{\sqrt{p}}\sum_{t=0}^{t_{max}-1}\sum_{s=1}^n\sum_{a=1}^m(y_s-f_s(t))(\psi(W_j(t)^T P_a x_s)-\psi(W_j(0)^T P_a x_s))$. By lemma 6 in the Appendix, we know that the design matrix $\psi(W(0)^TPX)$ satisfies \eqref{sublemma: condition_a} , \eqref{sublemma: condition_b} and \eqref{sublemma: condition_c}. So it suffices to bound $\frac{1}{p}\sum_{j=1}^p|b_j|$. Then we have
\begin{align*}
&\frac{1}{p}\sum_{j=1}^p|b_j|\\&\leq \frac{\gamma}{p^\frac{3}{2}}\sum_{j=1}^p\sum_{t=0}^{t_{max}-1}\sum_{s=1}^n\sum_{a=1}^m|y_s - f_s(t)||(W_j(t) - W_j(0))^T P_a x_s|\\
&\leq \frac{R_W\gamma}{p^{\frac{1}{2}}}\sum_{t=0}^{t_{max}-1}\sum_{s=1}^n\sum_{a=1}^m|y_s - f_s(t)|||P_a x_s||\\
&\leq \frac{R_W\gamma}{p^{\frac{1}{2}}}\sum_{t=0}^{t_{max}-1}||y - f_s(t)||\sqrt{\sum_{s=1}^n\sum_{a=1}^m||P_a x_s||^2}\\
&\lesssim \frac{R_W}{p^\frac{1}{2}}||y - f_s(0)||\sqrt{\sum_{s=1}^n\sum_{a=1}^m||P_a x_s||^2}\\
&\lesssim \frac{(mn)^2 logp}{p}\\
\end{align*}
where the last inequality is by $\sum_{s=1}^n\sum_{a=1}^m||P_a x_s||^2 \lesssim mnk$ due to a standard chi-squared bound , and $||f_s(0)||^2 \lesssim mn$ is due to Markov's inequality
and $\mathbb{E}|f_s(0)|^2 = \mathbb{E}Var(f_s(0)|X) \leq 1$. We then have $\frac{1}{p}||\tilde{\beta}-\hat{\beta}|| \lesssim \frac{(mn)^3 logp}{p} $, which is desired conclusion.
\end{proof}

According to Theorem \ref{thm: main1} pre-condition $\frac{mnlog(mn)}{k}$, $\frac{(mn)^3log(p)^4}{p}$ and $mn\gamma$, successful model purification requires large number of hidden layer neurons $p$, large partition dimension $k$, small number of partition $m$, small training examples $n$ and small poisoned ratio $\epsilon$. The assumption $\frac{log(p)}{k}$ further puts the constraint on the distance between $log(p)$ and $k$ in terms of successful parameter purification. Compared with theorem B.2 in \cite{gao2020model}, extra $m$ terms appear, and $d$ is substituted by $k$. It is reasonable since the construction of both design matrices takes account of $m$ and input dimension to feed into CNN is $k$ rather than $d$ of DNN. The reason $\beta$ could not be exactly recovered and has error bound  $\frac{(mn)^3log(p)}{p}$ is because subspace spanned by $\sum_{i=1}^{m}  \psi(W^{T}(t-1)P_{i} \mathbf{x}_s)$ keeps changing over $t$. Thus $\beta(t_{max}) -\beta(0)$ approximately lies in the subspace spanned by $A_\beta$.

One can also update CNN in a different way. $\beta$ can be updated after $W$ been updated $t_{max}$ iterations, i.e., after $\widehat{W} = W(t_{max})$. Then the CNN is trained by freezing the hidden layer $W = W(t_{max})$ and only updates $\beta$ via $\widetilde{X} = \psi(PX\widehat{W})$. $\beta(0)$ is initialized at $\mathbf{0}$. In this case, the following theorem holds.

\begin{theorem}\label{thm: main2} 
Under condition of theorem \ref{thm: bound_iteration} with additional assumption that $\frac{log(p)}{k}$ and $\epsilon\sqrt{mn}$ are sufficiently small,
then $\widetilde{W} = W$ and $\widetilde{\beta} = \beta$  with high probability.
\end{theorem} 
\begin{proof}[\emph {Proof of Theorem \ref{thm: main2}}]
     The analysis of $\widehat{u}_1,..,\widehat{u}_p$ is the same as that in the proof of Theorem \ref{thm: main1}, and we have $\widetilde{W_j} = \widehat{W_j}$ for all $j \in [p]$ with high probability.

To analyze $\widehat{v}$, we apply lemma 2 in the Appendix. It suffices to check \eqref{sublemma: condition_a} , \eqref{sublemma: condition_b} and \eqref{sublemma: condition_c} for the design matrix $\psi((PX)^T\widetilde{W}^T) = \psi((PX)^T\widehat{W}^T)$. Since

\begin{align*}
    &\sum_{s=1}^n \mathbb{E}\left(\frac{1}{p} \sum_{j=1}^p\sum_{a=1}^m c_j \psi\left(\widehat{W}_j^T P_a x_s\right)\right)^2 \\
    &\leq \sum_{s=1}^n \frac{1}{p} \sum_{j=1}^p \mathbb{E} \sum_{a=1}^m \psi\left(\widehat{W}^T P_a  x_s\right)^2
\end{align*}
and $\mathbb{E} \sum_{a=1}^m \psi\left(\widehat{W}^T P_a x_s\right)^2 \leq \mathbb{E} \sum_{a=1}^m \left|\widehat{W}_j^T P_a x_s\right|^2 \lesssim 1+R_W k \lesssim 1$, \eqref{sublemma: condition_a} holds with $\sigma^2 \asymp p$. We also need to check \eqref{sublemma: condition_b} and \eqref{sublemma: condition_c}. By Theorem \ref{thm: bound_iteration}, we have

\begin{align*}
&\left|\frac{1}{p} \sum_{j=1}^p\right| \sum_{s=1}^n  \sum_{a=1}^m  \psi(\widehat{W}_j^T P_a x_s) \Delta_s\lvert\\
&-\frac{1}{p} \sum_{j=1}^p\rvert \sum_{s=1}^n  \sum_{a=1}^m \psi\left(W_j(0)^T P_a x_s\right) \Delta_s \Bigg| \Bigg| \\
& \leq \frac{1}{p} \sum_{j=1}^p \sum_{s=1}^n \sum_{a=1}^m \left| \widehat{W}_j^T P_a x_s-W_j(0)^T P_a x_s \right|\lvert \Delta_s\rvert \\
& \leq R_W \sum_{s=1}^n \sum_{a=1}^m \lvert P_a x_s\rvert\lvert\Delta_s\rvert \\
& \lesssim \frac{(mn)^{3 / 2} \log p}{\sqrt{p}}
\end{align*}

By lemma 6 in the Appendix, we can deduce that
$$
\inf _{\|\Delta\|=1} \frac{1}{p} \sum_{j=1}^p\left|\sum_{s=1}^n \sum_{a=1}^m \psi\left(\widehat{W}_j^T P_a x_s\right) \Delta_s\right| \gtrsim 1,
$$
as long as $\frac{(mn)^{3 / 2} \log p}{\sqrt{p}}$ is sufficiently small. And we also have
$$
\sup _{\|\Delta\|=1} \frac{1}{p} \sum_{j=1}^p\left|\sum_{s=1}^n  \sum_{a=1}^m\psi\left(\widehat{W}_j^T P_a x_s\right) \Delta_s\right|^2 \lesssim m n .
$$
Therefore, \eqref{sublemma: condition_b} and \eqref{sublemma: condition_c} holds with $\bar{\lambda}^2 \asymp m n$ and $\underline{\lambda} \asymp 1$. Applying lemma 2 in the Appendix, we have $\widetilde{\beta}=\widehat{\beta}$ with high probability, as desired.
\end{proof}

%\Ren{need some analysis and an explanation of what makes the last theorem has this exact recovery.}
Under setting of Theorem \ref{thm: main2}, $\beta$ could be exactly purified since subspace spanned by $\sum_{i=1}^{m}  \psi(W^{T}(t_{max})P_{i} \mathbf{x}_s)$ keeps constant by freezing hidden layer $W^{T}(t_{max})$. Therefore, $\beta(t_{max}) -\beta(0)$ entirely lies in the subspace spanned by $A_\beta$.

\begin{figure}[ht]
  \centering
  \begin{subfigure}{0.45\textwidth}
             \includegraphics[trim=0 0 0 0,clip,width=1.0\textwidth]{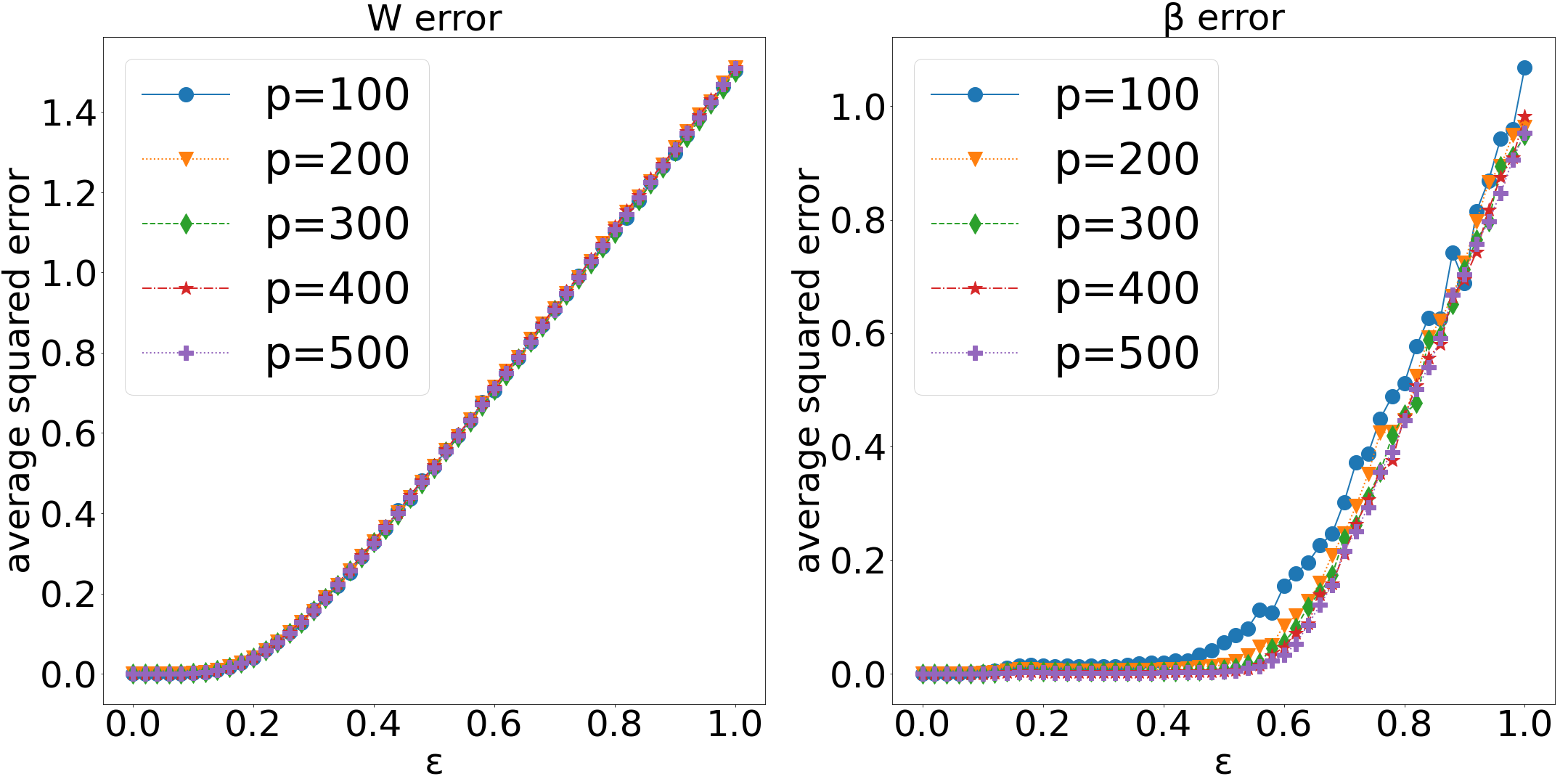}
             \caption{$k$=50}
  \end{subfigure}
  \begin{subfigure}{0.45\textwidth}
             \includegraphics[trim=0 0 0 0,clip,width=1.0\textwidth]{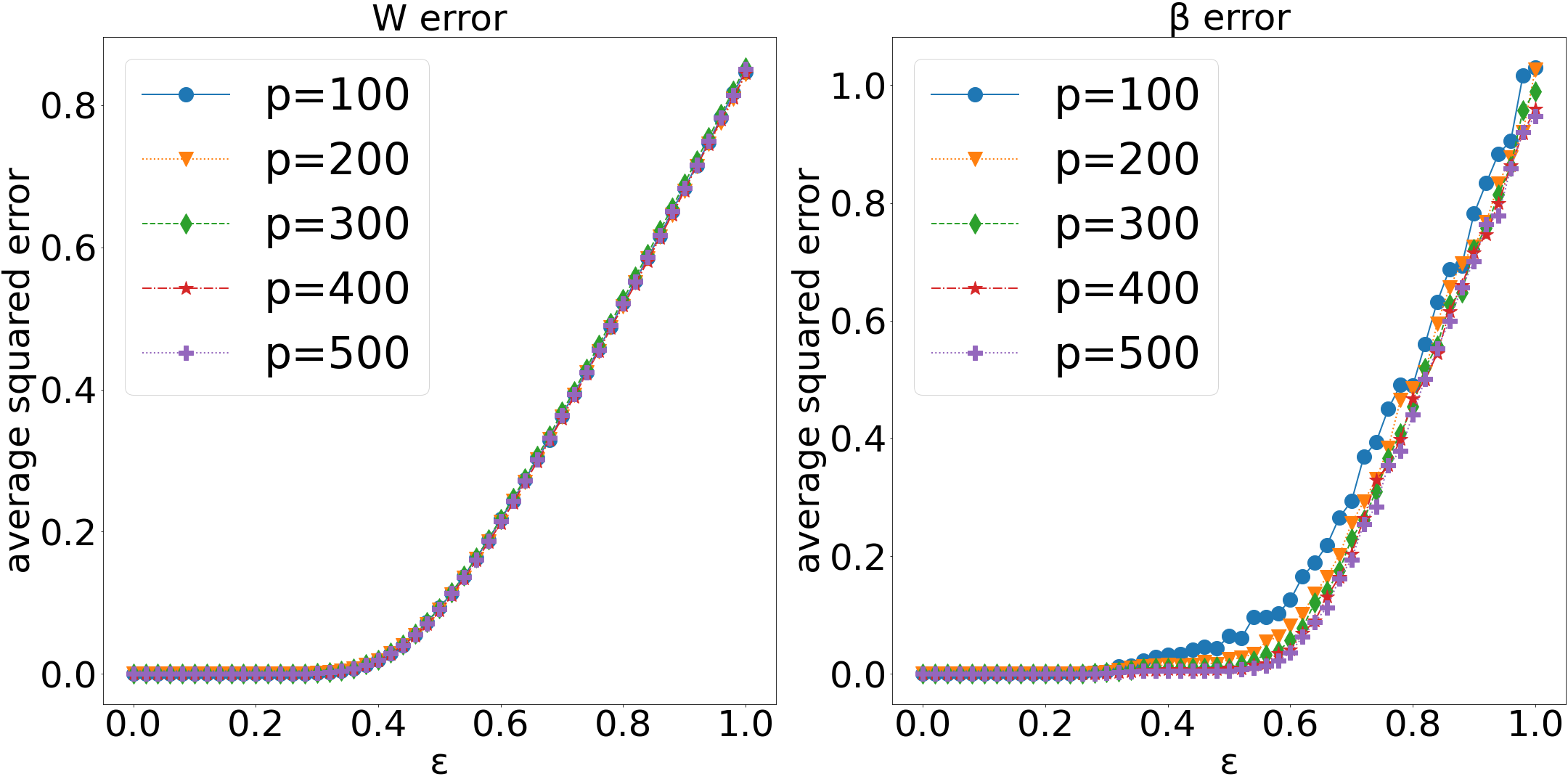}
             \caption{$k$=100}
  \end{subfigure}
    \begin{subfigure}{0.45\textwidth}
             \includegraphics[trim=0 0 0 0,clip,width=1.0\textwidth]{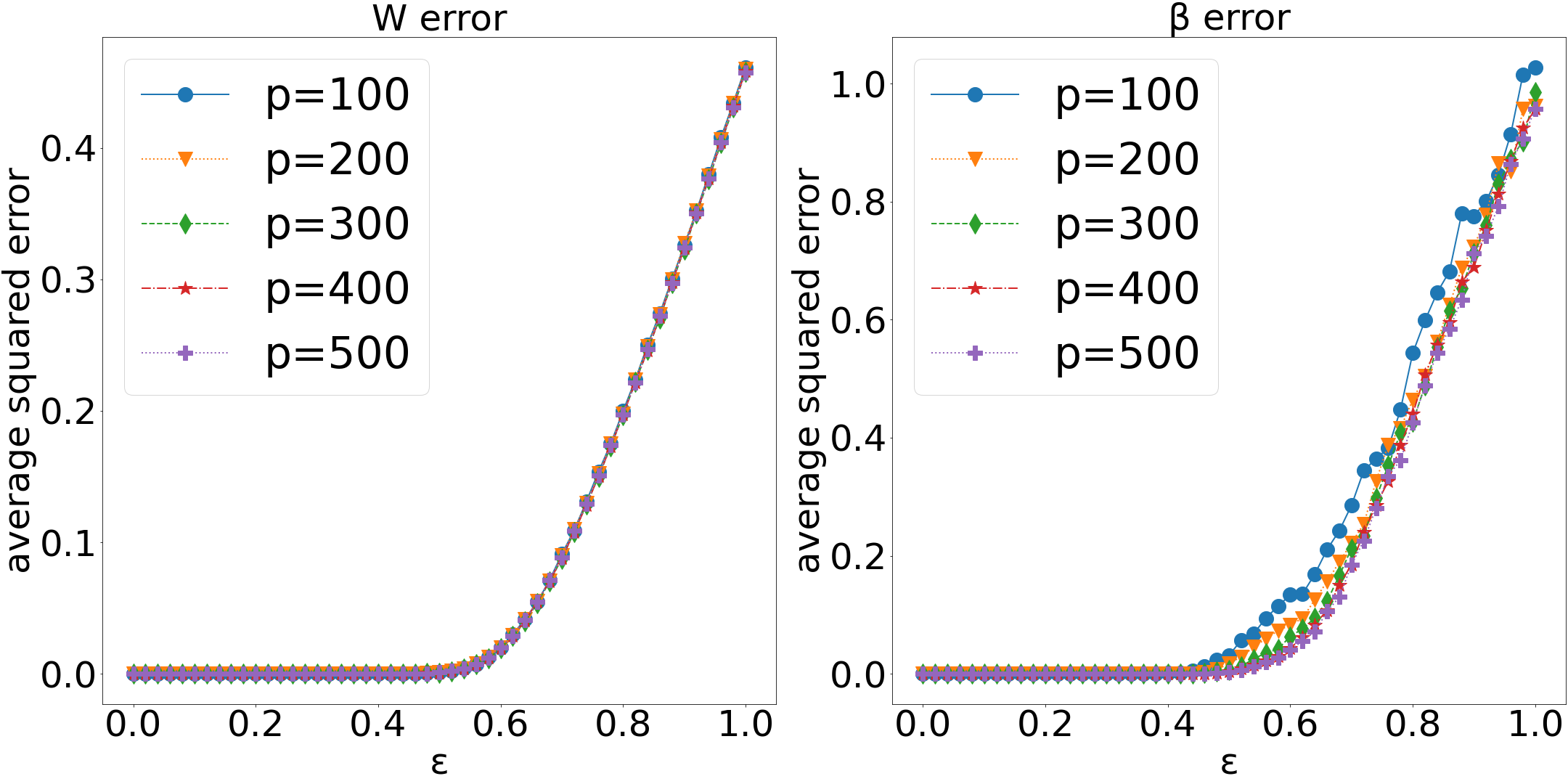} 
             \caption{$k$=150}
  \end{subfigure}
  \caption{{\bf Increasing $p$ and $k$ promotes the recovery performance ($n = 5, m=5$) on synthetic data}. Experiments under settings in Theorem \ref{thm: main1}. When $p$ increases, the limit of $\epsilon$ for successful recovery of $\beta$ also increases. When $k$ increases, the limit of $\epsilon$ for successful recovery of $W$ increases.} 
  \label{fig: syn_p_k_main1}
\end{figure}

\begin{figure}[h]
  \centering
  \begin{subfigure}{0.45\textwidth}
             \includegraphics[trim=0 0 0 0,clip,width=1\textwidth]{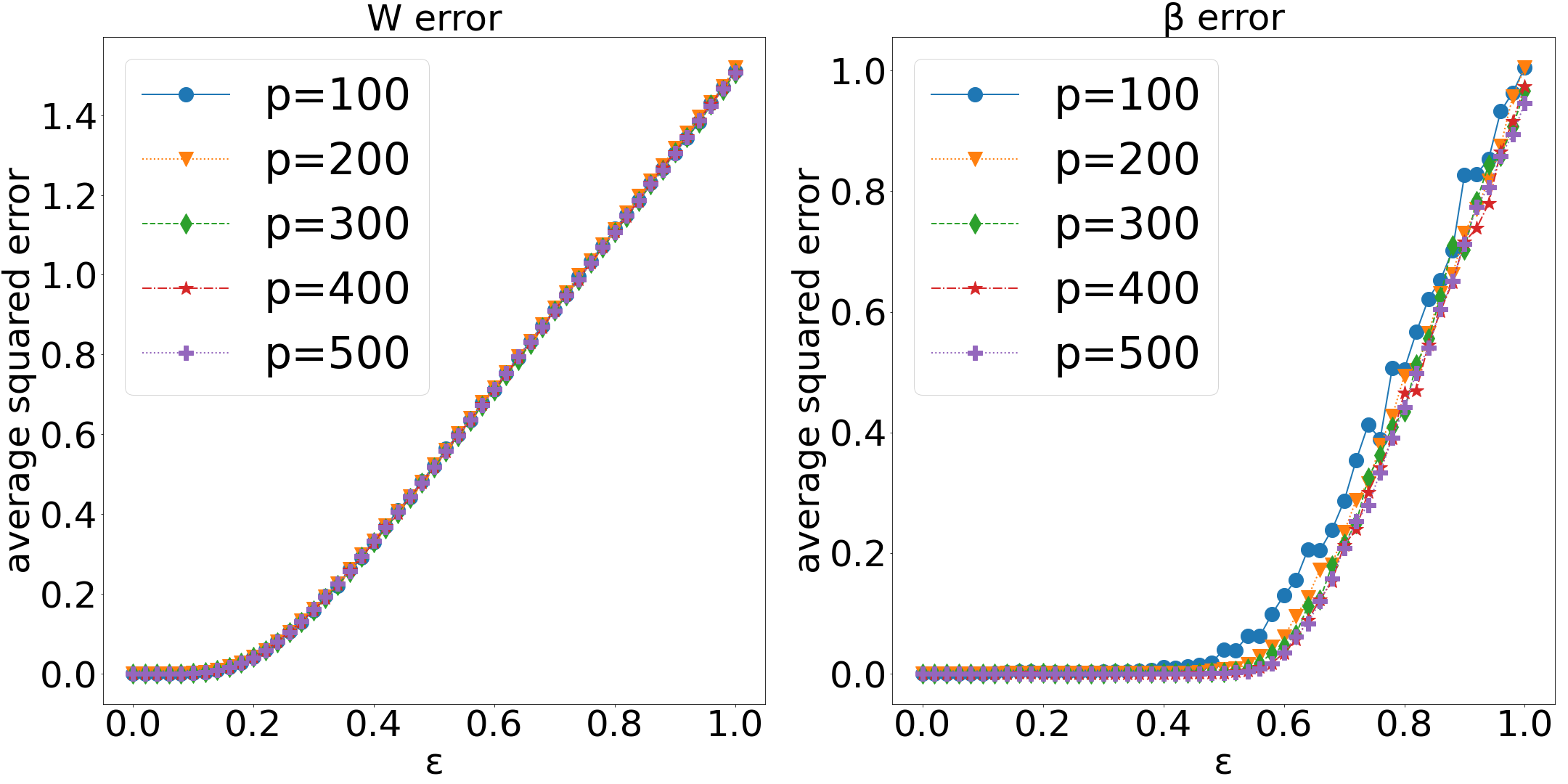}
             \caption{$k$=50}
  \end{subfigure}
  \begin{subfigure}{0.45\textwidth}
             \includegraphics[trim=0 0 0 0,clip,width=1\textwidth]{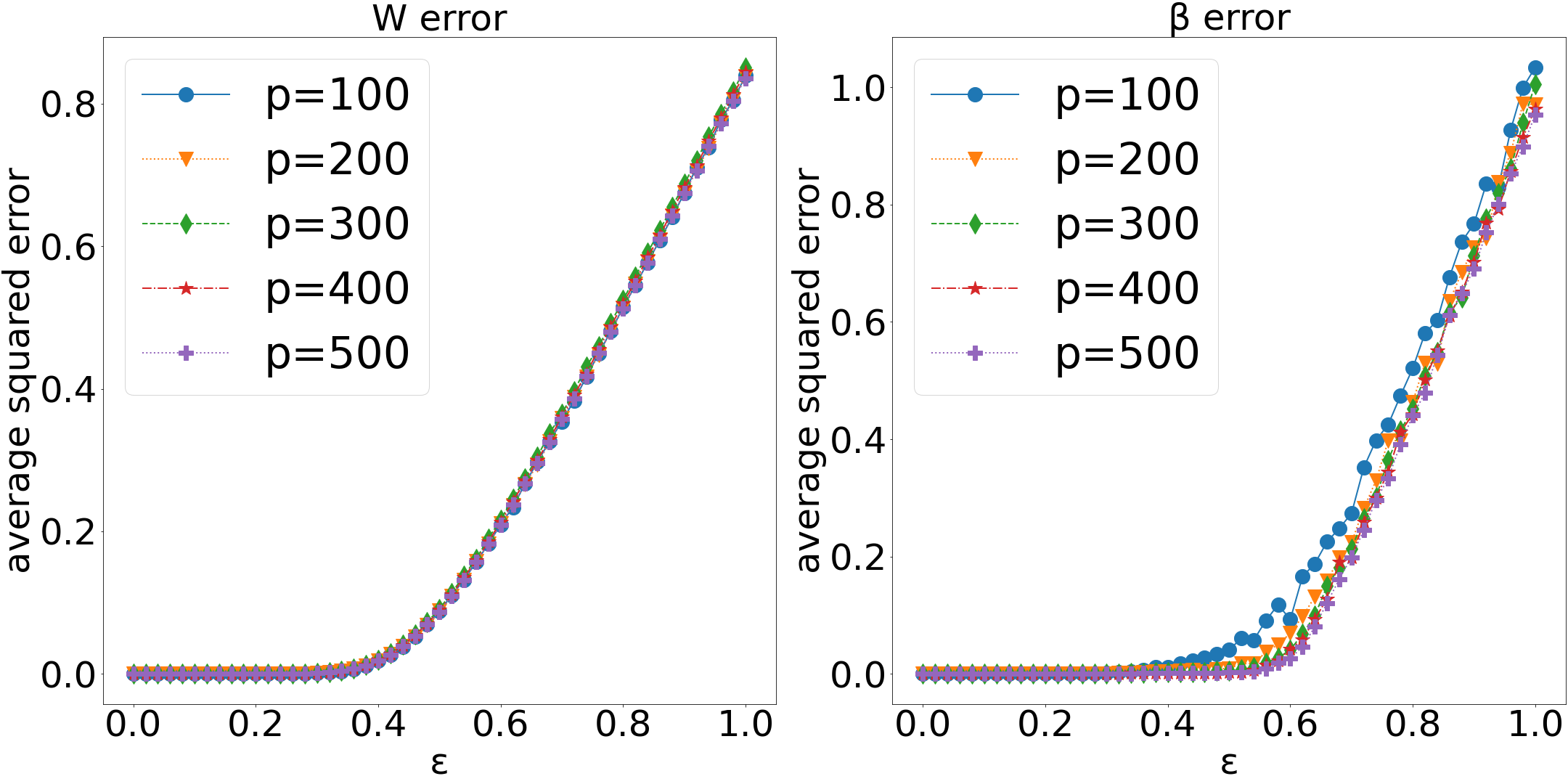}
             \caption{$k$=100}
  \end{subfigure}
    \begin{subfigure}{0.45\textwidth}
             \includegraphics[trim=0 0 0 0,clip,width=1\textwidth]{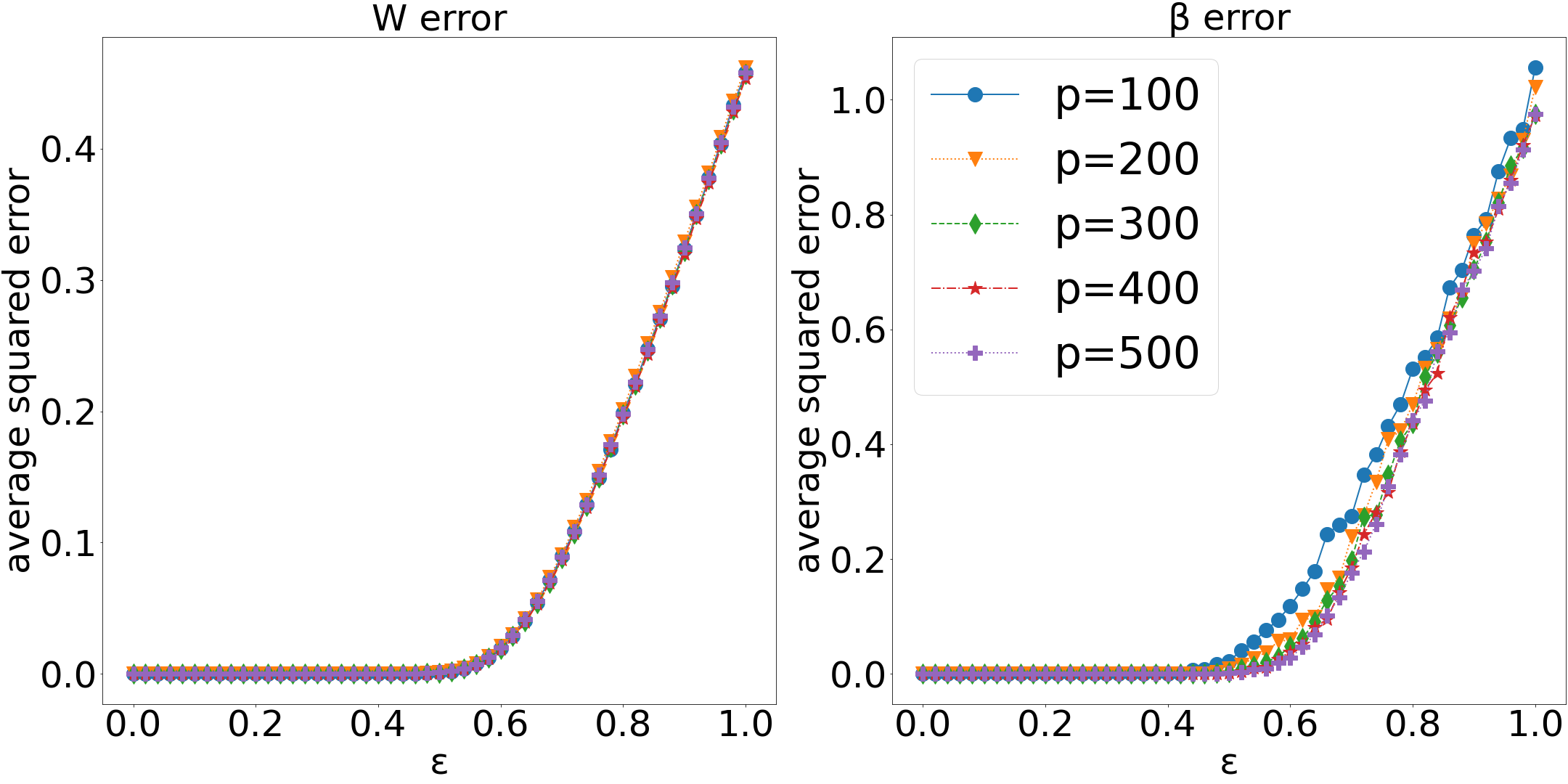} 
             \caption{$k$=150}
  \end{subfigure}
  \caption{{\bf Increasing $p$ and $k$ promotes the recovery performance ($n = 5, m=5$) on synthetic data} Experiments under setting in Theorem \ref{thm: main2} setting. When $p$ increases, the limit of $\epsilon$ for successful recovery of $\beta$ also increases. When $k$ increases, the limit of $\epsilon$ for successful recovery of $W$ increases.} 
  \label{fig: syn_p_k_main2}  
\end{figure}

\begin{figure}[h]
\centering
  \begin{subfigure}{0.45\textwidth}
             \includegraphics[trim=0 0 0 0,clip,width=1\textwidth]{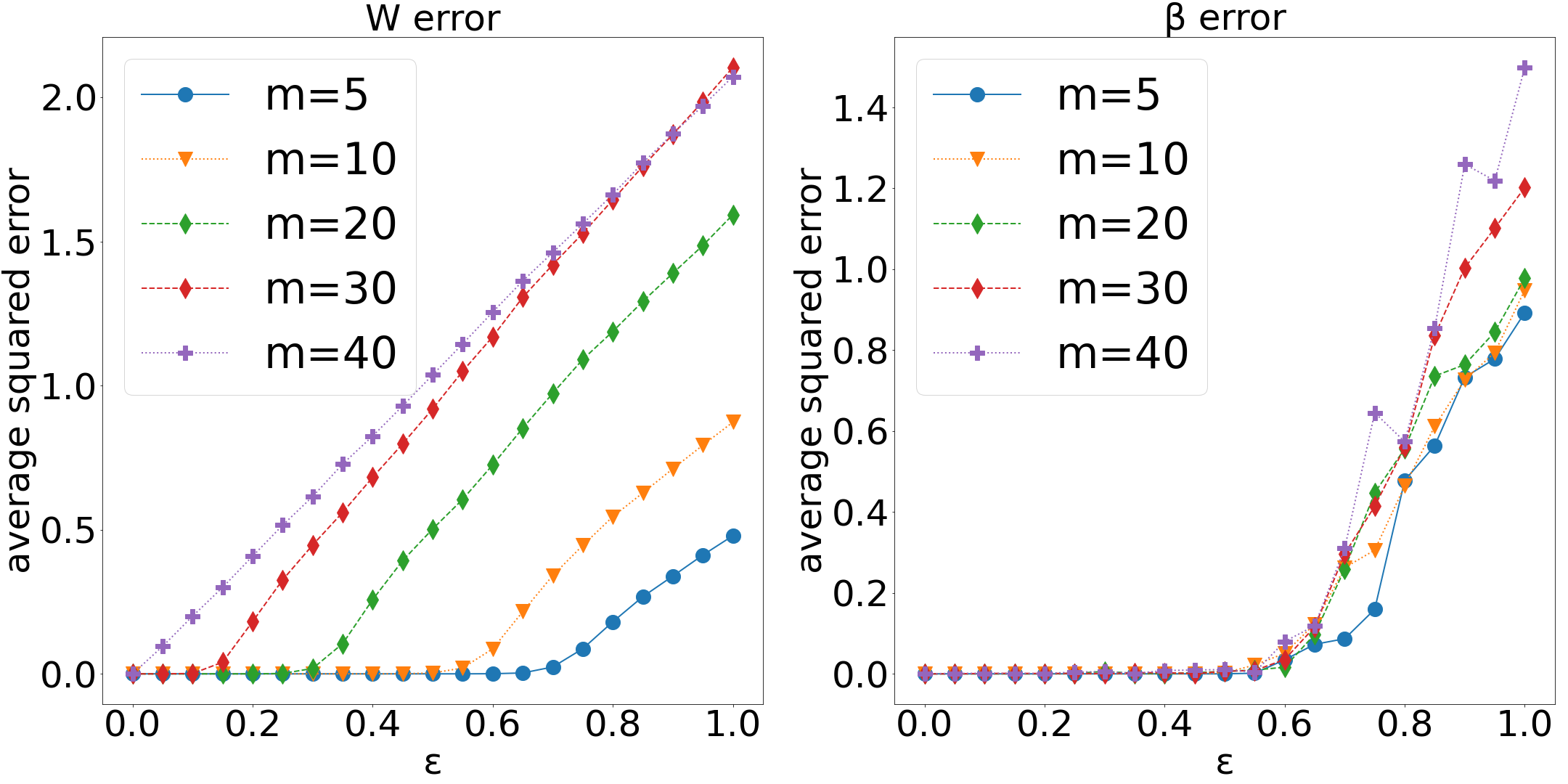}
  \end{subfigure}
  \caption{{\bf Decreasing $m$ promotes the recovery performance ($n = 5, p = 500, k = 200$) on synthetic data}. Experiments under settings in Theorem \ref{thm: main1} setting. When $m$ decreases, the limit of $\epsilon$ for successful recovery of both $W$ and $\beta$ also increases.}
  \label{fig: syn_m}
\end{figure} 

\begin{figure}[ht]
\centering
  \begin{subfigure}{0.45\textwidth}
             \includegraphics[trim=0 0 0 0,clip,width=1\textwidth]{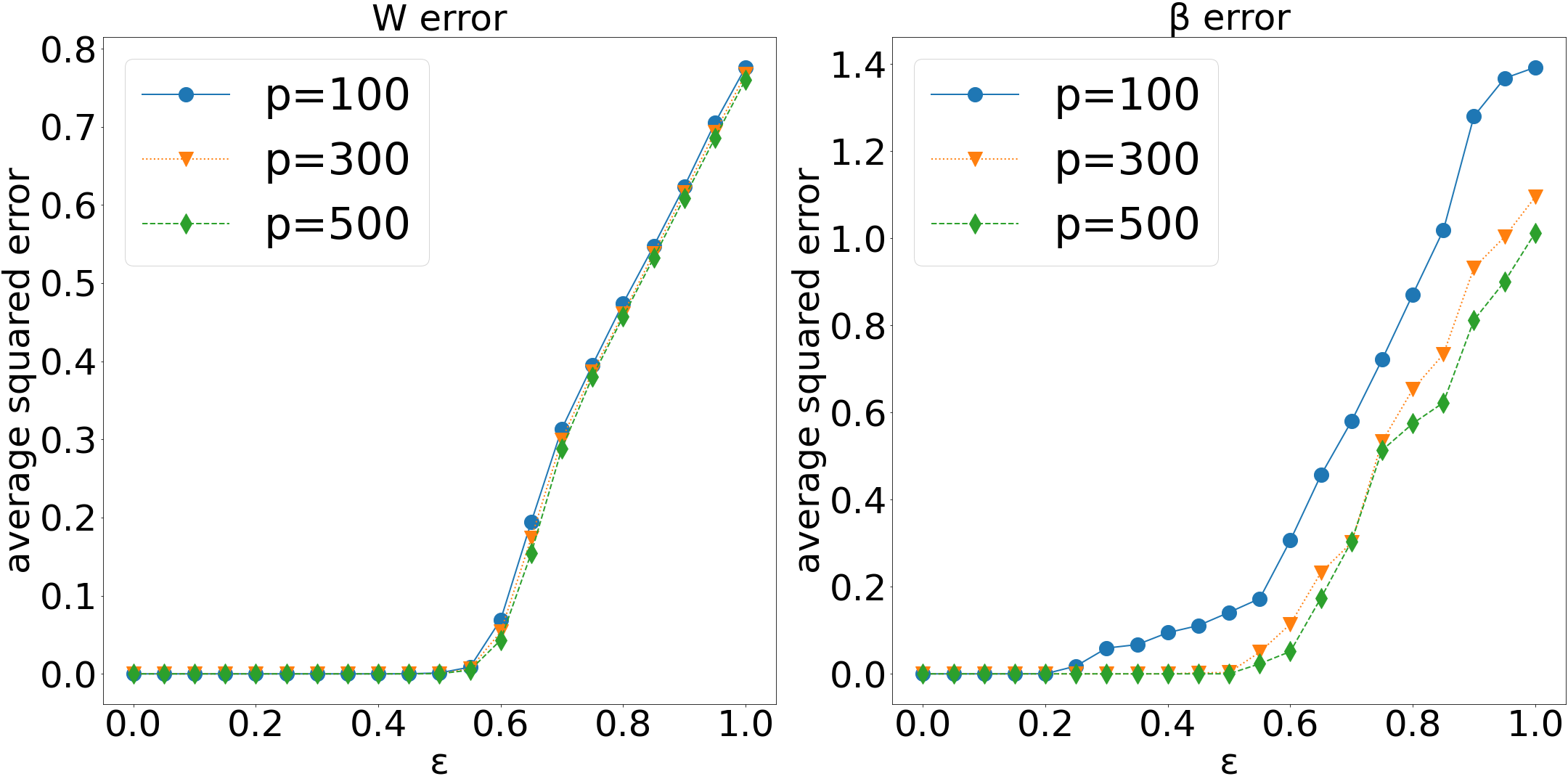}
             \caption{$m=2,k=392$}
  \end{subfigure}
  \begin{subfigure}{0.45\textwidth}
             \includegraphics[trim=0 0 0 0,clip,width=1\textwidth]{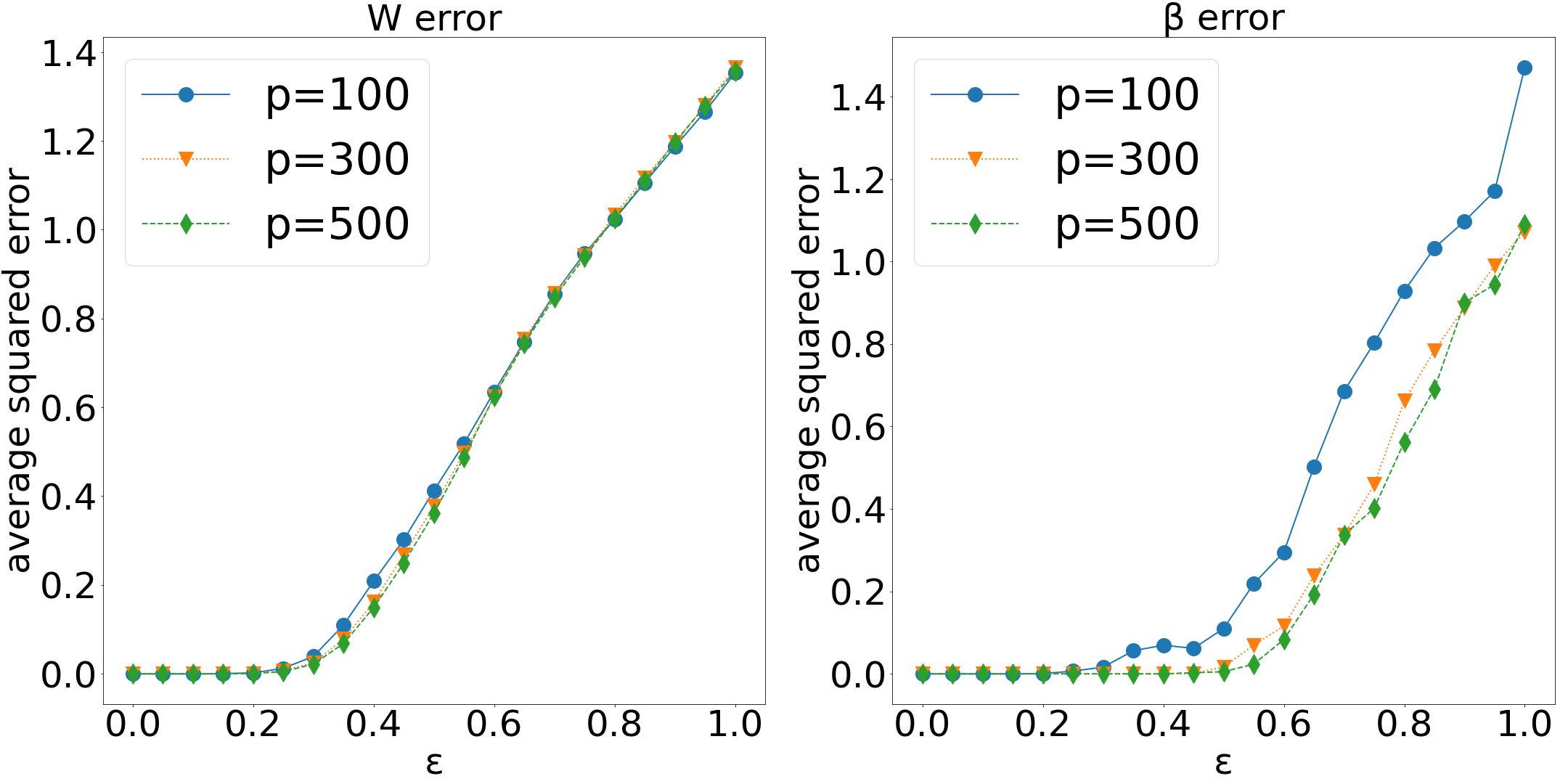}
             \caption{$m=4,k=196$}
  \end{subfigure}
    \begin{subfigure}{0.45\textwidth}
             \includegraphics[trim=0 0 0 0,clip,width=1\textwidth]{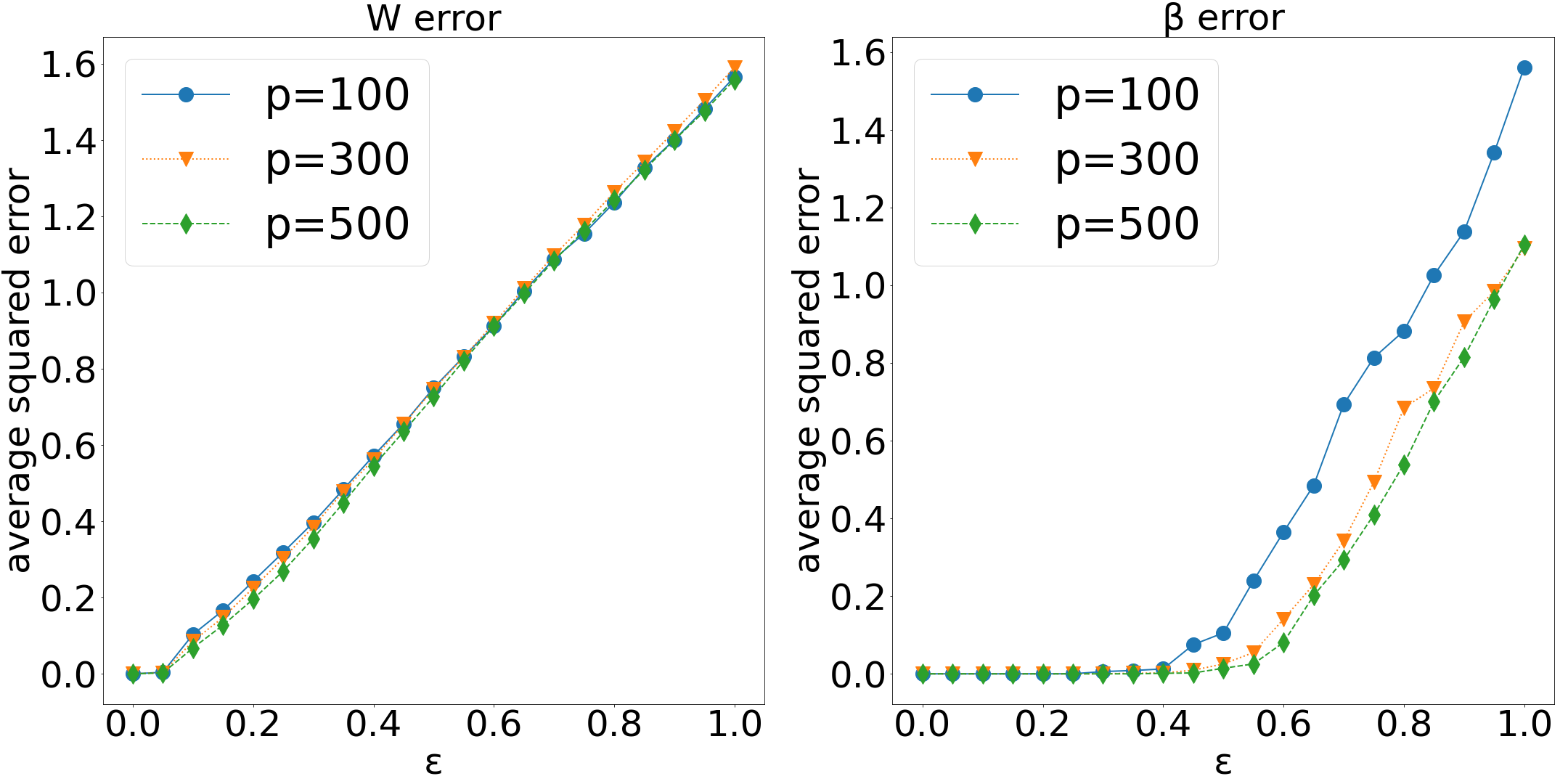} 
             \caption{$m=7,k=112$}
  \end{subfigure}
  \caption{{\bf Increasing $p$ promotes the recovery performance ($n = 21$) on MNIST dataset}. Experiments under the setting in Theorem \ref{thm: main1}. When $p$ increases, the limit of $\epsilon$ for successful recovery of $\beta$ also increases} 
  \label{fig: MNIST_p}
\end{figure} 

\begin{figure}[h]
  \centering
  \begin{subfigure}{0.45\textwidth}
  \centering
             \includegraphics[trim=0 0 0 0,clip,width=1\textwidth]{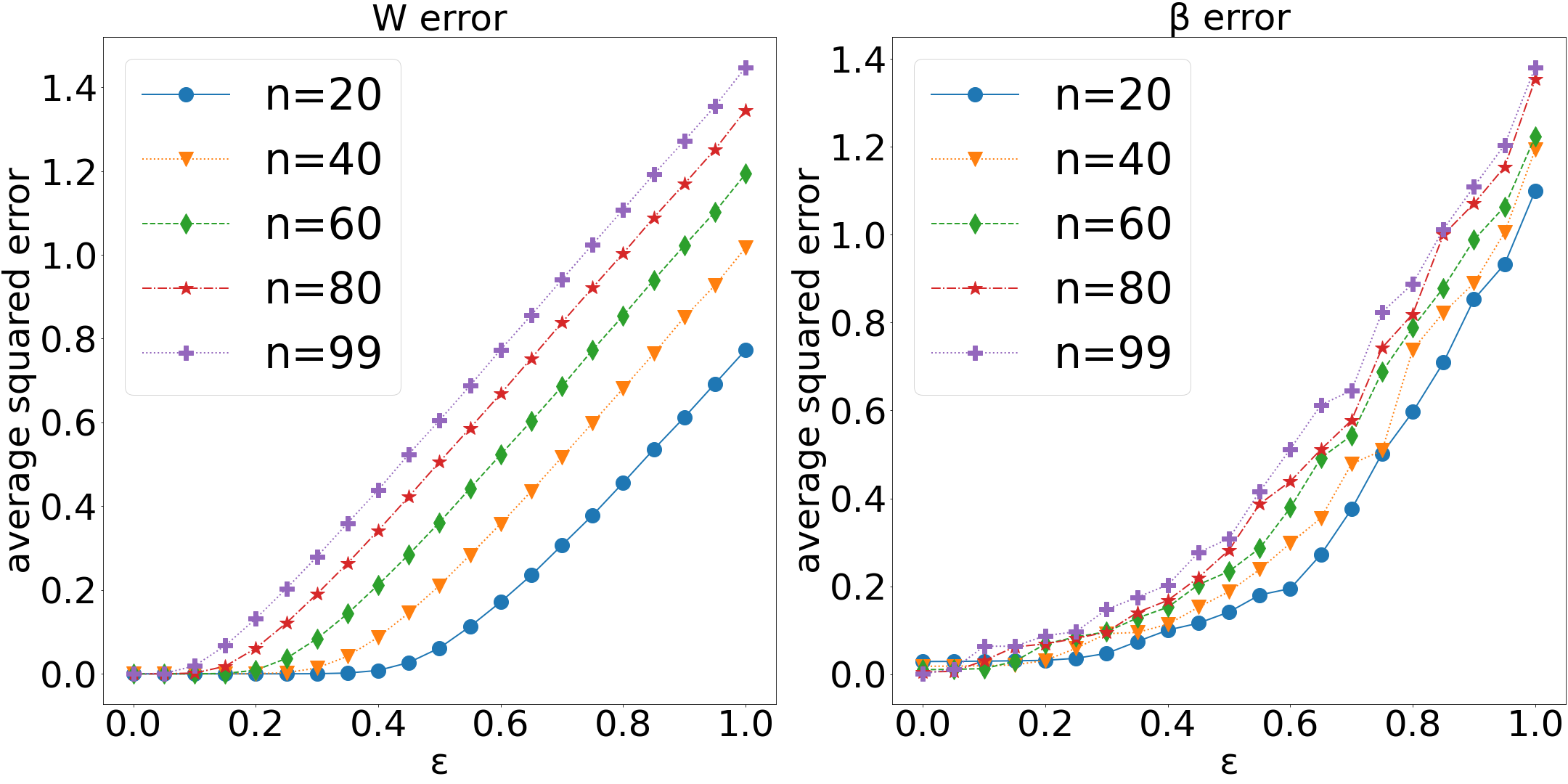}
             \caption{\footnotesize{CNN purification by training instances} }
  \end{subfigure}
    \begin{subfigure}{0.45\textwidth}
    \centering
             \includegraphics[trim=0 0 0 0,clip,width=1\textwidth]{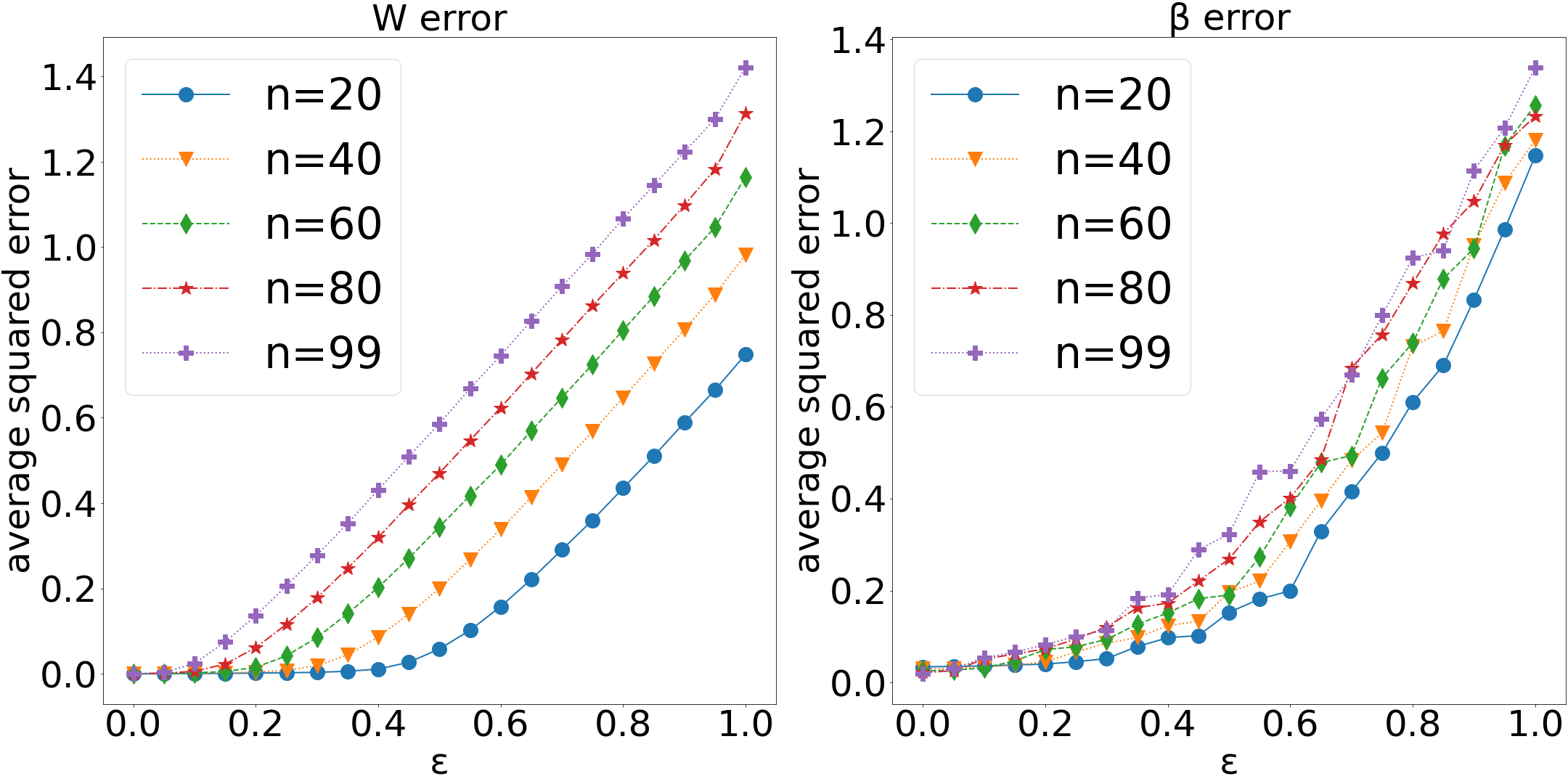} 
             \caption{\footnotesize{Model purification by non-training instances} }
  \end{subfigure}
  \caption{{\bf Our CNN purification method has the ability to yield good performance even when using a limited number of clean data points, which may not necessarily originate from the MNIST training dataset ($\textit{training batch size} = 99$)}. Experiments under settings in Theorem \ref{thm: main1}. When $n$ decreases, the limit of $\epsilon$ for successful recovery of both $W$ and $\beta$ also increases.} 
  \label{fig: MNIST_part_random}
\end{figure}

\begin{figure}[h]
\centering
  \begin{subfigure}{0.45\textwidth}
  \centering
             \includegraphics[trim=0 0 0 0,clip,width=1\textwidth]{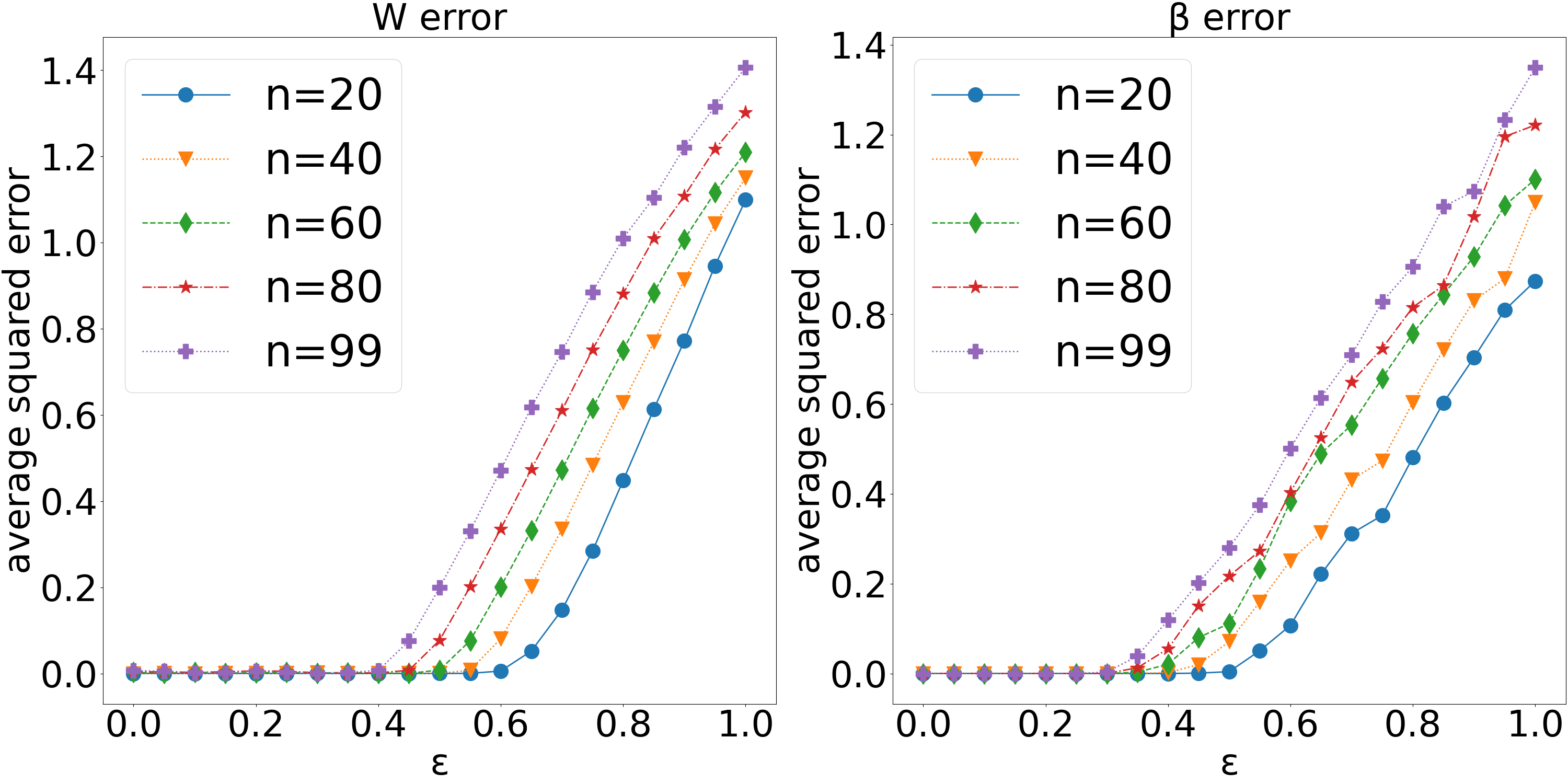}
             \caption{CNN purification by training instances}
  \end{subfigure}
  \begin{subfigure}{0.45\textwidth}
  \centering
             \includegraphics[trim=0 0 0 0,clip,width=1\textwidth]{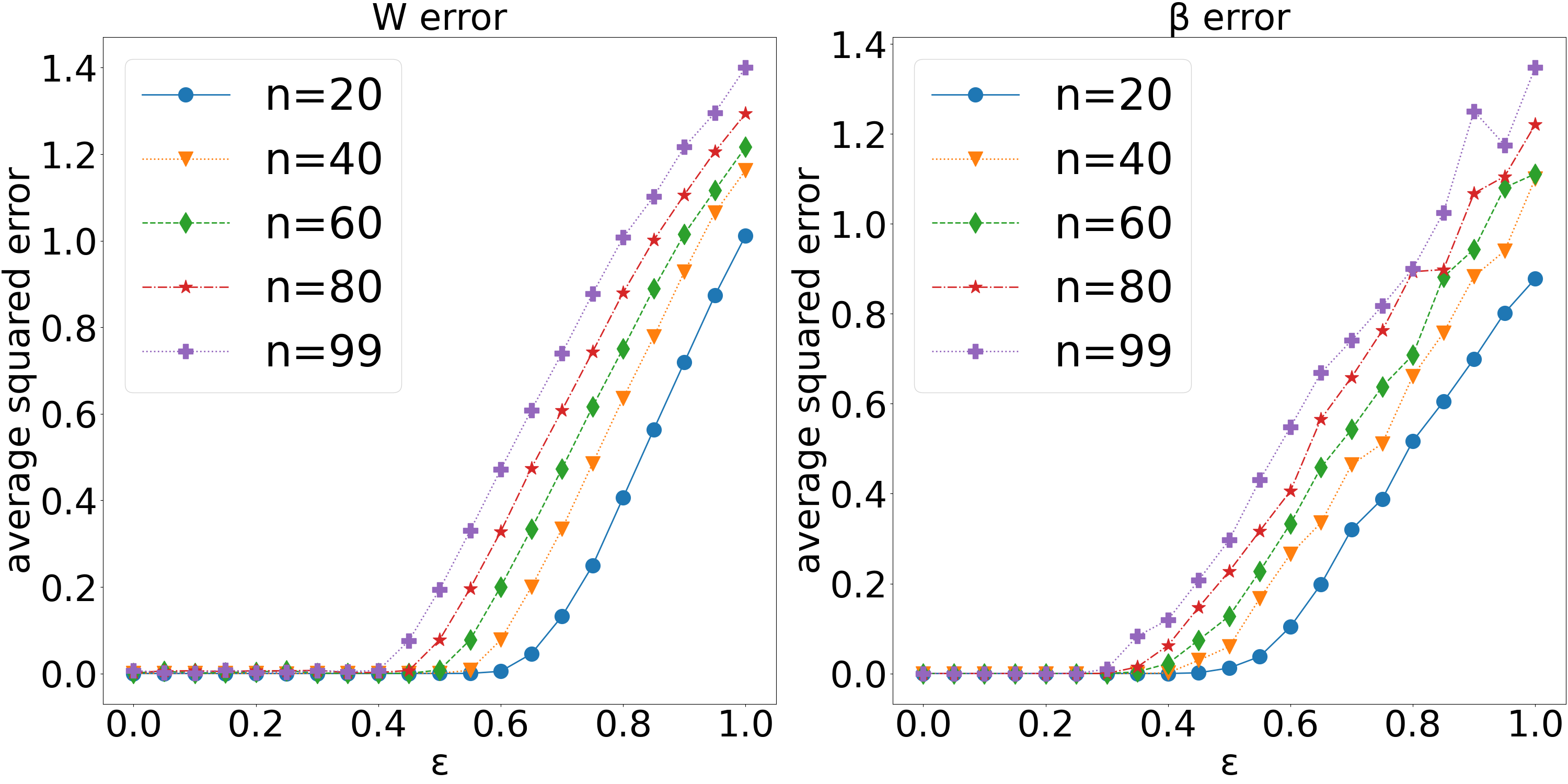}
             \caption{Model purification by non-training instances}
  \end{subfigure}  
  \caption{{\bf Our CNN purification method has the ability to yield good performance even when using a limited number of clean data points, which may not necessarily originate from the CIFAR-10 training dataset ($\textit{training batch size} = 99$)}. Experiments under settings in Theorem \ref{thm: main1}. When $n$ decreases, the limit of $\epsilon$ for successful recovery of both $W$ and $\beta$ also increases.} 
  \label{fig: CIFAR_part_random}  
\end{figure}

\begin{figure}[h]
\centering
  \begin{subfigure}{0.45\textwidth}
  \centering
             \includegraphics[trim=0 0 0 0,clip,width=1\textwidth]{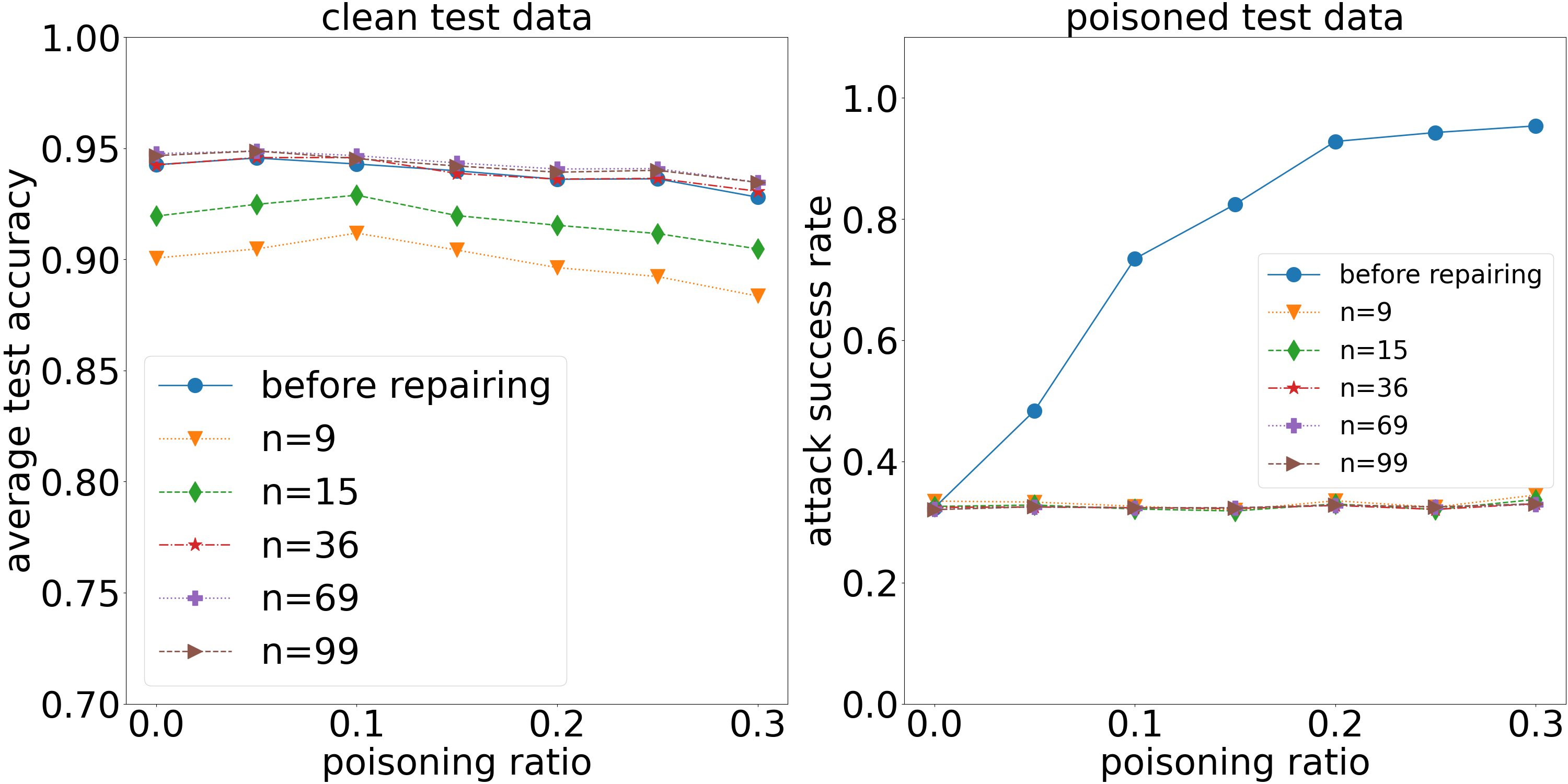}
             \caption{Mitigating poisoning attack by training instances}
  \end{subfigure}
  \begin{subfigure}{0.45\textwidth}
  \centering
             \includegraphics[trim=0 0 0 0,clip,width=1\textwidth]{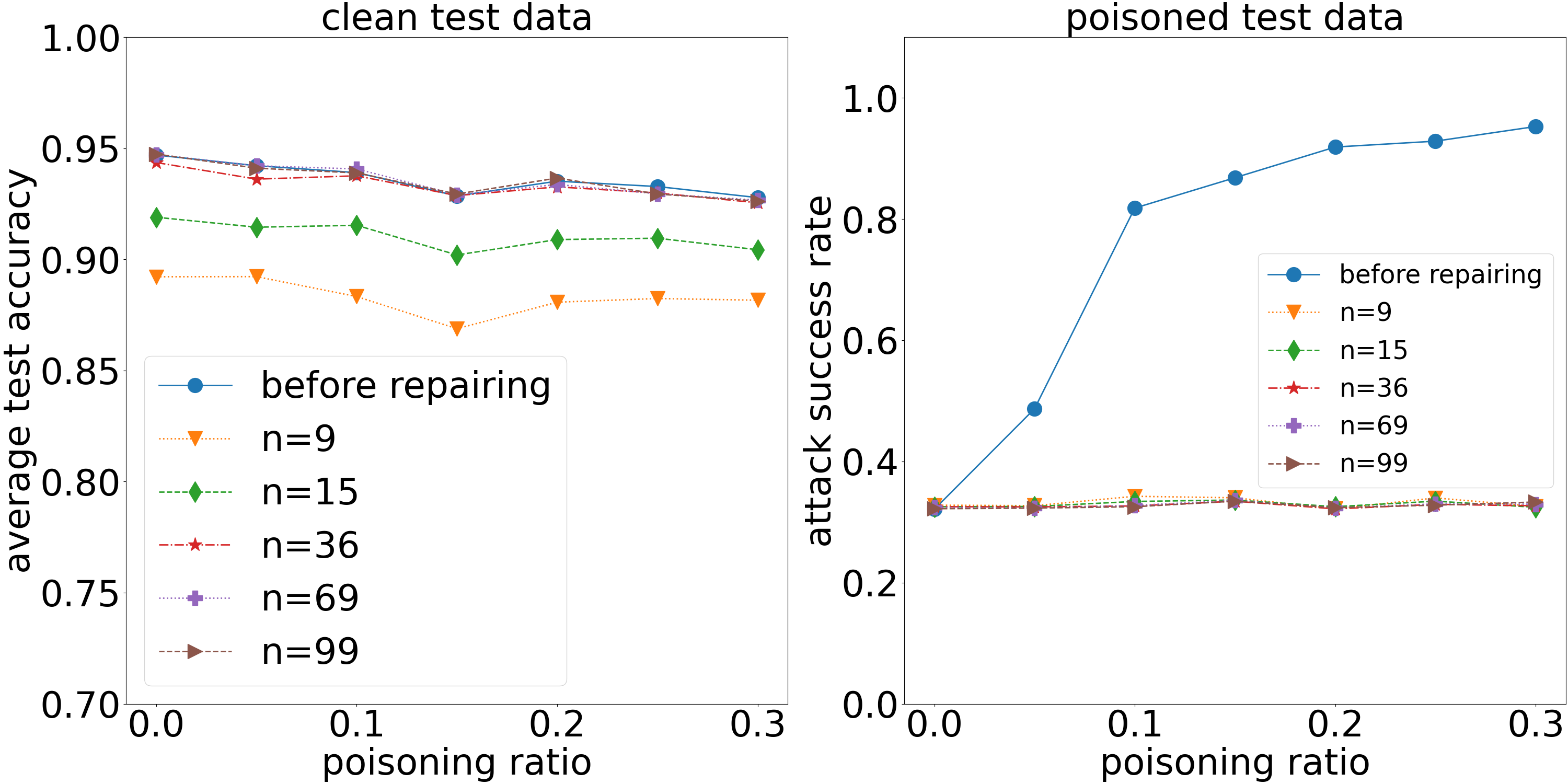}
             \caption{Mitigating poisoning attack  by non-training instances}
  \end{subfigure}  
  \caption{{\bf Even with a small number of clean data points, CNN purification can mitigate the poisoning effect ($\textit{training batch size} = 99$) on MNIST}. Experiments under settings in Theorem \ref{thm: main1}. The poisoned ratio indicates the percentage of the poisoned training data. The attack success rate is the percentage of test data that has been successively attacked.} 
  \label{fig: MNIST_backdoor}  
\end{figure}

\begin{figure}[h]
\centering
  \begin{subfigure}{0.45\textwidth}
  \centering
             \includegraphics[trim=0 0 0 0,clip,width=1\textwidth]{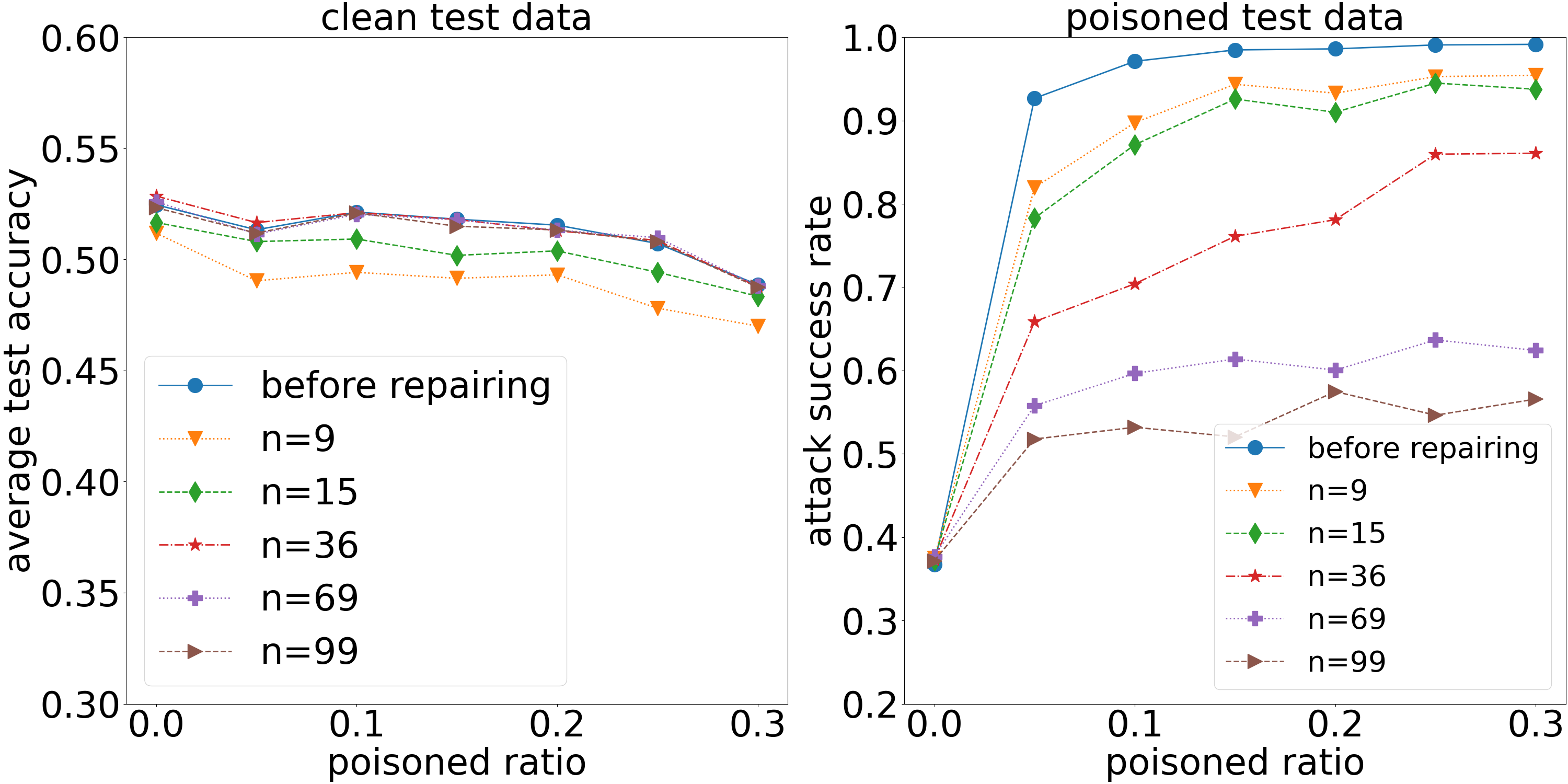}
             \caption{Mitigating poisoning attack by training instances}
  \end{subfigure}
  \begin{subfigure}{0.45\textwidth}
  \centering
             \includegraphics[trim=0 0 0 0,clip,width=1\textwidth]{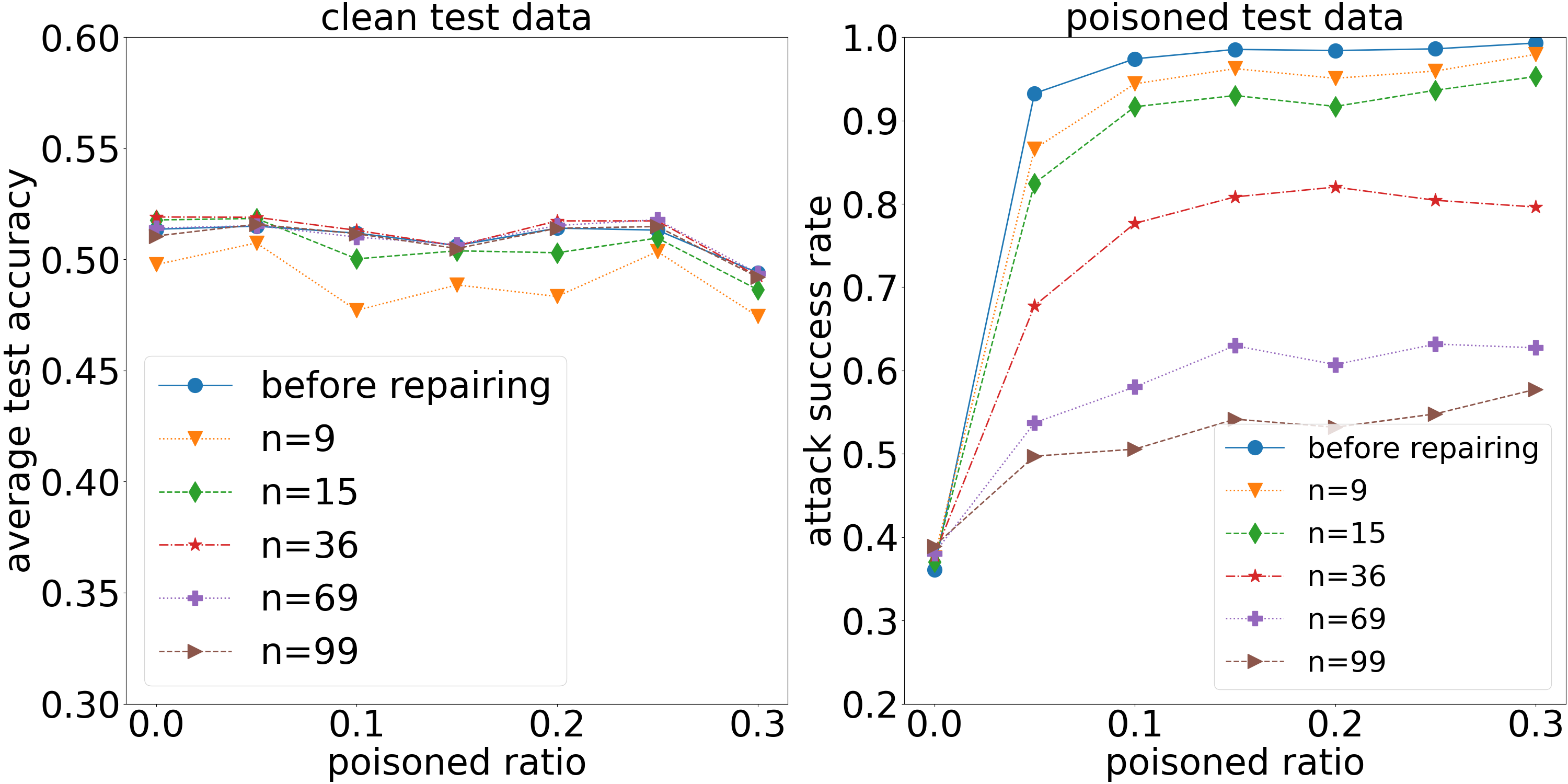}
             \caption{Mitigating poisoning attack  by non-training instances}
  \end{subfigure}  
  \caption{{\bf With a portion of clean data points, CNN purification can mitigate the poisoning effect ($\textit{training batch size} = 99$) on CIFAR-10}. The poisoned ratio indicates the percentage of the poisoned training data. The attack success rate is the percentage of test data that has been successively attacked.} 
  \label{fig: CIFAR_backdoor}  
  %\vspace{-2mm}
\end{figure}

%\begin{figure}[t]
%\centering
%     \includegraphics[trim=0 0 0 0,clip,width=0.5\textwidth,height=0.6\textheight]{Figs/MNIST_multilayer_part.png}
%  \caption{{\bf Even with multiple hidden layers, CNN purification can mitigate the poisoning effect ($\textit{training batch size} = 99$ , $\textit{repair data size = 36}$) on MNIST}. $l$ denotes the number of hidden layers in CNN. 'Before' and 'After' represents 'before CNN purification' and 'after CNN purification'. Here CNN is purified by \textbf{training instances}. } 
%  \label{fig: MNIST_multilayer_part}  
%\end{figure}

\begin{figure}[h]
%\vspace{-2mm}
\centering
  \begin{subfigure}{0.43\textwidth}
             \includegraphics[trim=0 0 0 0,clip,width=1\textwidth]{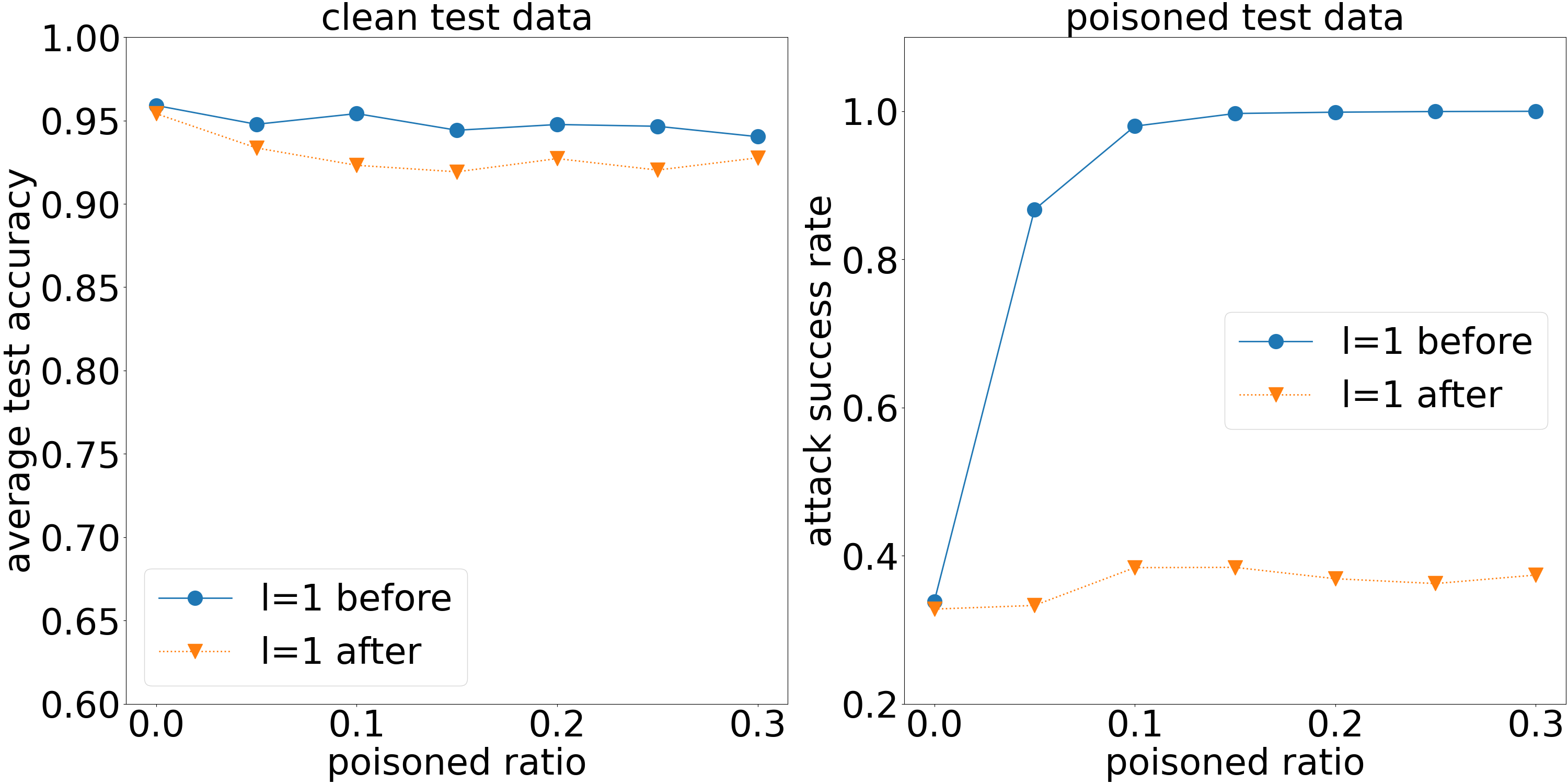}
  \end{subfigure}
  \begin{subfigure}{0.43\textwidth}
             \includegraphics[trim=0 0 0 0,clip,width=1\textwidth]{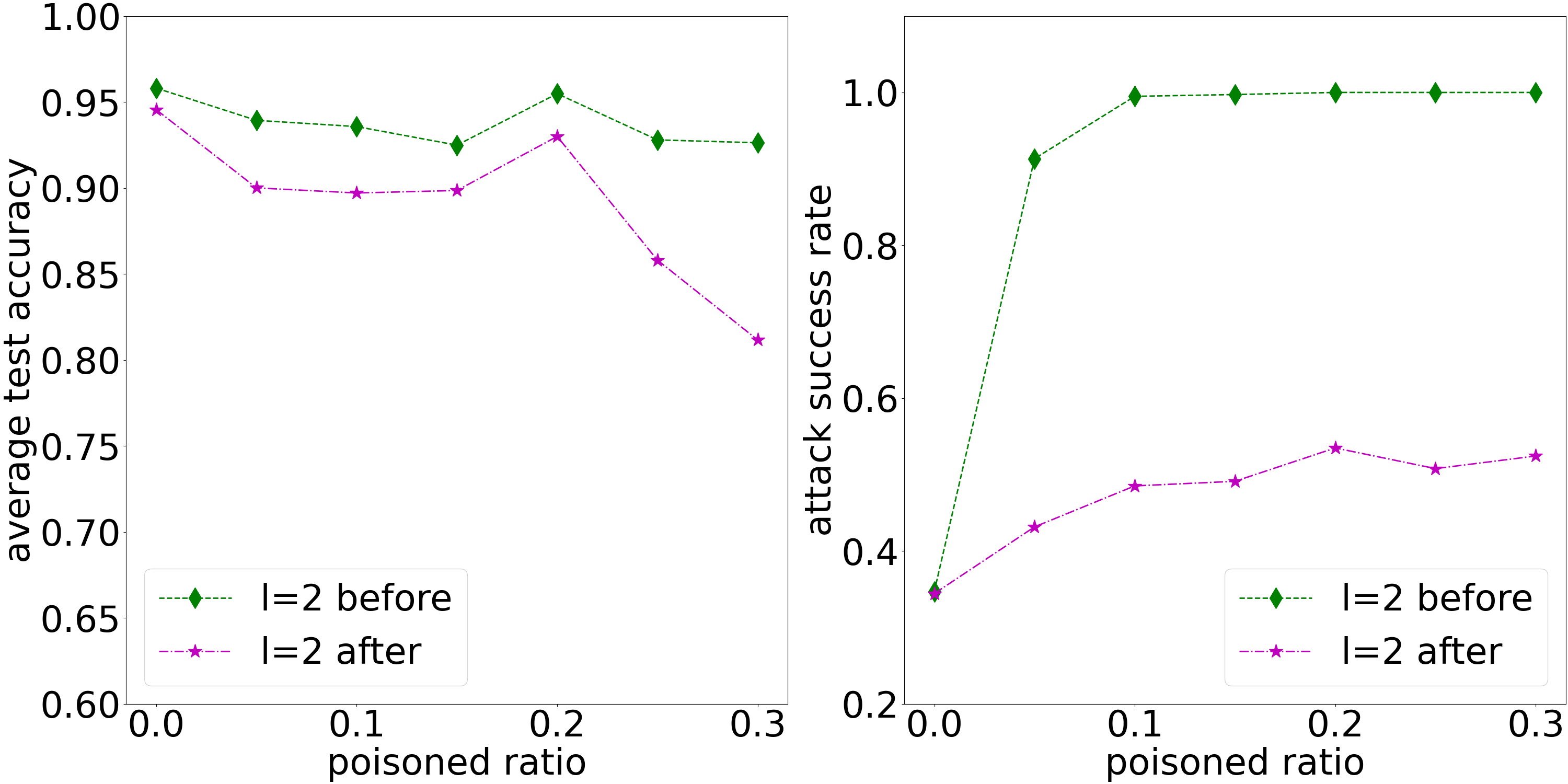}
  \end{subfigure}
  %   \begin{subfigure}{0.42\textwidth}
  %            \includegraphics[trim=0 0 0 0,clip,width=1\textwidth]{Figs/MNIST_multilayer_part_l3.png} 
  % \end{subfigure}
  \caption{{\bf Even with multiple hidden layers, CNN purification can mitigate the poisoning effect on MNIST}.  Here CNN is purified by training instances.} 
  \label{fig: MNIST_multilayer_part}  
\end{figure} 

%\begin{figure}[t]
%\centering
%     \includegraphics[trim=0 0 0 0,clip,width=0.5\textwidth, height=0.6\textheight]{Figs/MNIST_multilayer_random.png}
%  \caption{{\bf Even with multiple hidden layers, CNN purification can mitigate the poisoning effect with some compromise on average test accuracy($\textit{training batch size} = 99$ , $\textit{repair data size = 36}$) on MNIST}. $l$ denotes the number of hidden layers in CNN. 'Before' and 'After' represents 'before CNN purification' and 'after CNN purification'. Here CNN is purified by \textbf{non-training instances}.} 
%  \label{fig: MNIST_multilayer_random}  
%\end{figure}

\begin{figure}[h]
\centering
  \begin{subfigure}{0.45\textwidth}
             \includegraphics[trim=0 0 0 0,clip,width=1\textwidth]{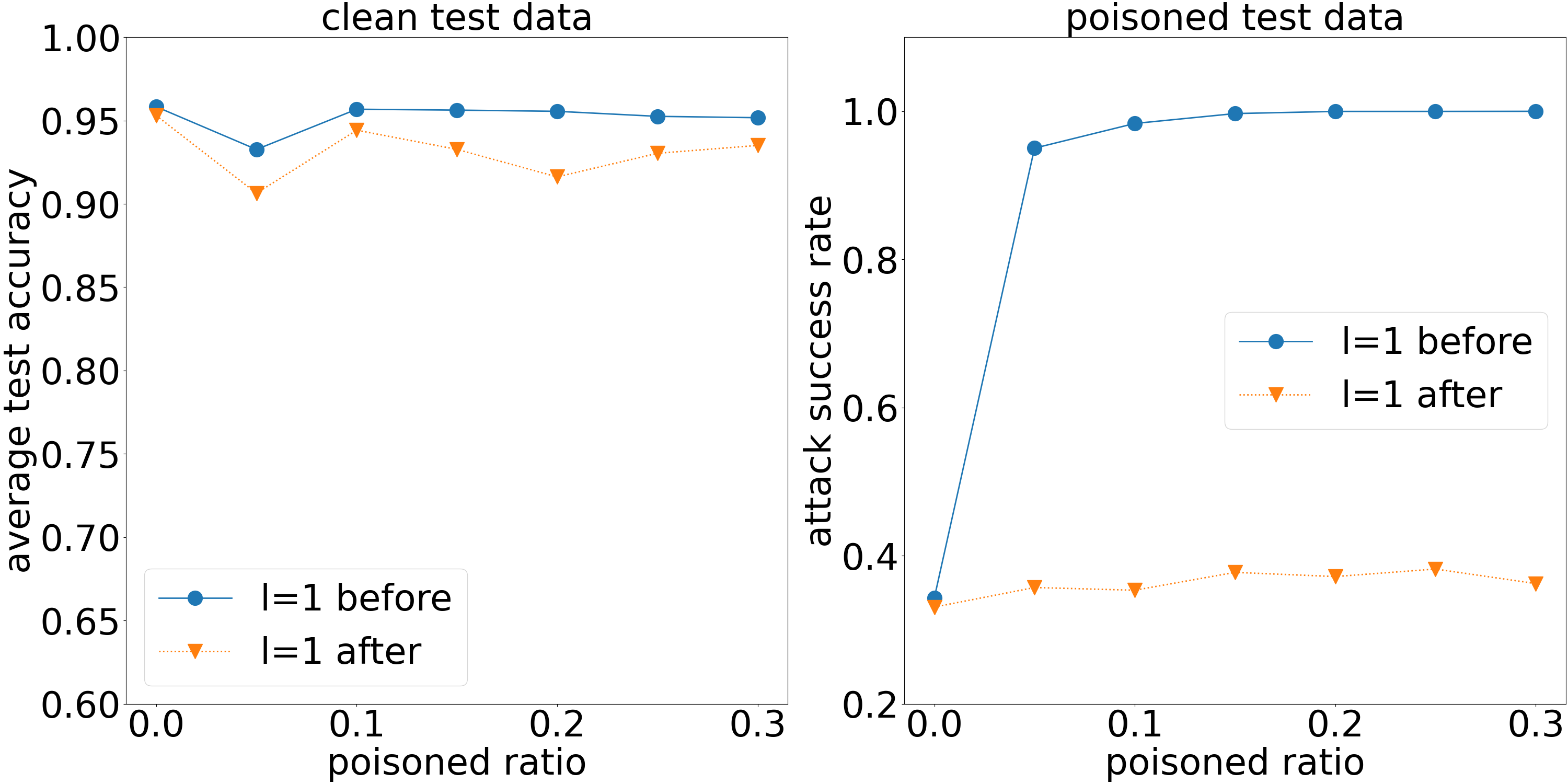}
  \end{subfigure}
  \begin{subfigure}{0.45\textwidth}
             \includegraphics[trim=0 0 0 0,clip,width=1\textwidth]{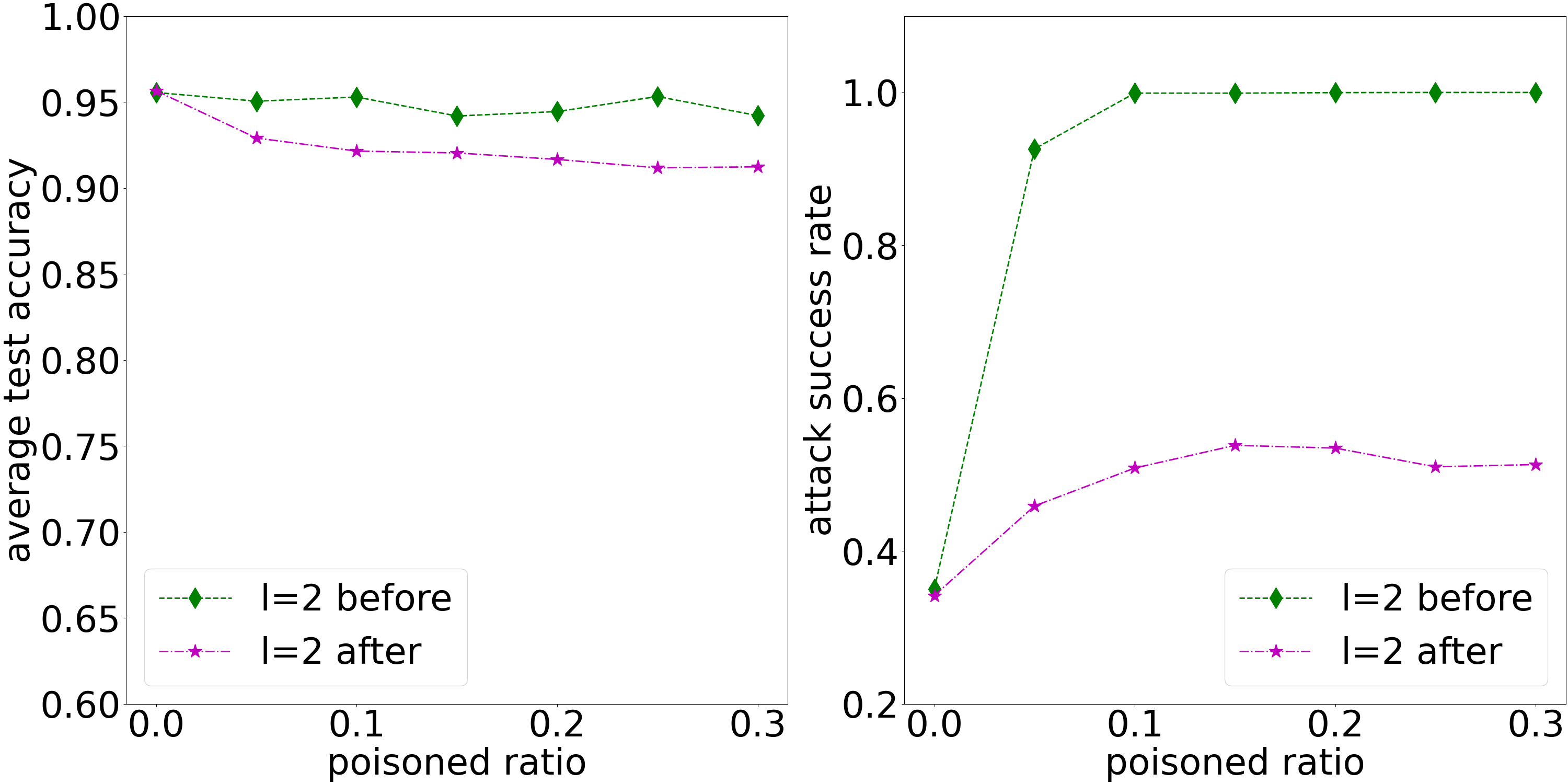}
  \end{subfigure}
  %   \begin{subfigure}{0.45\textwidth}
  %            \includegraphics[trim=0 0 0 0,clip,width=1\textwidth]{Figs/MNIST_multilayer_random_l3.png} 
  % \end{subfigure}
  \caption{{\bf Even with multiple hidden layers, CNN purification can mitigate the poisoning effect with some compromise on average test accuracy ($\textit{training batch size} = 99$ , $\textit{repair data size = 36}$) on MNIST}. Here CNN is purified by non-training instances.} 
  %\caption{{\bf Even with multiple hidden layers, CNN purification can mitigate the poisoning effect with some compromise on average test accuracy($\textit{training batch size} = 99$ , $\textit{repair data size = 36}$) on MNIST}. $l$ denotes the number of hidden layers in CNN. 'Before' and 'After' represents 'before CNN purification' and 'after CNN purification'. Here CNN is purified by \textbf{non-training instances}.} 
  \label{fig: MNIST_multilayer_random}  
\end{figure} 

%\begin{figure}[t]
%\centering
%     \includegraphics[trim=0 0 0 0,clip,width=0.5\textwidth,height=0.6\textheight]{Figs/CIFAR_multilayer_part.png}
%  \caption{{\bf Even with multiple hidden layers, CNN purification can mitigate the poisoning effect ($\textit{training batch size} = 99$ , $\textit{repair data size = 69}$) on CIFAR}. $l$ denotes the number of hidden layers in CNN. 'Before' and 'After' represents 'before CNN purification' and 'after CNN purification'. Here CNN is purified by \textbf{training instances}. } 
%  \label{fig: CIFAR_multilayer_part}  
%\end{figure}

\begin{figure}[ht]
\centering
  \begin{subfigure}{0.45\textwidth}
             \includegraphics[trim=0 0 0 0,clip,width=1\textwidth]{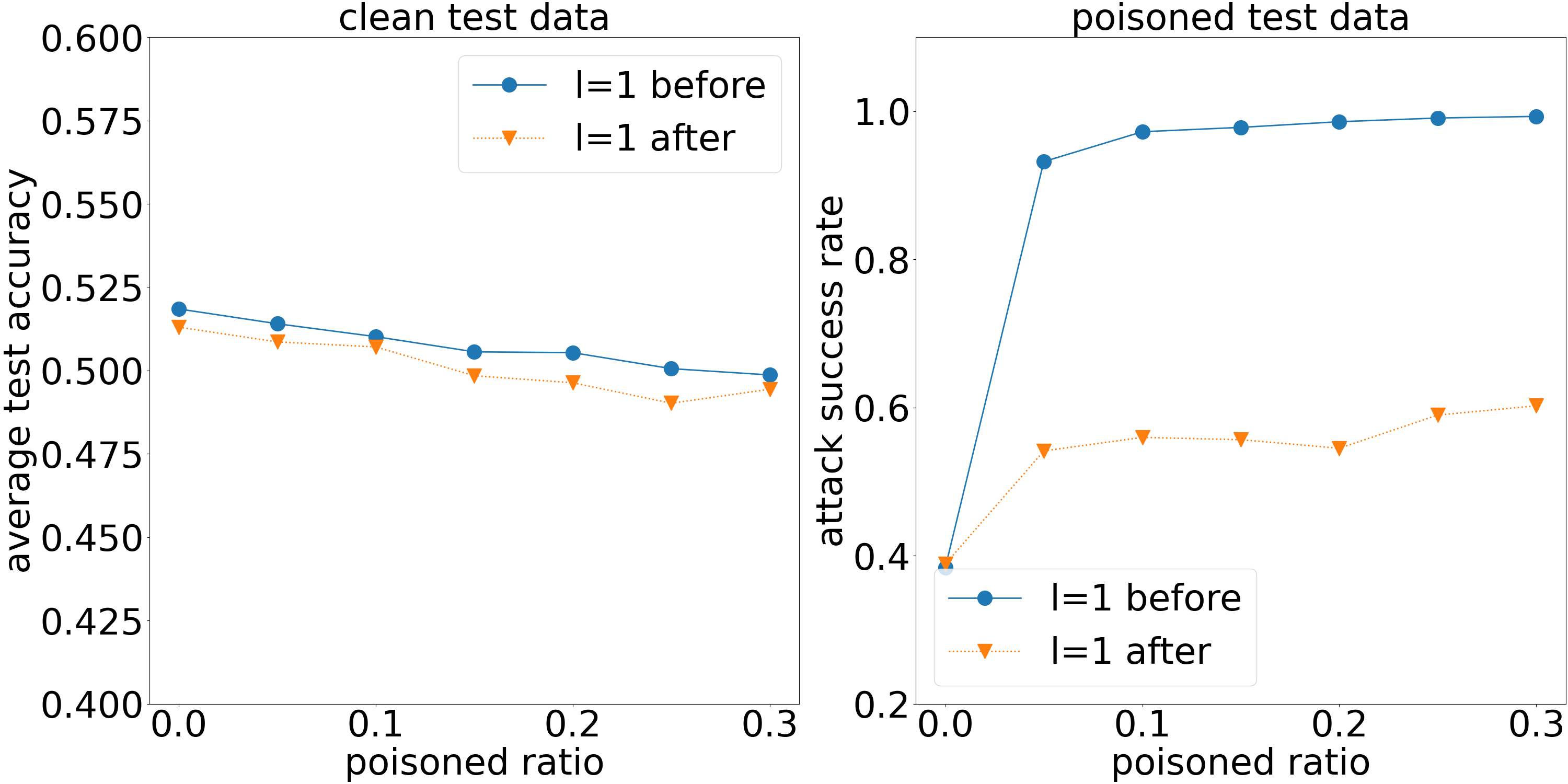}
  \end{subfigure}
  \begin{subfigure}{0.45\textwidth}
             \includegraphics[trim=0 0 0 0,clip,width=1\textwidth]{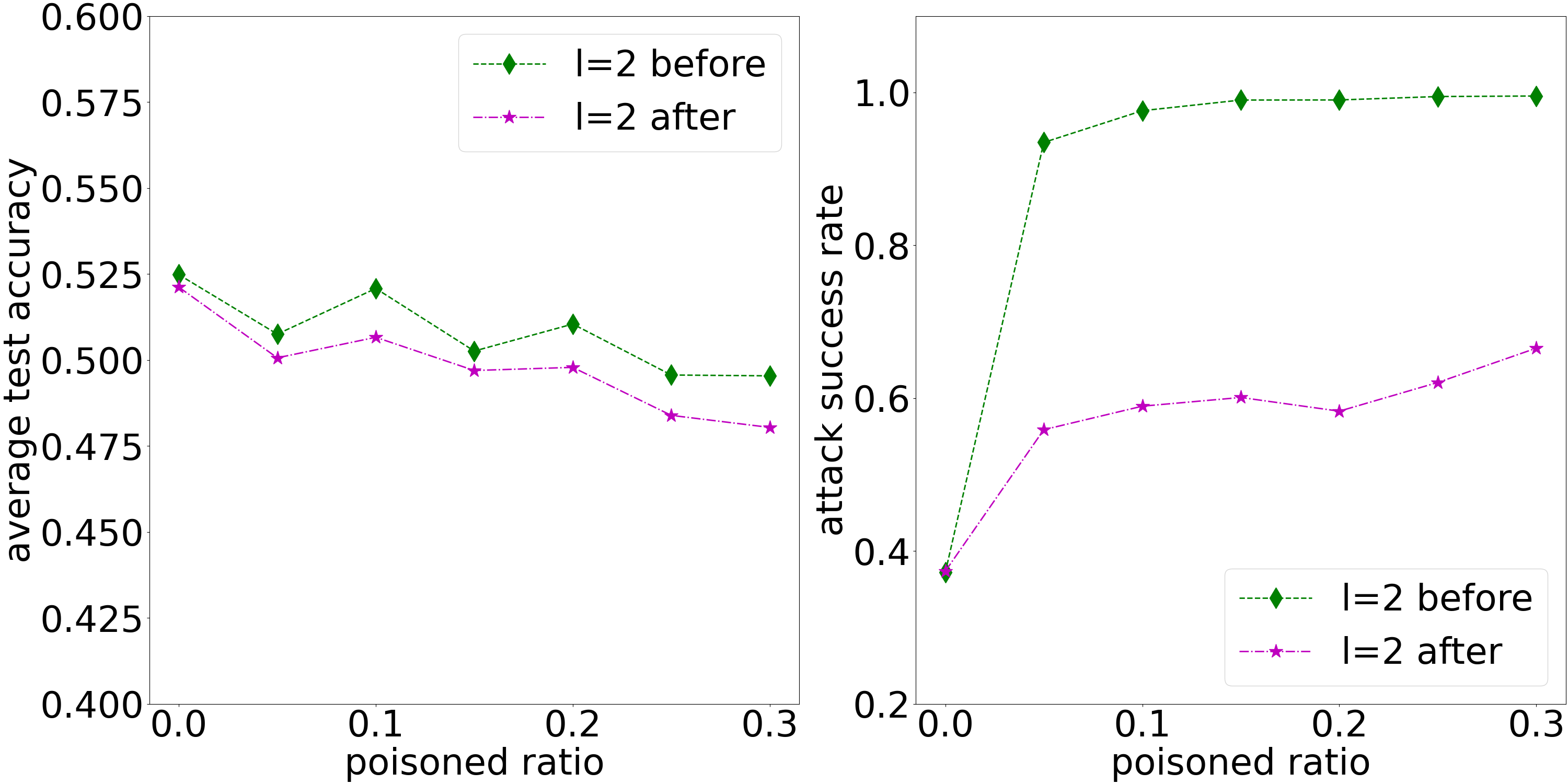}
  \end{subfigure}
    \begin{subfigure}{0.45\textwidth}
             \includegraphics[trim=0 0 0 0,clip,width=1\textwidth]{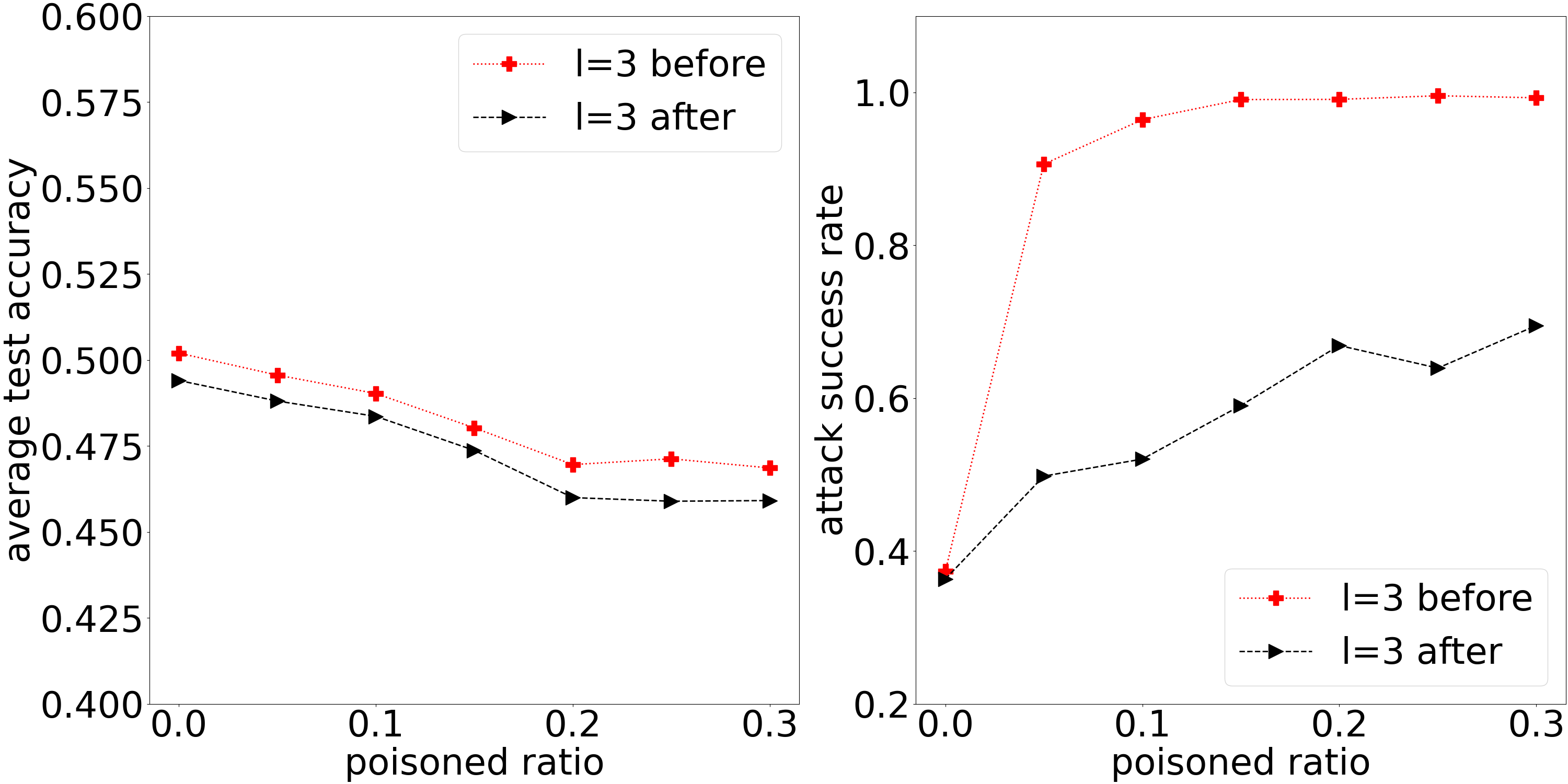} 
  \end{subfigure}
  \caption{{\bf Even with multiple hidden layers, CNN purification can mitigate the poisoning effect ($\textit{training batch size} = 99$ , $\textit{repair data size = 69}$) on CIFAR}. Here CNN is purified by training instances. }
  %\caption{{\bf Even with multiple hidden layers, CNN purification can mitigate the poisoning effect ($\textit{training batch size} = 99$ , $\textit{repair data size = 69}$) on CIFAR}. $l$ denotes the number of hidden layers in CNN. 'Before' and 'After' represents 'before CNN purification' and 'after CNN purification'. Here CNN is purified by \textbf{training instances}. } 
  \label{fig: CIFAR_multilayer_part}  
\end{figure}

\begin{figure}[ht]
\centering
  \begin{subfigure}{0.45\textwidth}
             \includegraphics[trim=0 0 0 0,clip,width=1\textwidth]{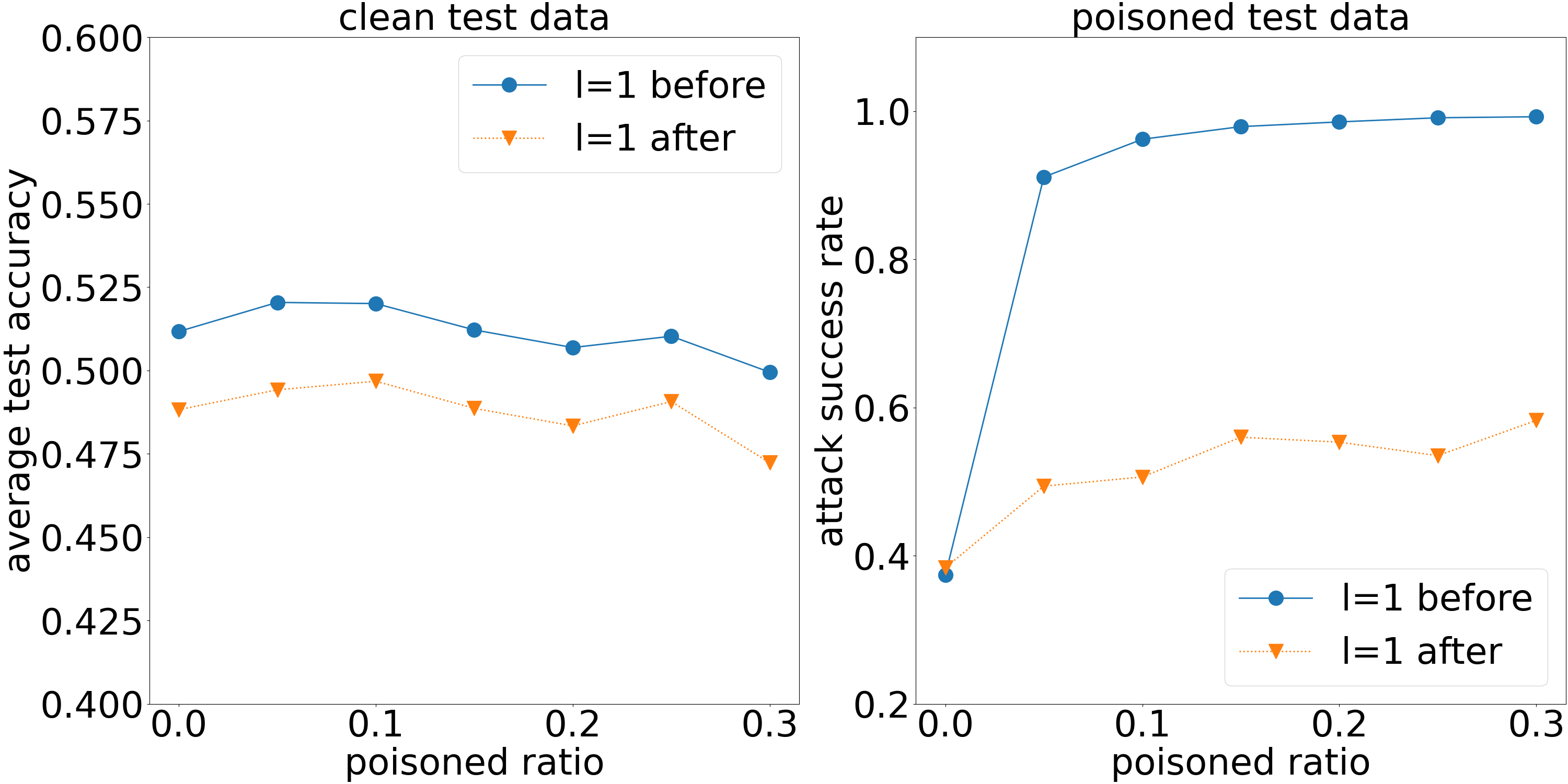}
  \end{subfigure}
  \begin{subfigure}{0.45\textwidth}
             \includegraphics[trim=0 0 0 0,clip,width=1\textwidth]{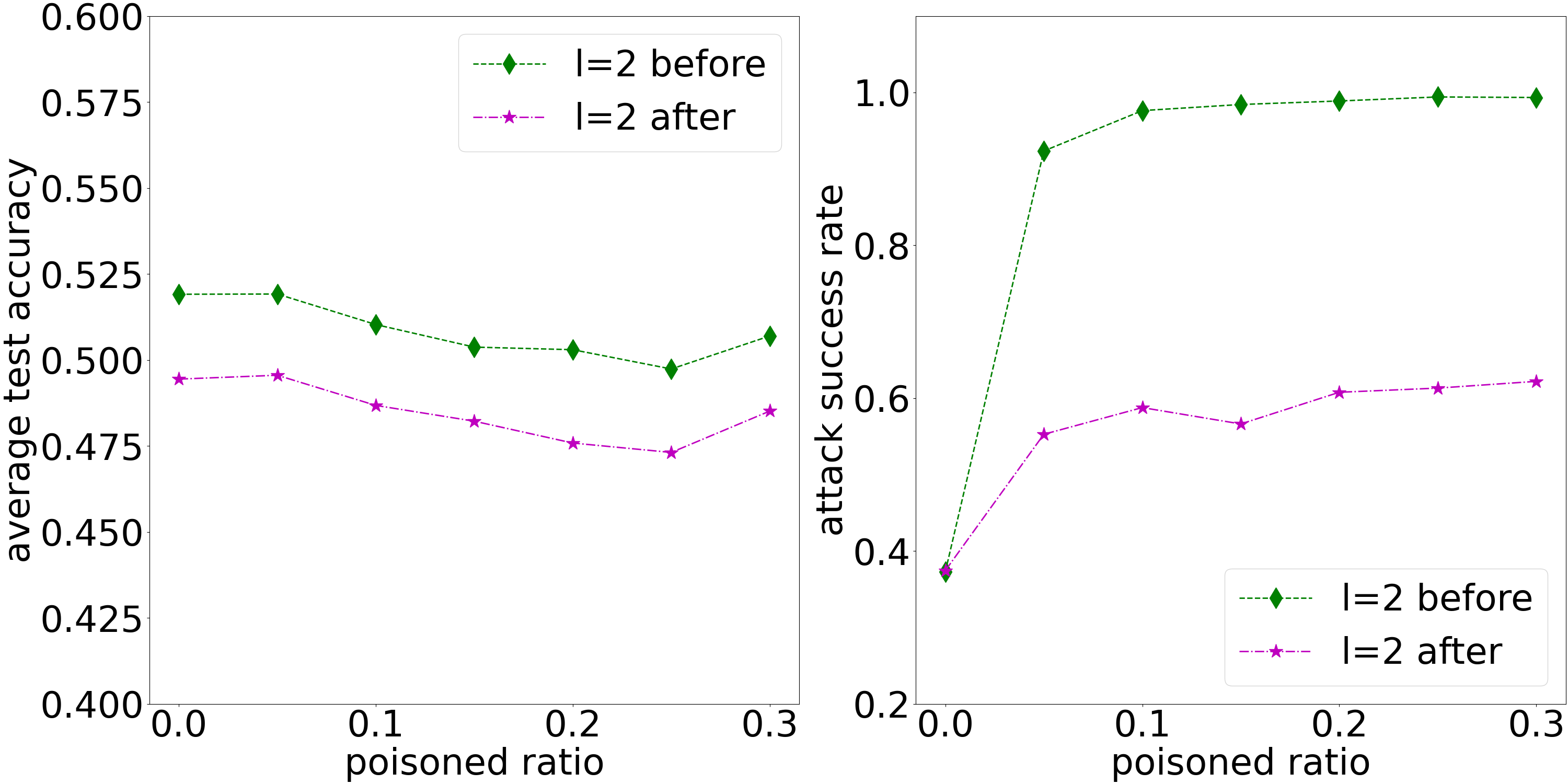}
  \end{subfigure}
    \begin{subfigure}{0.45\textwidth}
             \includegraphics[trim=0 0 0 0,clip,width=1\textwidth]{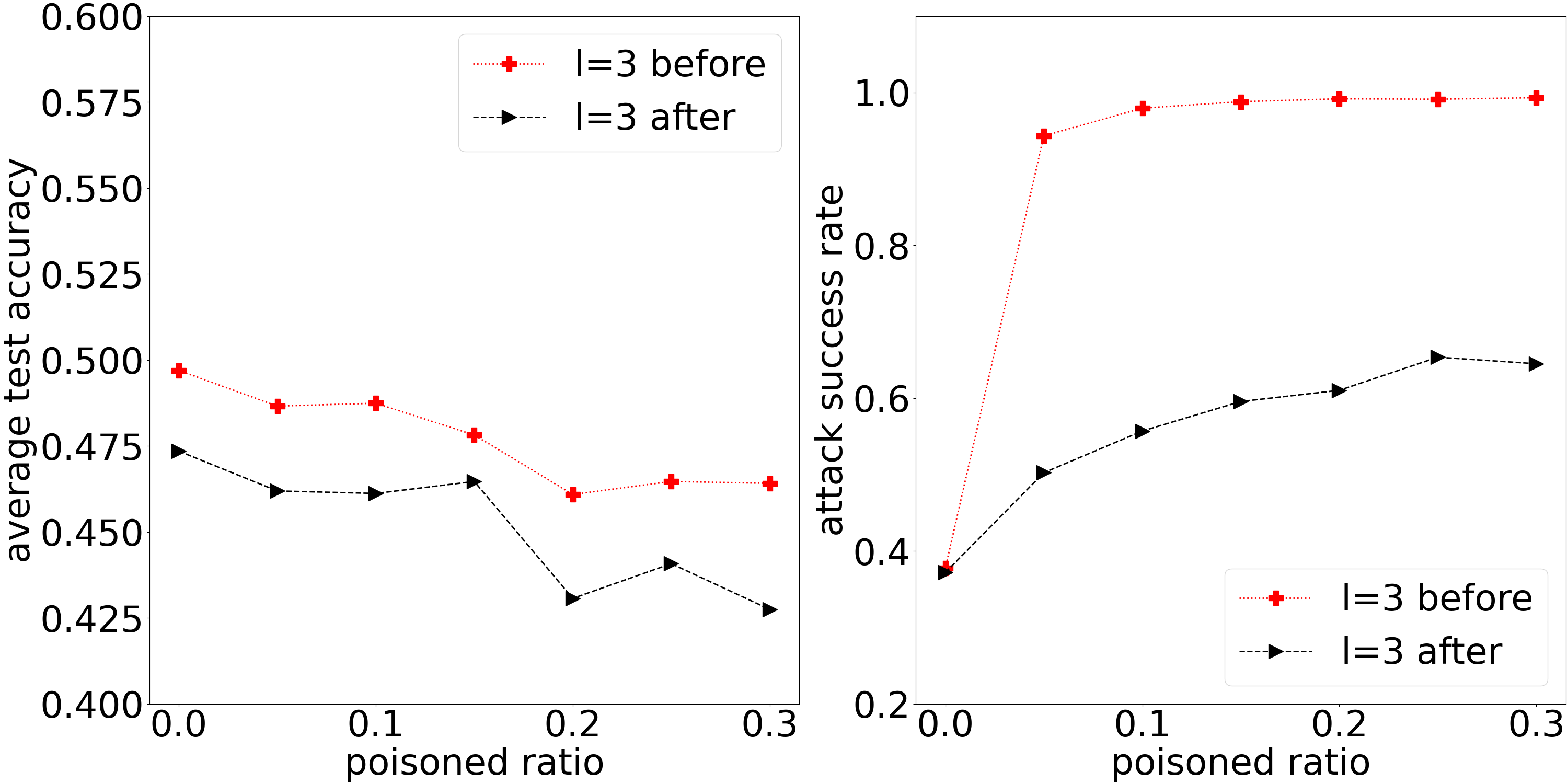} 
  \end{subfigure}
  \caption{{\bf Even with multiple hidden layers, CNN purification can mitigate the poisoning effect with some compromise on average test accuracy($\textit{training batch size} = 99$ , $\textit{repair data size = 69}$) on CIFAR}. Here CNN is purified by non-training instances.} 
  %\caption{{\bf Even with multiple hidden layers, CNN purification can mitigate the poisoning effect with some compromise on average test accuracy($\textit{training batch size} = 99$ , $\textit{repair data size = 69}$) on CIFAR}. $l$ denotes the number of hidden layers in CNN. 'Before' and 'After' represents 'before CNN purification' and 'after CNN purification'. Here CNN is purified by \textbf{non-training instances}.} 
  \label{fig: CIFAR_multilayer_random}  
\end{figure}

%\begin{figure}[t]
%\centering
%     \includegraphics[trim=0 0 0 0,clip,width=0.5\textwidth, height=0.6\textheight]{Figs/CIFAR_multilayer_random.png}
%  \caption{{\bf Even with multiple hidden layers, CNN purification can mitigate the poisoning effect with some compromise on average test accuracy($\textit{training batch size} = 99$ , $\textit{repair data size = 69}$) on CIFAR}. $l$ denotes the number of hidden layers in CNN. 'Before' and 'After' represents 'before CNN purification' and 'after CNN purification'. Here CNN is purified by \textbf{non-training instances}.} 
%  \label{fig: CIFAR_multilayer_random}  
%\end{figure}

\section{Experiment}\label{sec: experiments}
In this section, we conduct experiments on synthetic data and real data (MNIST \cite{Lecun1998gradient}, CIFAR-10 \cite{cifar10dataset}) to demonstrate the effectiveness of our proposed CNN purification method and evaluate the alignments of the results with our theoretical analysis. Furthermore, we apply the proposed method to mitigate poisoning attacks from the poisoned CNNs. The error is measured by the average $\ell_2$ norm. All the experimental results of synthetic data are averaged over 100 trials. All the experimental results of MNIST and CIFAR-10 data are averaged over 10 trials.

\subsection{Experiments on synthetic data} 
The synthetic data are generated by $x_s \sim \mathcal{N}(0,I_d)$. The noises $[\mathbf{z}_{W_j}]_i$, ($[\mathbf{z}_{\beta}]_i$) are generated from $\mathcal{N}(1,1)$. We first evaluate $p$ and $k$ by fixing the number of data points $n = 5$ and the number of partitions $m =5$. Experiments in Figure~\ref{fig: syn_p_k_main1}  and Figure~\ref{fig: syn_p_k_main2} are conducted under two different CNN training regimes (see settings in Theorem 2 and Theorem 3), respective. Experiments are conducted under different $p$ with $k=50, 100, 150$. When $\epsilon$ is small, e.g., $\epsilon < 0.2$, the recovery of both $W$ and $\beta$ are more likely to be successful. In each figure, one can see that increasing $p$ further increases the limit of $\epsilon$ for successful recovery of $\beta$. The phenomenon is consistent with Theorem~\ref{thm: main1} (Theorem~\ref{thm: main2}) as we require $\frac{(mn)^3log(p)}{p}$ to be small. Across all columns of the two figures, an obvious observation is that the limit of $\epsilon$ for successful recovery of $W$ increases when $k$ increases. In our theorems, successful recovery needs $\frac{log (p)}{k}$ and $\frac{mnlog (mn)}{k}$ to be sufficiently small. 

Then we evaluate $m$ by fixing the number of data points $n=5$, the number of first-layer neurons $p=500$, and each partition dimension $k=200$. Figure~\ref{fig: syn_m} shows results of recovery errors under different $m$. One can see that the limit of $\epsilon$ for successful recovery of $W$ and $\beta$ increases when $m$ decreases. The phenomenon is consistent with Theorem~\ref{thm: main1} as we require $\frac{mnlog (mn)}{k}, \frac{(mn)^3log(p)}{p}$ and $\epsilon\sqrt{mn}$ to be sufficiently small.

%\Ren{Important: For the previous figures on varying $p$ and $k$, you can still draw them vertically, as long as they can be presented in the current layout. For other figures, you can draw them depending on how to present them. In general, we don't want too many figures. So you can put some of them in the appendix.}

\subsection{Experiments on MNIST}
The MNIST data are randomly selected from MNIST training dataset, which is a widely used benchmark dataset in machine learning. The selected data come from three classes. The noise $[\mathbf{z}_{W_j}]_i$, ($[\mathbf{z}_{\beta}]_i$) are generated from $\mathcal{N}(1,1)$. First, we evaluate $p$ by fixing the number of data points $n=99$. Figure~\ref{fig: MNIST_p} shows results of recovery errors under different $p,m,k$. In the figure, one can see that increasing $p$ also increases the limit of $\epsilon$ for successful recovery of $\beta$. The phenomenon is similar to that shown in the synthetic data experiment and the same reason applies here.

We next evaluate the number of data points $n$ used in recovery by fixing the training data size to $99$. Figure~\ref{fig: MNIST_part_random} (a) shows the results of recovery errors by $n$ data points selected from the training batch. Figure~\ref{fig: MNIST_part_random} (b) shows the results of recovery errors by $n$ data points selected outside of the training batch. One can see that our CNN purification method can achieve good performance even when recovering with a small number of clean data points and potentially not from the training data. In practice, we can find a small amount of data from other resources, and the purification will not be affected. The recovery performance improves when $n$ decreases. The phenomenon is consistent with Theorem~\ref{thm: main1} as we require $\frac{(mn)^3log(p)}{p}$, $\frac{mnlog (mn)}{k}$ and $\epsilon\sqrt{mn}$ to be small. Decreasing $n$ can help reduce the magnitudes of these values.

%\Ren{move figures here}

\subsection{Experiments on CIFAR-10}
The CIFAR-10 dataset is widely recognized as a critical resource in the fields of machine learning and computer vision research. For the purposes of our experiments, we specifically selected data from three of the dataset's classes to examine the effectiveness of our recovery methods under varying conditions. In these experiments, the noises $[\mathbf{z}_{W_j}]_i$, ($[\mathbf{z}_{\beta}]_i$) are synthetically generated following a Gaussian distribution $\mathcal{N}(1,1)$. We evaluate the number of data points $n$ used in recovery by fixing the training data size to 99. 

Our analysis of the recovery process was documented in two parts, as illustrated in Figure~\ref{fig: CIFAR_part_random}. Figure~\ref{fig: CIFAR_part_random} (a) presents the recovery errors when we used $n$ data points directly selected from the training batch. Figure~\ref{fig: CIFAR_part_random} (b) contrasts these results with recovery errors observed when $n$ data points were chosen from outside the training batch, providing a broader view of the potential sources for recovery data. The results from these experiments indicate that our CNN purification method is capable of achieving commendable performance even when the recovery is conducted with a relatively small number of clean data points, which need not necessarily originate from the training dataset. The recovery performance improves when $n$ decreases. The phenomenon is consistent with Theorem~\ref{thm: main1}. This phenomenon highlights the robustness and efficiency of our purification method, demonstrating its effectiveness even under constrained data scenarios and thereby underscoring its potential applicability in various practical settings where data availability may be limited. %The figure demonstrates a trend similar to that observed in the previous MNIST experiments, indicating that parameter recovery performance improves as $n$ decreases. %This phenomenon is consistent with  Theorem~\ref{thm: main1} as we require $\frac{(mn)^3log(p)}{p}$, $\frac{mnlog (mn)}{k}$, and $\epsilon\sqrt{mn}$ to be small. Decreasing $n$ can help reduce the magnitudes of these values.

\subsection{Poisoning attack mitigation}
Here we apply our method on poisoning attack mitigation. We consider the backdoor attack, which is the most harmful attack category in poisoning attacks \cite{gu2019badnets,wang2020practical}. A backdoor attack on CNNs aims at compromising a machine learning model's reliability by implanting a malicious trigger during its training phase. This trigger, typically imperceptible to humans, can cause the model to output incorrect or manipulated responses when activated by specific inputs. Attackers carry out a backdoor attack by tainting the training data with a distinct pattern or feature along with a designated target label. As the model learns, it begins to link this pattern to the target label, thereby incorporating the backdoor into its functioning. Once deployed, the attacker can trigger this backdoor by embedding the specific pattern in new inputs, manipulating the model's output as intended. Furthermore, attackers can also alter CNN parameters directly to introduce backdoors, turning otherwise normal models into compromised versions with additional, harmful parameters. %A backdoor attack targeting neural networks is an adversarial strategy designed to undermine the reliability of a machine learning model by secretly embedding a malicious pattern or trigger during its training phase. This subtle trigger is often undetectable to humans but can force the model to generate erroneous or manipulated outputs when encountered in subsequent inputs. Attackers typically execute a backdoor attack by contaminating the training data with a particular pattern or feature and a corresponding target label. As the model trains, it learns to associate the pattern with the target label, effectively embedding the backdoor. Once the model is in use, the attacker can exploit this backdoor by incorporating the trigger into the input data, leading the model to produce the intended, manipulated output. Besides poisoning from training data, attackers can even directly manipulate CNN parameters to inject backdoors \cite{hong2022handcrafted}. In all the above attack settings, contaminated CNNs can be viewed as benign models with additional poisoning parameters. 
Existing poisoning mitigation defenses are mainly based on detection \cite{wang2020practical,pal2024backdoor} and fine-tuning \cite{pal2023towards,zhu2021clear}. To the best of our knowledge, no work has considered directly removing the poisoning effect from the model parameter. Our method only requires a small amount of benign data without label information.

We use the same data selected from the MNIST dataset. The training loss is set to softmax cross entropy loss. The source of noise is poisoned input generated from the poisoning attack utilized. %The poisoning attack aims to force CNNs to predict a target class when the input is injected by a fixed pattern. 
When the first five pixels of input images are set to black, the outputs of the CNN model will always be class zero (digit 0). Different from the definition of $\epsilon$, we define it as the ratio of poisoned training data. We evaluate the performance of the proposed method with respect to $\epsilon$ by fixing the training batch size to $99$. Figure~\ref{fig: MNIST_backdoor} shows results of test accuracy and attack success rate with respect to poisoned ratio under different $n$. The attack success rate is the percentage of test data that has been successively attacked. One can see that CNN purification can maintain high average test accuracy and mitigate the poisoning effect even with a small number of clean data points. The target classes consist of three categories: digits zero, one, and two. The CNN model trained achieved an average test accuracy of 95 percentage, while the attack success rate was approximately equivalent to random guessing, at 33 percentage. In Figure \ref{fig: MNIST_backdoor} panel (a) left column, one can see that the model's average test accuracy remains relatively stable when purified by all 99 training instances. Even when purified by only 10 percentage of the training data, the model's average test accuracy only drops by approximately 5 percentage. In Figure \ref{fig: MNIST_backdoor} panel (a) right column, the attack success rate after the recovery remains in the random guessing level within the range of 0 to 30 percentage poisoned ratio, even with only 9 data points. In contrast, the attack success rate before the recovery continuously increases within the range of 0 to 30 percentage poisoned ratio and finally reaches 100 percentage when the poisoned ratio gets to 30 percentage. Figure \ref{fig: MNIST_backdoor} panel (b) considers the scenario where the repair data is sourced from non-training resources. Here we pick the repair data from the same three classes but are not used in training. We observe a similar results as panel (a), indicating that the model can be repaired using data out of training data, as long as they come from the same distribution. In practice, users may not have training data. They can still purify CNN models using data they collected from other resources.

We also apply our method on the CIFAR dataset to mitigate backdoor attacks, presenting a tougher challenge compared to MNIST due to CIFAR's complexity. Unlike MNIST's simple images, CIFAR's 3-channel RGB images make hidden backdoor patterns within the first 5 pixels less noticeable, blending easily with the image's natural colors and textures. We maintain the same experimental setup as with MNIST except that we convert CIFAR images from RGB to grayscale and focus on three target classes: airplane, automobile, and bird. Figure \ref{fig: CIFAR_backdoor} shows results of test accuracy and attack success rate with respect to poisoned ratio under different $n$. One can see that CNN purification significantly reduces the impact of poisoning with only a minor trade-off in accuracy. In Figure \ref{fig: CIFAR_backdoor} panel (a) left column, one can see that the model's average test accuracy remains relatively stable when purified by all 99 training instances, Even when purified by only 10 percentage of the training data, the model's average test accuracy only drops by approximately 2 percentage. Compared with MNIST experiments, CIFAR dataset's complexity leads to a lower average test accuracy. The effect of CNN purification on average test accraucy mirrors the trend seen in MNIST experiments. In Figure \ref{fig: CIFAR_backdoor} panel (b) right column shows that the attack success rate increases steadily up to $100\%$ as the poisoned ratio reaches $30\%$. After recovery using our CNN purification method, the attack success rate drops sharply, even though the effect is not as strong as for MNIST due to CIFAR's complexity. Nonetheless, using just one third of the trained data for purification still reduces the attack success rate by $20\%$ post-recovery. Figure \ref{fig: CIFAR_backdoor} panel (b) considers the scenario where the repair data is sourced from non-training resources. We observe a similar results as panel (a), which indicates that model can be reparied by using data out of training data, as long as they come from the same distribution.

\subsection{Multi-layer Case}
Our study initially focuses on the theoretical aspects of CNN purification for models with a single hidden layer, but we have expanded our experimental framework to include multi-layer architectures while maintaining consistent backdoor attack scenarios as used in the single-layer experiments. Our exploration begins with an examination of the effects of our purification method on two-hidden-layer CNNs, using the relatively simpler MNIST dataset as a starting point.

%Although the theoretical analysis is on one-hidden-layer CNN, we extend our CNN purification experiments to multi-layer architectures, keeping the same backdoor attack settings as in the single-layer case. We first investigate how our method performs with two-hidden-layer CNN on the MNIST dataset. %CNN more hidden layers on the MNIST dataset. 

In these experiments, we set the training batch size at $99$ and the size of the repair data at $36$, aiming to understand how the number of hidden layers $l$ influences both the test accuracy and the success rate of backdoor attacks. `Before' and `After' represents `before CNN purification' and `after CNN purification'. The results, depicted in the left column of Figure~\ref{fig: MNIST_multilayer_part}, reveal that the model's average test accuracy tends to decrease marginally with the addition of more hidden layers, which could be indicative of an overfitting problem. Interestingly, while the purification process results in slightly lower test accuracies as more layers are added, the decline in accuracy is relatively minor even at lower poisoning ratios. Conversely, the right column of the same figure demonstrates a notable reduction in the attack success rate post-recovery, with this effect being pronounced for poisoned ratios up to 30\%. Another set of results from Figure~\ref{fig: MNIST_multilayer_random} examines scenarios where the repair data is derived from sources outside the training set. These findings are consistent with those from Figure~\ref{fig: MNIST_multilayer_part}, suggesting that the model can be effectively repaired using data out of training data, provided it is drawn from a similar distribution.

%We begin our experiments with the simpler MNIST dataset. Fixing the training batch size to $99$ and repair data size to $36$, we evaluate the impact of the number of hidden layers $l$ on test accuracy and attack success rate. Figure~\ref{fig: MNIST_multilayer_part} left column shows that the model's average test accuracy slightly decreases with more hidden layers, likely due to overfitting issue. The CNN purification results in lower test accuracy with more hidden layers, but the drop is small at low poison ratios. Figure~\ref{fig: MNIST_multilayer_part} right column shows a significant drop in attack success rate post-recovery for poisoned ratios up to $30\%$. %Although adding more hidden layers to the CNN mitigates this effect, the reduction remains at around $40\%$ for $l=3$. 
%Figure \ref{fig: MNIST_multilayer_random} considers the scenario where the repair data is sourced from non-training resources. We observe a similar results as Figure~\ref{fig: MNIST_multilayer_part}, which indicates that model can be repaired by using data out of training data, as long as they come from the same distribution.

Building on these insights, we extend our investigation to the more complex CIFAR-10 dataset, implementing experiments on CNNs with up to three hidden layers. Here, we maintain the same sizes for training batches and repair data as in the MNIST tests, at $99$ and $69$ respectively. Our analysis, shown in Figure~\ref{fig: CIFAR_multilayer_part}, focuses on the impact of CNN purification using data from training batches. Additionally, Figure~\ref{fig: CIFAR_multilayer_random} presents outcomes using repair data sourced from non-training environments, and these results echo the trends observed in the MNIST experiments. These CIFAR experiments further validate the patterns noted earlier, reinforcing the effectiveness of our proposed purification method across different datasets and more complex CNN architectures.

Overall, these extensive experimental results underscore the capability of our proposed method to mitigate the effects of model contamination in CNNs, extending beyond single-layer configurations to include networks with multiple hidden layers. This broad applicability highlights the potential of our approach to enhance the robustness of CNNs against sophisticated backdoor attacks, ensuring their reliability across a range of challenging real-world applications.

%We also experiment on the more challenging CIFAR-10 dataset using CNNs up to three layers, fixing the training batch size to $99$ and repair data size to $69$. We evaluate the impact of the number of hidden layers $l$ on test accuracy and attack success rate. Figure~\ref{fig: CIFAR_multilayer_part}  shows results using CNN purification with training batches. Figure~\ref{fig: CIFAR_multilayer_random} using non-training data for repair, yielding similar trends to Figure~\ref{fig: CIFAR_multilayer_part}. CIFAR experiments follow similar patterns to MNIST. The above results indicate that the proposed method has the ability to eliminate model noises on CNNs with more than one hidden layer.

%In summary, the CIFAR experiments follow similar patterns to MNIST. The only difference is that the CIFAR dataset's complexity causes reductions in test accuracy and purification effectiveness.

\section{Conclusion}\label{sec: conclusion}
CNNs are susceptible to various types of noise and attacks in applications. To address these challenges, this study introduces a robust recovery technique designed to eliminate noise from potentially compromised CNNs, offering an exact recovery guarantee for single-hidden-layer non-overlapping CNNs using the rectified linear unit (ReLU) activation function. Our theoretical findings indicate that CNNs can be precisely recovered under the overparameterization setting, given certain mild assumptions. We have successfully validated our method on both synthetic data, the MNIST dataset, and CIFAR-10 dataset. Additionally, we adapt the method to address backdoor attack elimination, demonstrating its potential as a defense mechanism against malicious model poisoning. We remark that this work mainly focuses on theoretical CNN purification. Our future directions are (1) Extending CNN purification to larger models and larger datasets (2) Improving the proposed method to eliminate various poisoning attacks. 

%Experimental outcomes validate the accuracy of our proofs and the efficacy of our approach in both synthetic and practical neural network contexts. Additionally, we adapt the method to address backdoor attack elimination, demonstrating its potential as a defense mechanism against malicious model poisoning.

%CNNs are susceptible to various types of noise and attacks in applications. To address these challenges, this paper proposes a robust recovery method that removes noise from potentially contaminated CNNs, offering an exact recovery guarantee for one-hidden-layer non-overlapping CNNs with the rectified linear unit (ReLU) activation function. The proposed method can precisely recover both the weights and biases of the CNNs, given some mild assumptions and an overparameterization setting. We have successfully validated our method on MNIST dataset. Our future directions are (1) Extending CNN purification to larger models and large datasets (2) Improving the proposed method to eliminate various poisoning attacks.  

%\newpage

\bibliographystyle{ieeetr}
\bibliography{ref,ref_backdoor}

\newpage
\twocolumn
\pagestyle{empty}
\appendix[Proof]

%{\appendix[Proof]
%Use $\backslash${\tt{appendix}} if you have a single appendix:
%Do not use $\backslash${\tt{section}} anymore after $\backslash${\tt{appendix}}, only $\backslash${\tt{section*}}.
%If you have multiple appendixes use $\backslash${\tt{appendices}} then use $\backslash${\tt{section}} to start each appendix.
%You must declare a $\backslash${\tt{section}} before using any $\backslash${\tt{subsection}} or using $\backslash${\tt{label}} ($\backslash${\tt{appendices}} by itself starts a section numbered zero.)}

\subsection{Additional lemmas of Lemma 1 }

We first prove that lemma 1 
holds for hidden layer design matrix $A_W$. Consider a linear model with contaminated parameter $\Theta_j=A_W u^*+z = {W_j} + z\in \mathbb{R}^k$, where $u^* \in \mathbb{R}^{mn}$. To robustly recover $W$ from polluted parameter $\Theta$, we propose the estimator
$$
\widetilde{u}=\underset{u \in \mathbb{R}^{mn}}{\operatorname{argmin}}\|\Theta_j -A_W u\|_1 .
$$
and define the purified model parameter as $ \widetilde{W_j} = A_W\widetilde{u}$.

 There exists some $\sigma^2$, $\underline{\lambda}$ and $\bar{\lambda}$ , such that for any fixed (not random) $c_1,...,c_k$ satisfying $\max_s|c_i|\leq 1$
 
\begin{align}
\Vert \frac{1}{k}\sum_{i=1}^k c_ia_i \Vert^2 \leq \frac{\sigma^2 mn}{k}\label{condition 1}\\
\underset{||\Delta||=1}{inf}\frac{1}{k}\sum_{i=1}^k|a_i^T\Delta|\geq \underline{\lambda}\label{condition 2}\\
\underset{||\Delta||=1}{sup}\frac{1}{k}\sum_{i=1}^k|a_i^T\Delta|^2\leq \bar{\lambda}^2 \label{condition 3}
\end{align}
with high probability.

With the above problem formulation, we can prove the following important lemma.

\begin{lemma}\label{4.1}
    Assume the design matrix $A_W$ satisfies equations \ref{condition 1}, \ref{condition 2} and \ref{condition 3}. Then if
$$\frac{\bar{\lambda}\sqrt{\frac{mn}{k}\log(\frac{ek}{mn})}+\epsilon \sigma \sqrt{\frac{mn}{k}}}{\underline{\lambda}(1-\epsilon)}$$

is sufficiently small, we have $\hat{u} = u^*$ with high probability.
\end{lemma}

This lemma reveals more precise conditions for successful recovery. At the same time, the verification of equations \ref{condition 1}, \ref{condition 2} and \ref{condition 3} is not easy. So for design matrix $A_W$, we give the following lemma to simplify this verification process.

\begin{lemma}\label{4.2}
Assume $mn/k$ is sufficiently small. Then, equations \ref{condition 1}, \ref{condition 2} and \ref{condition 3} hold for $A_W $ with some constants $\sigma ^2, \bar{\lambda}$ and $ \underline{\lambda}$.
\end{lemma} 

%The above lemma \ref{4.2} completes proof of \ref{lemma: conditions} for design matrix $A_W$. Next we derive the following lemma that will be used in theorems \ref{thm: main1} and \ref{thm: main2}

Since the gradient $\frac{\partial \mathcal{L}(\beta, W)}{\partial W_j}$ lies in the row space of $X$, the vector $\widehat{W}_j-W_j(0)$ also lies in the row space of $X$. Thus, the theoretical guarantee of the hidden layer purification directly follows Lemma \ref{4.3}.
\begin{lemma}\label{4.3}
 Assume $\frac{\sqrt{\frac{mn}{k}} \log \left(\frac{e k}{mn}\right)}{1-\varepsilon}$ is sufficiently small. We then have $W_j=\widetilde{W_j}$ with high probability.\\
\end{lemma}

Then we prove that lemma 1 holds for output layer design matrix $A_{\beta}$. Consider a random feature model with contaminated parameter $\eta = A_{\beta} v^*+z = {\beta} + z\in \mathbb{R}^p$, where $v^* \in \mathbb{R}^{n}$. To robustly recover $\beta$ from polluted parameter $\eta$, we propose the estimator

$$
\widetilde{v}=\underset{v \in \mathbb{R}^{n}}{\operatorname{argmin}}\|\eta -A_{\beta} v\|_1 .
$$
and define the purified model parameter as $ \widetilde{\beta} = A_{\beta}\widetilde{v}$.

There exists some $\sigma^2$, $\underline{\lambda}$ and $\bar{\lambda}$ , such that for any fixed (not random) $c_1,...,c_k$ satisfying $\max_s|c_i|\leq 1$
 
\begin{align}
\Vert \frac{1}{p}\sum_{i=1}^p c_ia_i \Vert^2 \leq \frac{\sigma^2 n}{p}\label{condition 1_}\\
\underset{||\Delta||=1}{inf}\frac{1}{p}\sum_{i=1}^p|a_i^T\Delta|\geq \underline{\lambda}\label{condition 2_}\\
\underset{||\Delta||=1}{sup}\frac{1}{p}\sum_{i=1}^p|a_i^T\Delta|^2\leq \bar{\lambda}^2 \label{condition 3_}
\end{align}
with high probability.

 The following lemma \ref{lemma a.6} is a stronger result of \ref{condition 3_}, which is used to prove lemma \ref{4.4}.

\begin{lemma}\label{lemma a.6}
We define the matrices $G,\overline{G} \in \mathbb{R}^{n\times n}$ by

$$G_{sl}=\frac{1}{p}\sum_{j=1}^p(\sum_{a=1}^m\psi(W_{j}^{T}P_{a}x_{s}))(\sum_{b=1}^m\psi(W_{j}^{T}P_{b}x_{l}))$$

and

$$\overline{G}_{sl}= \begin{cases}

\frac{1}{2},\quad & s=l       
\\
\frac{1}{2\pi}+\frac{1}{4}\frac{\overline{Px}_{s}^{T}\overline{Px}_{l}}{k}+\frac{1}{2\pi}\left(\frac{\Vert Px_{s} \Vert}{\sqrt{k}}-1+\frac{\Vert Px_{l} \Vert}{\sqrt{k}}-1\right),\quad & s\ne l

\end{cases} 
$$

where $Px_{s} = \sum_{a=1}^m P_ax_s$ and $Px_{l} = \sum_{b=1}^m P_bx_l$
    
Assume  $\frac{k}{log mn}$  is sufficiently large, and then

$$\Vert G - \bar{G} \Vert^{2}_{op} \lesssim \frac{m^2n^2}{p}+\frac{lognm}{k}+\frac{n^2}{k^2}$$
with high probability. And we also have $\Vert G  \Vert_{op} \lesssim m^2n^2$ with high probability.
\end{lemma}

With help of lemma \ref{lemma a.6}, we could complete proof of lemma \ref{lemma: conditions} for design matrix $A_{\beta}$.

\begin{lemma}\label{4.4} 
If  $\frac{mn}{\sqrt{p}}$ and $\frac{nlog(mn)}{k}$ are sufficiently small, then equations \ref{condition 1_}, \ref{condition 2_} and \ref{condition 3_} hold for $A_{\beta} $ with some constants $\sigma ^2, \bar{\lambda}$ and $ \underline{\lambda}$ .
\end{lemma} 

Next we could derive the following lemma based on lemma \ref{4.4}.

\begin{lemma}\label{4.5} If $\epsilon\sqrt{mn}$,$\frac{mn}{\sqrt{p}}$ and $\frac{nlog(mn)}{k}$, then we have $\widehat{\beta}=\widetilde{\beta}$ with high probability.
\end{lemma}

 Note that the purification of the output layer is more complicated because the gradient $\left.\frac{\partial \mathcal{L}(\beta, W)}{\partial \beta_j}\right|_{W=W(t-1)}$ lies in the row space of $\psi(X W(t-1))$, which changes over time. Thus, we cannot directly apply
the result of lemma \ref{4.5} for the theoretical guarantee of output layer purification. However, we will show that, in Theorems \ref{thm: main1} and \ref{thm: main2}, the purification of the output layer can carry over this result in certain cases.

\subsection{Additional proof of lemma 2. - lemma 7.} 

\emph {Proof of Lemma \ref{4.1}.}
    Define $L_k (u) = \frac{1}{k} \sum_{i=1}^{k}(|a_i^{T}(u^*-u)+z_i|-|z_i|)$, and $L(u) = \mathbb{E}(L_k(u)|A_W)$. We need ${inf}_{\Vert u-u^* \Vert \ge t} L_k(u) \le L_k(u^*) = 0$ , where $t\geq 0$ be arbitrary. By the convexity of $L_k(u)$, so it is obvious to conclude ${inf} _{\Vert u-u^* \Vert =t} L_k(u) \le 0$, and thus
$${inf}_{\Vert u-u^* \Vert =t} (L(u)-L_k(u)) \le {sup}_{\Vert u-u^* \Vert =t} (L(u)-L_k(u))$$
$$ {inf}_{\Vert u-u^* \Vert =t} L(u) \le {sup}_{\Vert u-u^* \Vert =t} (L(u)-L_k(u)) $$

Define $f_i(x) = \mathbb{E}_{z_i \sim Q_i}(|x+z_i| - |z_i|)$ and $Q_i(x) = Q_i(z_i\le x)$. It is easy to see that $f_i(0) = 0$ and $f_i^{'}(x)=1-2Q_i(-x)$. Observe that we can write
\begin{align}
   L(u) = (1-\epsilon)\frac{1}{k}\sum_{i=1}^{k}|a_i^T(u-u^*)|+\epsilon\frac{1}{k}\sum_{i=1}^k f_i(a_i^T(u^*-u)). \label{eq:A.4}
\end{align}

For any $u$ such that $\Vert u-u^*\Vert = t$, the first term of \ref{eq:A.4} can be lower bounded by 
$$(1-\epsilon)\frac{1}{k}\sum_{i=1}^{k}|a_i^T(u-u^*)| \ge \underline{\lambda}(1-\epsilon)t,$$
by equation \ref{condition 2}. To analyze the second term of \ref{eq:A.4} , we note that $f_i$ is a convex function, and therefore for any $u$ such that $\Vert u-u^*\Vert =t$,

\begin{align*}
\epsilon\frac{1}{k}\sum_{i=1}^k f_i(a_i^T(u^*-u)) &\ge \epsilon\frac{1}{k}\sum_{i=1}^k f_i(0)\notag\\
&+ \epsilon\frac{1}{k}\sum_{i=1}^k f_i^{'}(0)a_i^T(u^*-u)\notag\\
&= \epsilon\frac{1}{k}\sum_{i=1}^{k}(1-2Q_i(0))a_i^T(u^*-u)\notag\\
&\ge -\epsilon t \Vert \frac{1}{k}\sum_{i=1}^{k}(1-2Q_i(0))a_i\Vert \notag  \\
&\ge -\epsilon t\sigma \sqrt{\frac{mn}{k}}
\end{align*}
where the first inequality uses Cauchy-Schwarz, and the second inequality uses equation \ref{condition 1}. We have
$${sup}_{\Vert u-u^*\Vert =t}|L_m(u) - L(u)|\lesssim t\bar{\lambda}\sqrt{\frac{mn}{k}\log(\frac{ek}{mn})},$$
with high probability. Therefore, we have shown that $\Vert \hat{u} - u^*\Vert \ge t$ implies
$$\underline{\lambda}(1-\epsilon)t - \epsilon t\sigma \sqrt{\frac{mn}{k}} \lesssim t\bar{\lambda}\sqrt{\frac{mn}{k}\log(\frac{ek}{mn})}$$,
which is impossible when $\frac{\bar{\lambda}\sqrt{\frac{mn}{k}\log(\frac{ek}{mn})}+\epsilon \sigma \sqrt{\frac{mn}{k}}}{\underline{\lambda}(1-\epsilon)}$ is sufficiently small, and thus $\Vert \hat{u} - u^*\Vert<t $ with high probability. Then we must have $\hat{u}=u^*$ because $t$ is arbitrary.\\

\emph{Proof of Lemma \ref{4.2}.}
   Equation \ref{condition 1} is obvious. For equations \ref{condition 2} and \ref{condition 3}, we have
$${inf}_{\Vert \Delta \Vert =1} \frac{1}{k}\sum_{j=1}^{k} |a_j^T\Delta| \ge \sqrt{\frac{2}{\pi}} - {sup}_{\Vert \Delta \Vert=1}\left|\frac{1}{k}\sum_{i=1}^{k}|a_j^T \Delta| - \sqrt{\frac{2}{\pi}}\right|,$$
and we will analyze the second term on the right hand side of the inequality above via a discretization argument for the Euclidean sphere $S^{mn-1} = \{ \Delta \in \mathbb{R}^{mn} : \Vert \Delta \Vert = 1\}$. There exists a subset $\mathcal{N}_{\zeta} \subset S^{mn-1}$, such that for any $\Delta \in S^{mn-1}$ , there exists a $\Delta^{'} \in \mathcal{N}_{\zeta}$ that satisfies $\Vert \Delta -\Delta^{'}\Vert \le \zeta$, and we also have the bound $\log |\mathcal{N}_{\zeta}| \le mn \log (1 + 2/\zeta)$ according to Lemma 5.2 of \cite{Vershynin2010IntroductionTT}. For any $ \Delta \in S^{mn-1} $ and the corresponding $\Delta^{'}\in \mathcal{N}_{\zeta}$ that satisfies $\Vert \Delta -\Delta^{'}\Vert \le \zeta $, we have

\begin{align*}
   &\left|  \frac{1}{k}\sum_{j=1}^{k} |a_j^T\Delta| -\sqrt{\frac{2}{\pi}} \right| - \left|  \frac{1}{k}\sum_{j=1}^{k} |a_j^T\Delta^{'}| -\sqrt{\frac{2}{\pi}} \right| \\
&\le  \zeta {sup}_{\Vert \Delta \Vert =1} \frac{1}{k}\sum_{j=1}^{k} |a_j^T\Delta|\\
&\le  \zeta {sup}_{\Vert \Delta \Vert =1} \left| \frac{1}{k}\sum_{j=1}^{k} |a_j^T\Delta| -\sqrt{\frac{2}{\pi}}\right| + \zeta \sqrt{\frac{2}{\pi}}
\end{align*}

With some rearrangement, we obtain
\begin{align*}
&{sup}_{\Vert \Delta \Vert =1}\left|  \frac{1}{k}\sum_{j=1}^{k} |a_j^T\Delta| -\sqrt{\frac{2}{\pi}} \right| \\
&\le (1-\zeta)^{-1} \max_{\Delta\in\mathcal{N}_{\zeta}}\left|  \frac{1}{k}\sum_{j=1}^{k} |a_j^T\Delta| -\sqrt{\frac{2}{\pi}} \right|\\
&+\frac{\zeta}{1-\zeta}\sqrt{\frac{2}{\pi}}
\end{align*}
Setting $\zeta=1/3$,we then have
$${inf}_{\Vert \Delta \Vert =1}\frac{1}{k}\sum_{j=1}^{k} |a_j^T\Delta| \ge (2\pi)^{-1} - \frac{3}{2}\max_{\Delta\in\mathcal{N}_{1/3}}\left|  \frac{1}{k}\sum_{j=1}^{k} |a_j^T\Delta| -\sqrt{\frac{2}{\pi}} \right|.$$

And we have the fact that let $f:\mathbb{R}^k \to \mathbb{R}$ be a Lipschitz function with constant $L>0$ . Then, for any $t>0, Z \sim N(0,I_k)$ , 
\begin{align}
    \mathbb{P}(|f(z)-\mathbb{E}f(z)|>t) \le 2 {\exp} (-\frac{t^2}{2L^2})\label{lemma C}
\end{align}

Equation \ref{lemma C} together with a union bound argument leads to
$$\mathbb{P}\left( \max_{\Delta\in\mathcal{N}_{1/3}} \left|  \frac{1}{k}\sum_{j=1}^{k} |a_j^T\Delta| -\sqrt{\frac{2}{\pi}} \right|>t\right) \le 2\exp\left(mn\log(7)-\frac{k t^2}{2}\right),$$
which implies $\max_{\Delta\in\mathcal{N}_{1/3}}\left|  \frac{1}{k}\sum_{j=1}^{k} |a_j^T\Delta| -\sqrt{\frac{2}{\pi}} \right| \lesssim \sqrt{\frac{mn}{k}}$ with high probability. 
Since $ mn/k$ is sufficiently small, we have $${inf}_{\Vert \Delta \Vert =1}\frac{1}{k}\sum_{j=1}^{k} |a_j^T\Delta| \gtrsim 1 $$
with high probability . The high probability bound ${sup}_{\Vert \Delta \Vert =1} \frac{1}{k}\sum_{j=1}^{k} |a_j^T\Delta|^2 = \Vert A\Vert^2_{op}/k \lesssim 1+ mn/k$ is by \cite{Davidson2001Chapter8}, and the proof is complete.\\

\emph{Proof of Lemma \ref{4.3}.} 
   Since $W_j$ belongs to the row space of $A_W$, there exists some $u^* \in \mathbb{R}^{mn}$ such that ${W_j}=A_{W}^T u^*$. By Lemma \ref{4.1} and Lemma \ref{4.2}, we know that $\widetilde{u}=u^*$ with high probability, and therefore $\widetilde{W_j}=A_{W}^T \widetilde{u}=A_{W}^T u^*={W_j}$.\\

\emph{Proof of Lemma \ref{4.4}.} Equation \ref{condition 1_} is satisfied since 
\begin{align*}
    &\sum_{l=1}^n \mathbb{E} (\frac{1}{p} \sum_{j=1}^p \sum_{i=1}^m \psi(W_j^{T}P_i x_l))^2 \\
    &\leq \sum_{l=1}^n \frac{1}{p} \sum_{j=1}^p \sum_{i=1}^m \mathbb{E}\psi(W_j^{T}P_i x_l)^2  = mn
\end{align*}
%$$\Lu{\sum_{l=1}^n \mathbb{E} (\frac{1}{p} \sum_{j=1}^p \sum_{i=1}^m \psi(W_j^{T}P_i x_l))^2 \leq \sum_{l=1}^n \frac{1}{p} \sum_{j=1}^p \sum_{i=1}^m \mathbb{E}\psi(W_j^{T}P_i x_l)^2  = mn}$$ 
and Markov's inequality. Then first prove equation \ref{condition 2_}. Define $f(W,X,\varDelta,P)=\frac{1}{p}\sum_{j=1}^p|\sum_{l=1}^n\sum_{i=1}^m \psi(W_j^T P_i x_l)\varDelta_l|$
and
$g(X,\varDelta,P) = \mathop{\mathbb{E}}(f(W,X,\varDelta,P)|X,P).$
%$$\Lu{g(X,\varDelta,P) = \mathop{\mathbb{E}}(f(W,X,\varDelta,P)|X,P)}$$

Then we derive
\begin{align}
    &\underset{||\varDelta||=1}{\inf}f(W,X,\varDelta,P)\notag\\
    & \geq \underset{||\varDelta||=1}{\inf}\mathop{\mathbb{E}}f(W,X,\varDelta,P)-\underset{||\varDelta||=1}{\sup}|f(W,X,\varDelta,P)\notag\\
    &-\mathop{\mathbb{E}}f(W,X,\varDelta,P)|\notag\\
    &\geq \underset{||\varDelta||=1}{\inf}\mathop{\mathbb{E}}f(W,X,\varDelta,P)\label{eq:A.1}\\
    &- \underset{||\varDelta||=1}{\sup}|f(W,X,\varDelta,P)-g(X,\varDelta, P)|\label{eq:A.2}\\
    &- \underset{||\varDelta||=1}{\sup}|g(X,\varDelta, P)-\mathop{\mathbb{E}}f(W,X,\varDelta,P)| \label{eq:A.3}
\end{align}

\ref{eq:A.1}, \ref{eq:A.2}  and \ref{eq:A.3} will be analyzed separately.\\

$\textbf{Analysis of \ref{eq:A.1} }$ Define $h(W) = \mathop{\mathbb{E}}(\sum_{i=1}^{m}\psi(W^T P_ix)|W)$ and $\sum_{i=1}^{m}\overline{\psi}(W^T P_ix) = \sum_{i=1}^{m}\psi(W^T P_ix)-h(W)$, then

$$\mathop{\mathbb{E}}f(W,X,\varDelta,P) = \mathop{\mathbb{E}}|\sum_{l=1}^n\sum_{i=1}^m \overline{\psi}(W^T P_i x_l)\varDelta_l+\sum_{l=1}^n h(W) \varDelta_l|$$
The lower bound of above equation is 
\begin{align*}
    &\mathop{\mathbb{E}}f(W,X,\varDelta,P) \\
    &\geq |\sum_{l=1}^n \varDelta_l||\mathop{\mathbb{E}}h(W)|-\mathop{\mathbb{E}}|\sum_{l=1}^n\sum_{i=1}^m \overline{\psi}(W^T P_i x_l)\varDelta_l|
\end{align*}
where the second term is upper bounded by
\begin{align*}
&\mathop{\mathbb{E}}|\sum_{l=1}^n\sum_{i=1}^m \overline{\psi}(W^T P_i x_l)\varDelta_l|\\
&\leq \sqrt{\mathop{\mathbb{E}}|\sum_{l=1}^n\sum_{i=1}^m \overline{\psi}(W^T P_i x_l)\varDelta_l|^2}\\
&= \sqrt{\mathop{\mathbb{E}} \textbf{Var}(\bigg|\sum_{l=1}^n\sum_{i=1}^m {\psi}(W^T P_i x_l)\varDelta_l\bigg| |W)}\\
&= \sqrt{\mathop{\mathbb{E}} (\sum_{l=1}^n\sum_{i=1}^m \varDelta_l^2\textbf{Var}{\psi}(W^T P_i x_l)|W)}\\
&=\sqrt{\mathop{\mathbb{E}}|\sum_{i=1}^m\psi(W^TP_ix)|^2}\leq \sqrt{\mathop{\mathbb{E}}\sum_{i=1}^m|W^TP_ix|^2} = 1
\end{align*}

Since
$$\mathop{\mathbb{E}}h(W) = \frac{1}{\sqrt{\pi}}\frac{\Gamma((k + 1)/2)}{\sqrt{k}\Gamma(k/2)} \geq \frac{1}{\sqrt{2\pi}}\sqrt{\frac{k-1}{k}}$$

Therefore, as long as $k\geq 3$ and $\sum_{l=1}^n\varDelta_l\geq 7$, then $\mathop{\mathbb{E}}f(W,X,\varDelta)\geq 1$. The following conclusion holds
$$\underset{||\varDelta||=1,\sum_{l=1}^n\varDelta_i\geq 7}{\inf}\mathop{\mathbb{E}}f(W,X,\varDelta)\gtrsim 1$$

In another situation $|\sum_{l=1}^n\varDelta_l |< 7$, a lower bound for $|\sum_{l=1}^n \sum_{i=1}^m\psi(W^TP_ix_l)\varDelta_l|$ is

\begin{align}\label{A.6}
    &\mathop{\mathbb{E}}|\sum_{l=1}^n \sum_{i=1}^m\psi(W^TP_ix_l)\varDelta_l|\notag\\
    &\geq \sum_{l=1}^n|\sum_{i=1}^m\overline{\psi}(W^TP_ix_l)\varDelta_l|-\frac{7}{\sqrt{2\pi}}||W||
\end{align}
\\

%$$\Lu{\mathop{\mathbb{E}}|\sum_{l=1}^n \sum_{i=1}^m\psi(W^TP_ix_l)\varDelta_l|}$$

And we have the fact that 
\begin{align}
\mathbb{P}\left(\chi_k^2 \geq k+2 \sqrt{t k}+2 t\right) & \leq e^{-t}, \label{A.10}\\
\mathbb{P}\left(\chi_k^2 \leq k-2 \sqrt{t k}\right) & \leq e^{-t} .\label{A.11}
\end{align}
for any $t>0$.\\

Therefore,

\begin{align*}
&\mathop{\mathbb{E}}f(W,X,\varDelta)\geq\mathop{\mathbb{E}} (\lvert\sum_{l=1}^{n}\sum_{i=1}^{m}\psi\left(W^{T}P_{i}x_l\right)\Delta_{l}  \rvert  \times\\
&\mathbb{I} \{\lvert \sum_{l=1}^{n}\sum_{i=1}^{m}
\{\bar{\psi}\left(W^{T}P_{i}x_l\right)\Delta_{l} \rvert \geq 6,\frac{1}{2}\leq \Vert W \Vert^{2} \leq 2\})\\
&\geq \mathbb{P} (\lvert \sum_{l=1}^{n}\sum_{i=1}^{m}
\{\bar{\psi}\left(W^{T}P_{i}x_l\right)\Delta_{l} \rvert \geq 6, \frac{1}{2}\leq \Vert W \Vert^{2} \leq 2)\\
& = \mathbb{P} \left(\lvert \sum_{l=1}^{n}\sum_{i=1}^{m}
\{\bar{\psi}\left(W^{T}P_{i}x_l\right)\Delta_{l} \geq 6 \rvert |\frac{1}{2}\leq \Vert W \Vert^{2} \leq 2\right)\times\\
&\mathbb{P}(\frac{1}{2}\leq \Vert W \Vert^{2} \leq 2)\\
& \geq\mathbb{P} \left(\lvert \sum_{l=1}^{n}\sum_{i=1}^{m}
\{\bar{\psi}\left(W^{T}P_{i}x_l\right)\Delta_{l} \geq 6 \rvert |\frac{1}{2}\leq \Vert W \Vert^{2} \leq 2\right)\times\\
&(1-2\mathop{exp}(-k/16))
\end{align*}
where the last inequality is by equations \ref{A.10} , \ref{A.11} . Then we have
\begin{align}\label{A.7}
&Var\left(\sum_{i=1}^{m}\bar{\psi}\left(W^{T}P_ix\right)|W\right)\notag\\
& = \left\|  W  \right\| ^{2}Var\left(\sum_{i=1}^{m}max\left(0,W^{T}P_ix/\left\|  W  \right\|\right)|W\right)\notag \\
&= \left\|  W  \right\| ^{2}\frac{1-\pi^{-1}}{2}
\end{align}

and
\begin{align*}
    &\mathop{\mathbb{E}}\left(|\sum_{i=1}^{m}\bar{\psi}\left(W^{T}P_ix\right)|^{3} |W\right)\\
    &\leq3\mathop{\mathbb{E}}\left(|\sum_{i=1}^{m}\psi\left(W^{T}P_i x\right)|^{3} |W\right)+ 3|h\left(W\right)|^{3} \\
    &\leq \frac{3}{2}\Vert W\Vert^{3}
\end{align*}

Therefore, by theorem 2.20 of \cite{Ross2007ASC}, we have,
\begin{align*}
&\mathbb{P}\left(|\sum_{l=1}^n\sum_{i=1}^m\bar{\psi}\left(W^{T}P_{i}x_l\right)\Delta_{l}|\geq 6\bigg|\frac{1}{2}\leq \Vert W \Vert^{2}\leq 2 \right)\\
&\geq \mathbb{P}\left(\frac{|\sum_{l=1}^n\sum_{i=1}^m\bar{\psi}\left(W^{T}P_{i}x_l\right)\Delta_{l}|}{\Vert W\Vert\sqrt{\frac{1-\pi^{-1}}{2}}}\geq 21 \bigg|\frac{1}{2}\leq \Vert W \Vert^{2}\leq 2\right)\\
&\geq \mathbb{P}\left(N\left(0,1\right) > 21 \right)-\\ &\mathop{\sup}_{\frac{1}{2} \leq \Vert W \Vert^{2} \leq 2} 2\sqrt{3\sum_{l=1}^n|\Delta_{l}|^{3}\frac{\sum_{i=1}^{m}E\left(|\bar{\psi}\left(W^{T}P_{i}x_l\right)|^{3} |W\right)}{\Vert W \Vert^{3}\left(\frac{1-\pi^{-1}}{2}\right)^{\frac{3}{2}}}}
\\&\geq \mathbb{P}\left(N\left(0,1\right) > 21 \right)- 10\sqrt{\sum_{l=1}^{n}|\Delta_{l}|^{3}}
\\&\geq \mathbb{P}\left(N\left(0,1\right) > 21 \right)- 10\mathop{\max}_{1\leq l\leq n}|\Delta_{l}|^\frac{3}{2}
\end{align*}
When $\mathop{\max}_{1\leq l\leq n}|\Delta_{l}|^\frac{3}{2} \leq \delta_{0}^{\frac{3}{2}}:=P\left(N\left(0,1\right)>21\right)/20$ and $|\sum_{l=1}^{n}\Delta_{l}|<7$, we can lower bound $\mathop{\mathbb{E}} \left(f\left(X,W,\Delta\right)\right) $by an absolute constant, and 

\begin{align}
\mathop{\inf}_{\Vert\Delta\Vert=1,|\sum_{l=1}^{n}\Delta_{l}|<7,\mathop{\max}_{1\leq l\leq n}|\Delta_{l}|\leq \delta_{0}}\mathop{\mathbb{E}} f\left(X,W,\Delta\right) \gtrsim 1
\end{align}

Finally, we consider the case when$\mathop{}\max_{1\leq l\leq n}|\Delta_{l}|>\delta_{0}$ and $|\sum_{l=1}^{n}\Delta_{l}|<7$. Without loss of generality, we can assume $\Delta_{1}>\delta_{0}$, Note that the lower bound \ref{A.6} still holds, and thus we have
\begin{align*}
    &|\sum_{l=1}^{n}\sum_{i=1}^{m}\psi\left(W^{T}P_{i}x_l\right)\Delta_{l}|\geq \sum_{i=1}^{m}\bar{\psi}\left(W^{T}P_{i}x_1\right)\Delta_{1}\\
    &-|\sum_{l=2}^{n}\sum_{i=1}^{m}\bar{\psi}\left(W^{T}P_{i}x_l\right)\Delta_{l}|-\frac{7}{\sqrt{2\pi}}\Vert W \Vert
\end{align*}

We then lower bound $\mathop{\mathbb{E}} f\left(X,W,\Delta\right)$by

\begin{align*}
&\mathop{\mathbb{E}}( \lvert \sum_{l=1}^{n}\sum_{i=1}^{m}\psi\left(W^{T}P_{i}x_l\right)\Delta_{l}\rvert \mathbb{I}\{\sum_{i=1}^{m}\bar{\psi}\left(W^{T}P_{i}x_1\right)\Delta_{1} \geq 8, \\
&\lvert \sum_{l=2}^{n}\sum_{i=1}^{m}\bar{\psi}\left(W^{T}P_{i}x_l\right)\Delta_{i}\rvert \leq 2,\frac{1}{2}\leq \Vert W \Vert^{2} \leq 2 \})\\
&\geq \mathbb{P}(\sum_{i=1}^{m}\bar{\psi}\left(W^{T}P_{i}x_1\right)\Delta_{1} \geq 8,\\
&\lvert\sum_{l=2}^{n}\sum_{i=1}^{m}\bar{\psi}\left(W^{T}P_{i}x_l\right)\Delta_{l}\rvert \leq 2 \bigg| \frac{1}{2}\leq \Vert W \Vert^{2} \leq 2 )\mathbb{P}\left(\frac{1}{2}\leq \Vert W \Vert^{2} \leq 2\right)
\\&\geq \mathbb{P}\left(\sum_{i=1}^{m}\bar{\psi}\left(W^{T}P_{i}x_1\right)\Delta_{1} \geq 8 \bigg|\frac{1}{2}\leq \Vert W \Vert^{2} \leq 2 \right) \times 
\\&\mathbb{P}\left(\lvert \sum_{l=2}^{n}\sum_{i=1}^{m}\bar{\psi}\left(W^{T}P_{i}x_l\right)\Delta_{l}\rvert \leq 2 \bigg|\frac{1}{2}\leq \Vert W \Vert^{2} \leq 2 \right)\left( 1-2 {\exp}\left(-k\slash 16\right) \right)
\end{align*}

For any $W$ that satisfies $\frac{1}{2} \leq \Vert W \Vert ^{2} \leq 2 $, we have

\begin{align*}
&\mathbb{P}\left(|\sum_{i=1}^{m}\bar{\psi}\left(W^{T}P_{i}x_1\right)\Delta_{1}| \geq 8 |W \right) \\
& \geq \mathbb{P}\left(|\sum_{i=1}^{m}\bar{\psi}\left(W^{T}P_{i}x_1\right)|\geq \frac{8}{\delta_{0}}  |W \right)\\
&\geq \mathbb{P} \left(|\sum_{i=1}^{m}\psi\left(W^{T}P_{i}x_1\right)|\geq \frac{8}{\delta_{0}} + \frac{1}{\sqrt{\pi}}  |W \right)\\
&\geq \mathbb{P} \left(\sum_{i=1}^{m} W^{T}P_{i}x_1|\geq \frac{8}{\delta_{0}} + \frac{1}{\sqrt{\pi}}  |W \right)\\
&\geq \mathbb{P} \left(N\left(0,1\right)|\geq \frac{8\sqrt{2}}{\delta_{0}} + \sqrt{\frac{2}{\pi}}\right)
\end{align*}

which is a constant. We also have
\begin{align*}
&\mathbb{P}\left(|\sum_{l=2}^{n}\sum_{i=1}^{m}\bar{\psi}\left(W^{T}P_{i}x_l\right)\Delta_{l}|\leq 2 \bigg|\frac{1}{2}\leq \Vert W \Vert^{2} \leq 2 \right)\\
&\geq 1- \frac{1}{4}Var\left(\sum_{l=2}^{n}\sum_{i=1}^{m}\bar{\psi}\left(W^{T}P_{i}x_l\right)\Delta_{l}|W\right)
\geq \frac{1}{2}
\end{align*}
where the last inequality is by equation \ref{A.7}. Therefore, we have

\begin{align*}
    &\mathop{\mathbb{E}} f\left(X,W,\Delta\right) \\
    &\geq \frac{1}{2}\left( 1-2 {\exp}\left(-k \slash 16\right)\right)P\left(N\left(0,1\right)\geq \frac{8\sqrt{2}}{\delta_{0}} + \sqrt{\frac{2}{\pi}}\right) \\
    &\gtrsim 1
\end{align*}

and we can conclude that
\begin{align}
\mathop{\inf}_{\Vert\Delta\Vert=1,|\sum_{l=1}^{n}\Delta_{l}|<7,\mathop{\max}_{1\leq l\leq n}|\Delta_{l}|\geq \delta_{0}}\mathop{\mathbb{E}} f\left(X,W,\Delta\right) \gtrsim 1  
\end{align}

In the end, we combine the three cases, and  obtain the conclusion that $$\mathop{\inf}_{\Vert\Delta\Vert=1}\mathbb{E}f\left(X,W,\Delta\right)\gtrsim 1$$

\textbf{Analysis of \ref{eq:A.2}} We first extend some notations. Conditional expectation with respect to $X$ is denoted as $\mathbb{E}^X$. $\widetilde{W}$ is an independent copy of $W$. $\varepsilon_1,...,\varepsilon_n$ are independent Rademacher random variables. $F(\varepsilon,W,X,\Delta,P) = \frac{1}{n}\sum_{l=1}^n \varepsilon_l\bigg|\sum_{i=1}^m \psi(W^T P_i x_l)\Delta_l\bigg|$  and $\overline{F} (\varepsilon, X, \Delta,P) = \mathbb{E}^{\varepsilon,X,P} F(\varepsilon, W, X, \Delta,P)$. To prove term \ref{eq:A.2} has upper bound, we apply a standard symmetrization argument to bound its moment generating function. For any $\lambda > 0 $

\begin{align}
&\mathbb{E}^{X,P}\mathop{\exp}( \lambda\mathop{\sup}_{\Vert\Delta\Vert=1}|f\left(W,X,\Delta,P\right)-g\left(X,\Delta,P\right)|)\notag\\
&\leq \mathbb{E}^{X,P}\mathop{\exp}( \lambda \mathop{\sup}_{\Vert\Delta\Vert=1}|f\left(W,X,\Delta,P\right)-f\left(\widetilde{W},X,\Delta,P\right)|)\notag\\ 
&\leq \mathbb{E}^{X,P}\mathop{\exp}(\lambda \mathop{\sup}_{\Vert\Delta\Vert=1} F(\varepsilon,W,X,\Delta,P) - F(\varepsilon,\widetilde{W},X,\Delta,P))\notag\\ 
&\leq \mathbb{E}^{X,P}\mathop{\exp}(2 \lambda \mathop{\sup}_{\Vert\Delta\Vert=1} F(\varepsilon,W,X,\Delta,P))\notag 
\end{align}

Following the same discretization argument in lemma \ref{4.2} except the Euclidean sphere $ S^{n-1} = \{\Delta \in \mathbb{R}^n : \Vert \Delta\Vert = 1\}$, we reach similar conclusion , by lemma 5.2 of \cite{Vershynin2010IntroductionTT}, that for any $\Delta \in S^{n-1}$, there exists a $\Delta^{'} \in \mathcal{N} $ that satisfies $\Vert \Delta - \Delta^{'} \Vert \leq \frac{1}{2}$, and we also have the bound $log |\mathcal{N} | \leq 2n$. We could continue to bound above inequality by the derived discretization argument

\begin{align}
&\mathbb{E}^{X,P}\mathop{\exp}(2 \lambda \mathop{\sup}_{\Vert\Delta\Vert=1} F(\varepsilon,W,X,\Delta,P))\notag\\     
&\leq \mathbb{E}^{X,P}\mathop{\exp}(4 \lambda \mathop{\max}_{\Delta\in \mathcal{N}} F(\varepsilon,W,X,\Delta,P))\notag\\ 
&\leq \mathbb{E}^{X,P}\mathop{\exp}( 4\lambda \mathop{\max}_{\Delta \in \mathcal{N}} |F(\varepsilon,W,X,\Delta,P) - \overline{F} (\varepsilon, X, \Delta,P)| +\notag \\
&4\lambda \mathop{\max}_{\Delta \in \mathcal{N}}|\overline{F} (\varepsilon, X, \Delta ,P)|) \notag\\
&\leq \frac{1}{2} \sum_{\Delta \in \mathcal{N}} \mathbb{E}^{X,P}\mathop{\exp}(8\lambda |F(\varepsilon,W,X,\Delta,P) - \overline{F} (\varepsilon, X, \Delta,P)|)  \notag\\
&+\frac{1}{2}\sum_{\Delta \in \mathcal{N}} \mathbb{E}^{X,P}\mathop{\exp}(8\lambda |\overline{F} (\varepsilon, X, \Delta,P)|) \label{A.12}
\end{align}

The left and right term in inequality \eqref{A.12} could be bound separately. For the left term of \eqref{A.12}

\begin{align}
&\frac{1}{2} \sum_{\Delta \in \mathcal{N}} \mathbb{E}^{X,P}\mathop{\exp}(8\lambda |F(\varepsilon,W,X,\Delta,P) - \overline{F} (\varepsilon, X, \Delta,P)|) \notag \\
& \leq \frac{1}{2} \sum_{\Delta \in \mathcal{N}} \mathbb{E}^{X,P}\mathop{\exp}(8\lambda \frac{1}{p}\sum_{j=1}^{p}\sum_{l=1}^{n}\sum_{i=1}^{m}|(\psi(W_j^T P_i x_l)- \psi(\widetilde{W_j}^T P_i x_l))\Delta_l|) \notag \\
& \leq \frac{1}{2} \sum_{\Delta \in \mathcal{N}} \mathbb{E}^{X,P}\mathop{\exp}(8\lambda \frac{1}{\sqrt{p}}\Vert W -  \widetilde{W}\Vert \sqrt{\sum_{l=1}^{n}\sum_{i=1}^{m}\Vert P_ix_l\Vert^{2} |\Delta_l|^2}) \notag \\
& \leq \frac{1}{2} \sum_{\Delta \in \mathcal{N}} \mathbb{E}^{X,P}\mathop{\exp}(8\lambda \sqrt{\frac{3nm}{p}} \Vert \sqrt{k}W - \sqrt{k} \widetilde{W}\Vert)\notag 
\end{align}
where the last inequality holds under the event ${\sum_{l=1}^{n}\sum_{i=1}^{m} \Vert P_ix_l \Vert ^2 \leq 3mnk}$, which indicates that $F(\varepsilon,W,X,\Delta,P)$ is Lipschitz function by definition. By lemma A.3 in \cite{gao2020model}, one could find the sub-Gaussian tail probability of Lipschitz function $\mathbb{P}(|F(\varepsilon,W,X,\Delta,P) - F (\varepsilon,\widetilde{W}, X, \Delta,P)|>t|X,P) \leq 2 \mathop{\exp} \left(\frac{-pt^2}{6nm}\right) $ for any $t > 0$. Given this sub-Gaussian tail probability, the bound of left term could be expressed, by lemma 5.5 of \cite{Vershynin2010IntroductionTT}, as

\begin{align}
&\frac{1}{2} \sum_{\Delta \in \mathcal{N}} \mathbb{E}^{X,P} \mathop{\exp} (8\lambda \left|F(\varepsilon,W,X,\Delta,P) - \overline{F} (\varepsilon, X, \Delta,P)\right|) \notag\\
&\leq \frac{1}{2} \sum_{\Delta \in \mathcal{N}} \mathop{\exp} (C_{1}\frac{nm}{p} \lambda^{2}) \label{a.14}
\end{align}
for some constant $C_1 > 0$. On the other hand, the bound of right term of \eqref{A.12} is 

\begin{align}
& \frac{1}{2}\sum_{\Delta \in \mathcal{N}} \mathbb{E}^{X,P}\mathop{\exp}(8\lambda |\overline{F} (\varepsilon, X, \Delta,P)|) \notag\\
& \leq \frac{1}{2}\sum_{\Delta \in \mathcal{N}}\mathbb{E}^{X,P} \mathop{\exp}(8\lambda\left|\frac{1}{p}\sum_{j=1}^{p} \varepsilon_j \right|\sqrt{\sum_{l=1}^{n}\sum_{i=1}^{m}\mathbb{E}^{X,P}\left|\psi(W^TP_iX_l)\right|^2})\notag\\
& \leq \frac{1}{2}\sum_{\Delta \in \mathcal{N}}\mathbb{E}^{X,P}\mathop{\exp}(8\lambda\left|\frac{1}{p}\sum_{j=1}^{p} \varepsilon_j \right|\sqrt{\sum_{l=1}^{n}\sum_{i=1}^{m}\frac{1}{k}\Vert P_iX_l\Vert^2}) \notag \\
& \leq \frac{1}{2}\sum_{\Delta \in \mathcal{N}}\mathbb{E}^{X,P}\mathop{\exp}(8\lambda\sqrt{3nm}\left|\frac{1}{p}\sum_{j=1}^{p} \varepsilon_j \right|) \notag\\
& \leq \frac{1}{2}\sum_{\Delta \in \mathcal{N}}\exp{(C_{1}\frac{nm}{p} \lambda^{2})}\label{a.15}
\end{align}
where the second last inequality holds under the event ${\sum_{l=1}^{n}\sum_{i=1}^{m} \Vert P_ix_l \Vert ^2 \leq 3mnk}$ and the last inequality holds with an application of Hoeffding-type inequality
(Lemma 5.9 of \cite{Vershynin2010IntroductionTT}). Note \eqref{a.14} and \eqref{a.15} could share the same $C_1$ when it is sufficiently large. Plugging two moment generating function bounds  \eqref{a.14} and \eqref{a.15} into \eqref{A.12} and then applying union bound, one could obtain

\begin{align}
&\mathbb{E}^{X,P}\exp(\lambda\underset{||\varDelta||=1}{\sup}\left|f(W,X,\varDelta,P)-g(X,\varDelta,P)\right|\notag\\
&\leq \sum_{\Delta \in \mathcal{N}}\exp{(C_{1}\frac{nm}{p} \lambda^{2})}\leq\exp(C_{1}\frac{nm}{p}\lambda^{2}+2n)\label{a.16}
\end{align}

Given the above expectation bound \eqref{a.16}, we apply Chernoff bound to obtain the sub-Gaussian tail probability of \eqref{eq:A.2} $\mathbb{P}(\underset{||\varDelta||=1}{\sup}\left|f(W,X,\varDelta,P)-g(X,\varDelta,P)\right| > t)\leq\exp{(C_{1}\frac{nm}{p} \lambda^{2}+2n-\lambda t)}$. Optimizing over $\lambda$ and setting $t\asymp \sqrt{\frac{m^2n^2}{p}}$ in this sub-Gaussian tail probability, we have

$$\underset{||\varDelta||=1}{\sup}\left|f(W,X,\varDelta,P)-g(X,\varDelta,P)\right|\lesssim \sqrt{\frac{m^2n^2}{p}}$$

with high probability.\\

\textbf{Analysis of \ref{eq:A.3}} We use the same discretization argument as in analysis of \ref{eq:A.2}. And we reach similar conclusion , by Lemma 5.2 of \cite{Vershynin2010IntroductionTT}, that for any $\Delta \in S^{n-1}$, there exists a $\Delta^{'} \in \mathcal{N} $ that satisfies $\Vert \Delta - \Delta^{'} \Vert \leq \zeta$, and we also have the bound $log |\mathcal{N} | \leq n \log(1+2/\zeta)$. In \ref{eq:A.3}, $\mathop{\mathbb{E}}f(W,X,\varDelta,P)$ is equal to another form $\mathop{\mathbb{E}}g(X,\Delta,P)$ by law of total expectation. Since $X,P$ are fixed in pre-condition of $g(X,\Delta,P)$, we bound \ref{eq:A.3} by finding its superior $\mathop{\sup}_{\Vert \Delta \Vert =1}  \left| g(X,\Delta,P)-\mathop{\mathbb{E}}g(X,\Delta,P) \right|$ upper bound

\begin{align}
&\mathop{\sup}_{\Vert \Delta \Vert =1}  \left| g(X,\Delta,P)-\mathop{\mathbb{E}}g(X,\Delta,P) \right|\notag\\
&\leq \mathop{\sup}_{\Vert \Delta \Vert =1}\left| g(X,\Delta^{'},P)-\mathbb{E}g(X,\Delta^{'},P) \right| \notag\\
&+ \zeta \mathop{\sup}_{\Vert \Delta \Vert =1} \left| g(X,\Delta ,P)-\mathbb{E} g (X,\Delta,P) \right| \notag\\
&+ 2\zeta \mathop{\sup}_{\Vert \Delta \Vert =1} \mathbb{E}g(X,\Delta ,P) \notag \\
& \leq (1-\zeta)^{-1} \max_{\Delta\in \mathcal{N}_{\zeta}} \left| g(X,\Delta ,P)-\mathbb{E}g(X,\Delta,P) \right| \notag\\
&+ 2\zeta(1-\zeta)^{-1}\mathop{\sup}_{\Vert \Delta \Vert =1}\mathbb{E} g(X,\Delta,P) \label{A.13}
\end{align}

The left and right terms of inequality \eqref{A.13} could be bounded separately. For left terms of \eqref{A.13}

\begin{align*}
&|g(X, \Delta, P)-g(\widetilde{X}, \Delta, \widetilde{P})| \\
& \leq \mathbb{E}^{X,P}\left|\sum_{s=1}^n \sum_{a=1}^m\sum_{b=1}^m\left(\psi\left(W_j^T P_a x_s\right)-\psi\left(W_j^T \widetilde{P_a}\widetilde{{x}_i}\right)\right) \Delta_s\right| \\
& \leq \sqrt{\sum_{s=1}^n \sum_{a=1}^m\mathbb{E}^{X,P}\left(W_j^T \left(P_a x_s-\widetilde{P}_a\widetilde{x}_i\right)\right)^2} \\
& =\frac{ 1}{\sqrt{k}} \sqrt{\sum_{s=1}^n\sum_{a=1}^m\left\|P_a x_s-\widetilde{P}_a\widetilde{x}_i\right\|^2 }
\end{align*}
for any $X, \widetilde{X}, P, \widetilde{P}$. Therefore, we conclude $g(\widetilde{X},\widetilde{P}, \Delta)$ is Lipschitz function. By \cite{Cirelson1976NormsOG} and union bound, we obtain the sub-Gaussian tail probability $$\mathbb{P}\left( \max_{\Delta\in \mathcal{N}_{\zeta}} \left| g(X,\Delta,P)-\mathbb{E}g(X,\Delta,P) \right|> t\right ) $$ $$\le 2\exp \left( -\frac{k t^2}{2} + n\log(1+\frac{2}{\zeta})\right )$$ for any $ t > 0$. This tail probability implies the left term of inequality $\eqref{A.13}$ is bounded by

\begin{align}
\max_{\Delta\in \mathcal{N}_{\zeta}} \left| g(X,\Delta,P)-\mathbb{E}g(X,\Delta,P) \right| \lesssim \sqrt{\frac{n\log(1+2/\zeta)}{k}} \label{a.21} 
\end{align}
with high probability. For the right terms of inequality $\eqref{A.13}$, we have

\begin{align}
\mathbb{E} g(X,\Delta,P)  \le \sqrt{\mathbb{E}\sum_{l=1}^{n}\sum_{i=1}^{m}|\psi(W^T P_ix_l)|^2} \le \sqrt{mn}\label{a.22}     
\end{align}

Plugging \eqref{a.21} and \eqref{a.22} into \eqref{A.13}, we obtain

$$\sup_{\Vert \Delta \Vert  = 1} |g(X,\Delta,P)-\mathop{\mathbb{E}}f(W,X,\varDelta,P)| \lesssim \sqrt{\frac{n\log\left(1+\frac{2}{\zeta}\right)}{k}}+ \sqrt{m n}\zeta $$
with high probability as long as $\zeta \leq \frac{1}{2}$. We choose $\zeta = \frac{c}{\sqrt{mn}}$ with a sufficiently small constant $c > 0$,and thus the bound is sufficiently small as long as  $\frac{n \log mn}{k} $ is suffiently small.

To prove the equation \ref{condition 3_} of lemma \ref{4.4}, we establish the following stronger result.\\

\emph{Proof of lemma \ref{lemma a.6}.} 
Define $\widetilde{G} \in \mathbb{R}^{n\times n} $  with entries $\widetilde{G}_{sl} = \mathbb{E}(\psi(W^{T}Px_{s})\psi(W^{T}Px_{l})|X,P) $ Note that
\begin{align*}
  &\mathbb{E}(G_{sl}-\widetilde{G}_{sl})^{2}=\mathbb{E}\textbf{Var}(G_{sl}|X,P)\\
  &\le \frac{1}{p}\mathbb{E}|(\psi(W^{T}Px_{s}))(\psi(W^{T}Px_{l}))|^{2}=\frac{3m^2}{2p}\mathbb{E}(\Vert  W \Vert^{4})\le \frac{5m^2}{p}  
\end{align*}

We then have the difference between $G$ and $\widetilde{G}$ . 

$$\mathbb{E}(\Vert G - \widetilde{G} \Vert^{2}_{op})\le \mathbb{E}(\Vert G - \widetilde{G} \Vert^{2}_{F})\le \frac{5n^{2}m^{2}}{p}$$

By Markov's inequality,

\begin{align}\label{eq:a.22}
\Vert G - \widetilde{G} \Vert^{2}_{op}\lesssim \frac{n^{2}m^{2}}{p}
\end{align}

with high probability.\\

Next, for any $s\in [n],a\in[m]$ ,$$\widetilde{G}_{ss}=\mathbb{E}(\psi(W^{T}Px_{s}))^2|X,P)=\frac{\Vert Px_{s} \Vert^{2}}{2k}.$$By \cite{Laurent2000AdaptiveEO} and a union bound argument, we have

\begin{align}\label{eq:a.23}
 max_{1\le s \le n} |\widetilde{G}_{ss}- \bar{G}_{ss}| \lesssim \sqrt{\frac{log mn}{k}}
\end{align}
with high probability.\\

Now we analyze the entries that are off-diagonal . Using the notation $\overline{Px}_{s}= \frac{\sqrt{k}Px_{s}}{\Vert Px_{s}\Vert}$, for any $ s \ne l$, we have

\begin{align}
&\widetilde{G}_{sl}\notag\\
&=\mathbb{E}(\psi(W^{T}\overline{Px}_{s})\psi(W^{T}\overline{Px}_{l})|X,P)\label{A.15}\\
&+\mathbb{E}((\psi(W^{T}P{x}_{s})-\psi(W^{T}\overline{Px}_{s}))\psi(W^{T}\overline{Px}_{l})|X,P)\label{A.16}\\
&+\mathbb{E}((\psi(W^{T}\overline{Px}_{s})(\psi(W^{T}P{x}_{l})-\psi(W^{T}\overline{Px}_{l}))|X,P)\label{A.17}\\
&+\mathbb{E}((\psi(W^{T}P{x}_{s})-\psi(W^{T}\overline{Px}_{s}))\times\notag\\
&(\psi(W^{T}P{x}_{l})-\psi(W^{T}\overline{Px}_{l}))|X,P)\label{A.18}
\end{align}

Since 
\begin{align*}
&cov(\psi(W^{T}\overline{Px}_{s})\psi(W^{T}\overline{Px}_{l}))\\&=\mathbb{E}[(\psi(W^{T}\overline{Px}_{s})-0)(\psi(W^{T}\overline{Px}_{l})-0)]\\&=\mathbb{E}(W^T\overline{Px}_{s}\overline{Px}_{l}^{T}W)\\
&=tr(\overline{Px}_{s}\overline{Px}_{l}^{T}\frac{I_{k}}{k})=\frac{\overline{Px}_{s}^{T}\overline{Px}_{l}}{k}=\rho
\end{align*}

and $W^{T}\overline{Px}_{s}, W^{T}\overline{Px}_{l}$ is equivalent to $\sqrt{1-\rho}U+\sqrt{\rho}Z,\sqrt{1-\rho}V+\sqrt{\rho}Z$ when $\rho \geq 0$ and similar argument for $ \rho < 0$ with $U,V,Z \stackrel{\text{iid}}{\sim} N(0,1)$.\\

Observing the first term on the right hand side of \ref{A.15}, we can have the fact that $E(\psi(W^{T}\overline{Px}_{s})\psi(W^{T}\overline{Px}_{l})|X,P)$ is a function of $\frac{\overline{Px}_{s}^T\overline{Px}_{l}}{k}$ , and thus we can write

\begin{align*}
E(\psi(W^{T}\overline{Px}_{s})\psi(W^{T}\overline{Px}_{l})|X,P)=f(\frac{\overline{Px}_{s}^T\overline{Px}_{l}}{k})
\end{align*}
where

\begin{align*}
&f(\rho)= \\
&\begin{cases}
\mathbb{E}\psi(\sqrt{1-\rho}U+\sqrt{\rho}Z)\psi(\sqrt{1-\rho}V+\sqrt{\rho}Z),\quad & \rho\geq 0 \\
\mathbb{E}\psi(\sqrt{1+\rho}U-\sqrt{-\rho}Z)\psi(\sqrt{1+\rho}V+\sqrt{-\rho}Z),\quad & \rho <0  
\end{cases}  
\end{align*}

Besides, we have $f(0)= \frac{1}{2\pi},f^{'}(0)=\frac{1}{4}$, and $sup_{{|\rho|}\leq 0.2} \frac{|f^{'}(\rho)-f^{'}(0)|}{|\rho|} \lesssim 1$. Therefore, as long as $|\rho| \leq \frac{1}{5}$,

$$|f(\rho)-\frac{1}{2\pi}-\frac{1}{4}\rho|\le C_{1}|\rho|^{2},$$
for some constant $C_{1}>0$. By \cite{Laurent2000AdaptiveEO}, we know that $max_{s \ne l} |\frac{\overline{Px}_{s}^{T}\overline{Px}_{l}}{k}| \lesssim \sqrt{\frac{logmn}{k}} \le \frac{1}{5}$ with high probability, which then implies 
$$\sum_{s\ne l}(\mathbb{E}(\psi(W^T\overline{Px}_{s})\psi(W^T\overline{Px}_{l})|X)-\bar{G}_{sl})^2 \le C_{1}\sum_{s\ne l}|\frac{\overline{Px}_{s}^{T}\overline{Px}_{l}}{k}|^4$$

The term on the right hand side has been analyzed before, and we have $\sum_{s\ne l}|\frac{\overline{Px}_{s}^{T}\overline{Px}_{l}}{k}|^4 \lesssim \frac{n^2}{k^2}$ with high probability.

We also need to analyze the contributions of \ref{A.16} and \ref{A.17}. Observe the fact that $\mathbb{I}\{W^T Px_{s}\ge 0\}= \mathbb{I}\{W^{T}\overline{Px}_{s}\ge 0\}$ , which implies 

\begin{align}
    &\psi(W^T P x_s) - \psi(W^T\overline{Px}_{s}) \notag\\&= W^T (Px_{s}-\overline{Px}_{s}) \mathbb{I}\{W^T\overline{Px}_{s} \ge 0\}  \notag\\&= (\frac{\Vert P x_{s}\Vert }{\sqrt{k}}-1) \psi(W^T\overline{Px}_{s})\label{eq:28}
\end{align}

Then, the sum of\ref{A.16} and \ref{A.17} can be written as 

$$
(\frac{\Vert Px_s \Vert}{\sqrt{k}} -1 + \frac{\Vert P x_l\Vert}{\sqrt{k}} -1 )f(\frac{\overline{Px}_{s}^{T}\overline{Px}_{l}}{k})
$$

Note that 

\begin{align*}
 &\sum_{s \ne l}(\frac{\Vert Px_s \Vert}{\sqrt{k}} -1 + \frac{\Vert P x_l\Vert}{\sqrt{k}} -1)^2[f(\frac{\overline{Px}_{s}^{T}\overline{Px}_{l}}{k})- \frac{1}{2\pi}]^2 \\&\lesssim \sum_{s \ne l}(\frac{\Vert P x_s \Vert}{\sqrt{k}} -1 + \frac{\Vert P x_l\Vert}{\sqrt{k}} -1)^4 + \sum_{s\ne l}| \frac{\overline{Px}_{s}^{T}\overline{Px}_{l}}{k}| ^4
\end{align*}

We have already shown that $\sum_{s\ne l}| \frac{\overline{Px}_{s}^{T}\overline{Px}_{l}}{k}| ^4 \lesssim \frac{n^2}{k^2}$ with high probability. By integrating out the probability tail bound of \cite{Laurent2000AdaptiveEO}, we have $\mathbb{E}(\frac{\Vert P x_s\Vert}{\sqrt{k}} -1 )^4 \lesssim k^{-2}$ , which then implies

$$\mathbb{E}(\frac{\Vert P x_s\Vert}{\sqrt{k}} -1 + \frac{\Vert P x_l\Vert}{\sqrt{k}} -1)^4 \lesssim \frac{n^2}{k^2}$$
and the corresponding high-probability bound by Markov's inequality.

Finally, we show that the contribution of \ref{A.18} is negligible. By \ref{eq:28}, we can write \ref{A.18}  as

$$(\frac{\Vert P x_s\Vert}{\sqrt{k}} -1)(\frac{\Vert P x_l\Vert}{\sqrt{k}} -1)\mathbb{E}(\psi(W^{T}\overline{Px}_{s})^2 \psi(W^{T}\overline{Px}_{l})^2 |X),$$
whose absolute value can be bounded by $\frac{3}{2}|\frac{\Vert P x_s\Vert}{\sqrt{k}} -1||\frac{\Vert P x_l\Vert}{\sqrt{k}} -1|$. Since
$$
\sum_{s\ne l}\mathbb{E}(\frac{\Vert P x_s\Vert}{\sqrt{k}} -1)^2\mathbb{E}(\frac{\Vert P x_l\Vert}{\sqrt{k}} -1)^2 \lesssim \frac{n^2}{k^2},
$$
we can conclude that \ref{A.18} is bounded by $O(\frac{n^2}{k^2})$ with high probability by Markov's inequality.\

Combining the analyses of \ref{A.15} ,\ref{A.16} , \ref{A.17} and \ref{A.18}, we conclude that $\sum_{s \ne l}(\bar{G}_{sl}-\widetilde{G}_{sl})^2 \lesssim \frac{n^2}{k^2}$ with high probability. Together with \ref{eq:a.22} and \ref{eq:a.23}, we obtain the desired bound for $\Vert G-\bar{G} \Vert_{op}$.

To prove the last conclusion $\Vert \bar{G}\vert \lesssim mn$, it suffices to analyze $\lambda_{max}(\bar{G})$. We bound this quantity by $\mathbb{E} \lambda_{max}(\bar{G})^2 \le \mathbb{E}\vert G \vert _{F}^2 \lesssim m^2n^2$, which leads to the desired conclusion.

\emph{Proof of Lemma \ref{4.5}.}
 Since $\hat{\beta}$ belongs to the row space of $\sum_{i=1}^{m}  \psi(W_{j}^{T}(t-1)P_{i} \mathbf{x}_s)$ %and $\delta_s$ only depends on $\mathbf{x}_s,$
 there exists some $v^* \in \mathbb{R}^{n}$ such that $\widehat{\beta}=A_{\beta}^T v^*$. By  Lemma  \ref{4.1} and Lemma \ref{4.4}, we know that $\widetilde{v}=v^*$ with high probability, and therefore $\widetilde{\beta}={A_{\beta}}^T \widetilde{v}={A_{\beta}}^T v^*=\widehat{\beta}$.

\subsection{Additional proof of Theorem 1.} 
\begin{proof}
We need the following kernel random matrix result to prove theorem 1.

\begin{lemma}\label{A.7.}
Consider independent parameters $\beta_1,...,\beta_p \sim N(0,1)$. We define the matrices $H,\bar{H}$ by 

$$H_{s l} = \frac{(P x_s)^T P x_l}{k}\frac{1}{p}\sum_{j=1}^p\beta_j^2 \mathbb{I}\{W_j^T P x_s \geq 0, W_j^T P x_l \geq 0\}$$
$$\bar{H} _{s l} = \frac{1}{4}\frac{(P x_s)^T P x_l}{||P x_s|| ||P x_l||} + \frac{1}{4} \mathbb{I}\{s=l\}$$
where $Px_{s} = \sum_{a=1}^m P_ax_s$ and $Px_{l} = \sum_{b=1}^m P_bx_l.$

Assume $k/ log mn$ is sufficiently large, and then

$$||H-\bar{H}||^2_{op}\lesssim \frac{m^2n^2}{pk}+\frac{m^2n}{p}+\frac{log mn}{k}+\frac{n^2}{k^2}$$
with high probability. If we additionally assume that $\frac{k}{n}$ and $\frac{p}{m^2 n}$ are sufficiently large, we will also have

$$\frac{1}{5} \leq \lambda_{min}(H) \leq \lambda_{max}(H) \lesssim 1$$

with high probability.

\end{lemma} 

\emph{ Proof of lemma \ref{A.7.}.}
    Define $\tilde{H} \in \mathbb{R}^{n \times n}$ with entries $\tilde{H} _{s l}= \frac{(P x_s)^T(P x_l)}{k}\mathbb{E}(\beta^2\mathbb{I}\{W^T P x_s \geq 0 , W^T P x_l\geq 0\}|X)$  and we first bound the difference between $H$ and $\tilde{H}$. Note that

$$\mathbb{E}(H_{s l} - \tilde{H}_{s l})^2 = \mathbb{E} Var(H_{s l}|X) $$ 
$$ \leq \frac{1}{p} \mathbb{E}(\frac{|(P x_s)^T (P x_l)|^2}{k^2} \beta^4 ) \leq  \begin{cases}
      \frac{3m^2}{pk} & \text{s $\neq$ l}\\
      \frac{9m^2}{p} & \text{s $=$ l }\\
      \end{cases} $$

We then have

$$\mathbb{E}||H - \tilde{H}||_{op}^2 \leq \mathbb{E}||H - \tilde{H}||_{F}^2 \leq \frac{3 (mn)^2}{pk} + \frac{9 m^2n}{p}$$

By Markov's inequality

$$||H - \tilde{H}||_{op}^2 \lesssim \frac{(mn)^2}{pk} + \frac{m^2n}{p}$$
with high probability

Next, we study the diagonal entries of $\tilde{H}$. For any $s \in [n]$ and $a \in [m]$, $\tilde{H}_{s s} = \frac{\Vert P x_s\Vert^2}{k}\mathbb{E}(\beta^2\mathbb{I}\{W^T P x_s \geq 0\}|X) = \frac{\Vert P x_s\Vert^2}{2k}$. The same analysis that leads to the bound \eqref{eq:a.23} also implies that

$$\underset{1 \leq i \leq n}{max}|\tilde{H} - \bar{H}| \lesssim \sqrt{\frac{log mn}{k}}$$
with high probability.

Now we analyze the off-diagonal entries. Recall the notation $\overline{P x}_s = \frac{\sqrt{k} P x_s}{||P x_s||}$. For any $s \neq l$, we have $$\tilde{H}_{s l} = \frac{||P x_s|| ||P x_l||}{k}\frac{\overline{P x}_s^T \overline{P x}_l}{k} \mathbb{P}(W^T \overline{P x}_s \geq 0, W^T \overline{P x}_l\geq 0| X)$$
$$= \frac{\overline{P x}_s^T \overline{P x}_l}{k} \mathbb{P}(W^T \overline{P x}_s \geq 0, W^T \overline{P x}_l\geq 0| X) $$
$$+ ( \frac{||P x_s|| ||P x_l||}{k} - 1) \frac{\overline{P x}_s^T \overline{P x}_l}{k} \mathbb{P}(W^T \overline{P x}_s \geq 0, W^T \overline{P x}_l\geq 0| X)$$

Since $\mathbb{P}(W^T \overline{P x}_s \geq 0, W^T \overline{P x}_l \geq 0| X)$ is a function of $\frac{\overline{P x}_s^T \overline{P x}_l}{k}$, we can write 

$$\frac{\overline{P x}_s^T \overline{P x}_l}{k} \mathbb{P}(W^T\overline{P x}_s \geq 0, W^T \overline{P x}_l \geq 0| X) = f(\frac{\overline{P x}_s^T \overline{P x}_l}{k})$$
where for $\rho \geq 0$

$$f(\rho) = \rho \mathbb{P}(\sqrt{1 - \rho}U + \sqrt{\rho}Z \geq 0, \sqrt{1 - \rho}V + \sqrt{\rho}Z \geq 0)$$
$$= \rho \mathbb{E}\mathbb{P}(\sqrt{1 - \rho}U + \sqrt{\rho}Z \geq 0, \sqrt{1 - \rho}V + \sqrt{\rho}Z \geq 0|Z)$$
$$= \rho \mathbb{E} \Phi(\sqrt{\frac{\rho}{1-\rho}}Z)^2$$
with $U,V,Z \stackrel{\text{iid}}{\sim} N(0,1)$ and $\Phi$  being the cumulative distribution function $N(0,1)$. Similarly, for $\rho < 0$,
$$f(\rho) = \rho\mathbb{E}[\Phi(\sqrt{\frac{\rho}{1-\rho}}Z)(1-\Phi(\sqrt{\frac{\rho}{1-\rho}}Z))]$$

By some direct calculations, we have $f(0) = 0$, $f'(0) = \frac{1}{4}$ , and

$$\underset{|\rho|\leq \frac{1}{5}}{sup}|f''(\rho)|\lesssim \underset{|t|\leq \frac{1}{2}}{sup}|\mathbb{E}\phi(tZ)\Phi(tZ)Z/t|+\underset{|t|\leq \frac{1}{2}}{sup}|\mathbb{E}\phi(tZ)Z/t|$$
where $\phi(x) = (2\pi)^{-1/2}e^{-x^2/2}$. For any $|t| \leq 1/2$,

$$|\mathbb{E}\phi(tZ)Z/t| = |\mathbb{E}\frac{\phi(tZ)-\phi(0)}{tZ}Z^2| = |\mathbb{E}\xi\phi(\xi)Z^2|\leq \frac{|t|}{\sqrt{2\pi}}\mathbb{E}|Z|^3\lesssim 1$$
where $\xi$ is a scalar between 0 and $tZ$ so that $|\xi|\leq |tZ|$. By a similar argument, we also have $\underset{|t|\leq \frac{1}{2}}{sup}|\mathbb{E}\phi(tZ)\Phi(tZ)Z/t|\lesssim 1$ so that $\underset{|\rho|\leq \frac{1}{5}}{sup}|f''(\rho)|\lesssim 1$. Therefore, as long as $\frac{\overline{P x}_s^T \overline{P x}_l}{k}\leq 1/5$,

$$|f(\frac{\overline{P x}_s^T \overline{P x}_l}{k})-\frac{1}{4}\frac{\overline{P x}_s^T \overline{P x}_l}{k}|\leq C_1|\frac{\overline{P x}_s^T \overline{P x}_l}{k}|^2$$
for some constant $C_1 > 0$. And we know that $max_{s \neq l} \frac{\overline{P x}_s^T\overline{P x}_l}{k}\lesssim \sqrt{\frac{log mn}{k}} \leq 1/5$ with high probability. We then have the high probability bound,

\begin{align}
    &\sum_{s \neq l }(\tilde{H_{s l}} - \frac{1}{4}\frac{\overline{P x}_s^T \overline{P x}_l}{k})^2 \leq 2\sum_{s \neq l} (\frac{||\overline{P x}_s|| ||\overline{P x}_l||}{k}-1)^2|\frac{\overline{P x}_s^T \overline{P x}_l}{k}|^2 \notag\\
    &+ 2C_1\sum_{s \neq l}|\frac{\overline{P x}_s^T \overline{P x}_l}{k}|^4 
   \leq \sum_{s \neq l}(\frac{||\overline{P x}_s|| ||\overline{P x}_l||}{k}-1)^4\notag\\ &+(2C+1)\sum_{s \neq l}|\frac{\overline{P x}_s^T \overline{P x}_l}{k}|^4\label{A.19}
\end{align}

For the first term on the right hand side of \ref{A.19}, we use  a probability tail bound. By integrating out this tail bound, we have

$$\sum_{s \neq l}\mathbb{E}(\frac{||\overline{P x}_s|| ||\overline{P x}_l||}{k}-1)^4 \lesssim \frac{n^2}{k^2}$$
which, by Markov's inequality, implies $\sum_{s \neq l}(\frac{||\overline{P x}_s|| ||\overline{P x}_l||}{k}-1)^4 \lesssim \frac{n^2}{k^2}$ with high probability. Using the same argument in the proof of lemma \ref{lemma a.6}, we have $\sum_{s \neq l }|\frac{\overline{P x}_s^T \overline{P x}_l}{k}|^4$ with high probability.
Finally, combining the bounds , we obtain the desired bound for $||H-\bar{H}||_{op}$. 
The proof is complete.\\

Return to the analysis of $G(k)$ in the main paper and continue with the detailed proof.

%Then analyze $f_s(k+1)-f_s(k)$. For each $i \in[n]$, we have

%$$
%\begin{aligned}
%& f_s(k+1)-f_s(k) = \frac{1}{\sqrt{p}} \sum_{j=1}^p \beta_j(k+1)\times\\
%&\sum_{a=1}^m\left(\psi\left(W_j(k+1)^T P_ax_s\right)-\psi\left(W_j(k)^T P_ax_s\right)\right) \\
%

To analyze $G(k)$, we first bound the distance between $G(k)$ and $G(0)$. Since
$$
\begin{aligned}
&\left|G_{s l}(k)-G_{s l}(0)\right| \\&\leq  \frac{1}{p} \sum_{j=1}^p\sum_{b=1}^{m}\left|\psi\left(W_j(k)^TP_b x_l\right)-\psi\left(W_j(0)^T P_b x_l\right)\right| \\
&+\frac{1}{p} \sum_{j=1}^p \sum_{a=1}^m\left|\psi\left(W_j(k)^T P_a x_s\right)-\psi\left(W_j(0)^T P_a x_s\right)\right| \\
&\leq \frac{1}{p} \sum_{j=1}^p\sum_{b=1}^{m}\left|\left(W_j(k)-W_j(0)\right)^T P_b x_l\right| \\
&+\frac{1}{p} \sum_{j=1}^p \sum_{a=1}^m \left|\left(W_j(k)-W_j(0)\right)^T P_a x_s\right| \\
&\leq  R_W\left(\left\|Px_l\right\|+\left\|Px_s\right\|\right)
\end{aligned}
$$

then, by $\max _{1 \leq s \leq n}\left\|Px_s\right\| \lesssim \sqrt{mk}$,

\begin{align*}
   & \|G(k)-G(0)\|_{\mathrm{op}} \leq \max _{1 \leq l \leq n} \sum_{s=1}^n\left|G_{s l}(k)-G_{s l}(0)\right| \\&\leq 2 R_W mn \max _{1 \leq s \leq n}\left\|Px_s\right\| \lesssim \frac{(mn)^2 \log p}{\sqrt{p}}
\end{align*}

By lemma \ref{lemma a.6} and the fact that G(k) is positive semi-definite, we have
\begin{align}
0\le \lambda_{min}(G(k)) \le \lambda_{max}(G(k)) \lesssim m n.\label{A.28}
\end{align}

We also need to bound the distance between H(k) and H(0). We have

\begin{align}
&|H_{sl}(k) - H_{sl}(0)| \\&\le |\frac{(Px_s)^{T}{Px_{l}}}{k}|\frac{1}{p}\sum_{j=1}^{p} |\beta_{j}(k+1)^2 - \beta_{j}^2(0)|\label{A.29}\\
&+|\frac{(Px_s)^{T}{Px_{l}}}{k}|\frac{1}{p}\sum_{j=1}^{p} \sum_{a=1}^m \beta_{j}(0)^2 |\psi^{'}(W_{j}(k)^{T}P_a x_{s})\notag\\
&- \psi^{'}(W_{j}(0)^{T}P_a x_{s})|\label{A.30}\\
&+|\frac{(Px_s)^{T}{Px_{l}}}{k}|\frac{1}{p}\sum_{j=1}^{p} \sum_{b=1}^{m} \beta_{j}(0)^2 |\psi^{'}(W_{j}(k)^{T}P_b x_{l})\notag\\&- \psi^{'}(W_{j}(0)^{T}P_b x_{l})|\label{A.31}
\end{align}

We can bound \ref{A.29} by $|\frac{(Px_s)^{T}{Px_{l}}}{k}|\frac{1}{p}\sum_{j=1}^{p}R_{\beta}(R_{\beta}+2|\beta_{j}(0)|)$.To bound \ref{A.30}, we note that
\begin{align}
&| \sum_{a=1}^m (\psi^{'}(W_{j}(k)^{T}P_a x_{s})-\psi^{'}(W_{j}(0)^{T}P_a x_{s}))| \notag\\&\le \mathbb{I}\{|W_{j}(0)^{T}Px_{s}|\le |(W_{j}(k)- W_{j}(0))^{T}Px_{s}|\} \notag\\
&\le \mathbb{I}\{|W_{j}(0)^{T}Px_{s}|\le R_{W} \Vert Px_{s} \Vert\}\label{A.32}
\end{align}
which implies

$$
|\frac{{(Px_s)}^{T}{Px_{l}}}{k}|\frac{1}{p}\sum_{j=1}^{p}\sum_{b=1}^{m} \beta_{j}(0)^2 |\psi^{'}(W_{j}(k)^{T}P_b x_{l})- \psi^{'}(W_{j}(0)^{T}P_b x_{l})|
$$

$$
\le |\frac{{(Px_s)}^{T}{Px_{l}}}{k}|\frac{1}{p}\sum_{j=1}^{p} \beta_{j}(0)^2 \mathbb{I}\{|W_{j}(0)^{T}Px_{s}|\le R_{W} \Vert Px_{s} \Vert \}
$$
and a similar bound holds for \ref{A.31}. Then,
$$
\begin{aligned}
&\|H(k)-H(0)\|_{\mathrm{op}} \\&\leq  \max _{1 \leq s \leq n}\left|H_{ss}(k)-H_{ss}(0)\right|+\max _{1 \leq l \leq n} \sum_{s \in[n] \backslash\{l\}}\left|H_{s l}(k)-H_{s l}(0)\right| \\
&\lesssim  \max _{1 \leq s \leq n} \frac{1}{p} \sum_{j=1}^p \beta_j^2(0) \mathbb{I}\left\{\left|W_j^T(0)^T Px_s\right| \leq R_W \left\|x_s\right\|\right\} \\
& +k^{-1 / 2} n \max _{1 \leq s \leq n} \frac{1}{p} \sum_{j=1}^p \beta_j^2(0) \mathbb{I}\left\{\left|W_j^T(0)^TP x_s\right| \leq R_W \left\|x_s\right\|\right\} \\
& +\max _{1 \leq l \leq n} \sum_{s=1}^n\left|\frac{(Px_s)^T Px_l}{k}\right| R_{\beta} \frac{1}{p} \sum_{j=1}^p\left(R_{\beta}+2\left|\beta_j(0)\right|\right) \\
&\lesssim  \left(m+\frac{n}{\sqrt{k}}\right)(\sqrt{mk}  R_W \log p+\frac{\sqrt{\log (mn)} \log p}{\sqrt{p}}+R_{\beta}^2+R_{\beta} \sqrt{\log p}) \\
&\lesssim  \left(m+\frac{n}{\sqrt{k}}\right) \frac{mn(\log p)^2}{\sqrt{p}},
\end{aligned}
$$
where we have used $$\max _{1 \leq j \leq p}\left|\beta_j(0)\right| \leq 2 \sqrt{\log p} , \max _{1 \leq s \leq n}\left\|Px_s\right\|  \lesssim \sqrt{mk}$$, $$\max _{1 \leq s \neq l \leq n,1 \leq a \neq b \leq m}\left|\frac{(P_ax_s)^T P_bx_l}{k}\right|\lesssim k^{-1 / 2}$$ $$\max _{1 \leq l \leq n} \sum_{s=1}^n\left|\frac{(Px_s)^T Px_l}{k}\right|  \lesssim m+\frac{n}{\sqrt{k}}$$ and 
\begin{align*}
    &\max _{1\le s\le n}\frac{1}{p} \sum_{j=1}^p \mathbb{I}\left\{\left|W_j(0)^T Px_s\right| \leq R_W\left\|Px_s\right\|\right\}  \\&\lesssim \sqrt{mk} R_W+\sqrt{\frac{\log (mn)}{p}}
\end{align*}
In view of Lemma \ref{A.7.}, we then have

\begin{align}
\frac{1}{6} \leq \lambda_{min}(H(k)) \leq \lambda_{max}(H(k)) \lesssim 1 \label{A.33}
\end{align}
under the conditions of $k , p , m$ and $n$.\\

Next, we give a bound for $r_{s}(k)$. Observe that
\begin{align*}
   &(\psi(W_{j}(k+1)^{T}Px_s)-\psi(W_{j}(k)^{T}Px_{s}))\\&= (W_{j}(k+1)-W_{j}(k))^{T}Px_{s}\psi^{'}(W_{j}(k)^{T}Px_{s}),
\end{align*}
when $\mathbb{I}\{W_{j}(k+1)^{T}Px_s>0 \} = \mathbb{I}\{W_{j}(k)^{T}Px_s>0\}.$ Thus, we only need to sum over those $j\in [p]$ that $\mathbb{I}\{W_{j}(k+1)^{T}Px_s>0\} \ne \mathbb{I}\{W_{j}(k)^{T}Px_s>0\}$. By \ref{A.32}, we have
$$
\begin{aligned}
& \left|\mathbb{I}\left\{W_j(k+1)^T Px_s>0\right\}\mathbb{I}\left\{W_j(k)^T Px_s>0\right\}\right| \\
&\leq  \left|\mathbb{I}\left\{W_j(k+1)^T Px_s>0\right\}-\mathbb{I}\left\{W_j(0)^T Px_s>0\right\}\right|\\&+\left|\mathbb{I}\left\{W_j(k)^TP x_s>0\right\}-\mathbb{I}\left\{W_j(0)^T Px_s>0\right\}\right| \\
& \leq 2 \mathbb{I}\left\{\left|W_j^T(0)^TP x_s\right| \leq R_W\left\|Px_s\right\|\right\} .
\end{aligned}
$$
Therefore,

\begin{align*}
& \lvert \psi\left(W_j(k+1)^TP x_s\right)-\psi\left(W_j(k)^T Px_s\right)\\
&-\left(W_j(k+1)-W_j(k)\right)^T Px_s \psi^{\prime}\left(W_j(k)^T Px_s\right)\rvert\\
& \leq 4\left|\left(W_j(k+1)-W_j(k)\right)^T Px_s\right| \times \\
&\mathbb{I}\left\{\left|W_j^T(0)^T Px_s\right| \leq R_W\left\|Px_s\right\|\right\} \\
& \le\frac{4\lambda}{k\sqrt{p}} \big| \beta_{j}(k+1) \big| \|y-f_s(k) \| \|Px_{s}\|\times\\
&\sqrt{\sum_{l}\Vert Px_{l}\Vert^2} \mathbb{I}\{|W_{j}(0)^{T}Px_s|\le R_{W}\Vert Px_{s} \Vert \}\\
\end{align*}
which implies
\begin{align*}
& \lvert r_{s}(k) \lvert\le \frac{4\lambda}{kp} \sum_{j=1}^{p} \rvert \beta_{j}(k+1) \rvert ^2 \|y-f_s(k) \| \|Px_{s}\|\times\\
&\sqrt{\sum_{l}\Vert Px_{l}\Vert^2} \mathbb{I}\{|W_{j}(0)^{T}Px_s|\le R_{W}\Vert Px_{s} \Vert \}\\
&\lesssim m\sqrt{n}\Vert y-f_s(k) \Vert \gamma \frac{1}{p} \sum_{j=1}^{p}(\beta_{j}(0)^2+R_{\beta}^2)\times\\
&\mathbb{I}\{ |W_{j}(0)^{T}Px_s|\le R_{W}\Vert Px_{s} \Vert\}\\
&\lesssim \gamma m\sqrt{n}logp(\sqrt{mk}R_{W}+\sqrt{\frac{log m n}{p}})\Vert y-f_s(k) \Vert \\
\end{align*}

This leads to the bound

\begin{align}
&\Vert r(k) \Vert=\sqrt{\sum_{s}\big| r_{s}(k)\big|^2}\notag\\
&\lesssim \gamma mn logp(\sqrt{mk}R_{W}+\sqrt{\frac{log mn}{p}})\Vert y-f_s(k) \Vert \label{A.34}
\end{align}

The last part is the analysis of $\Vert y-f_s(k+1)\Vert^2$ in the main paper.

\end{proof}

\subsection{Additional proof of Theorem 2} 
\begin{proof}

Consider $\eta = b + Av^* + z \in \mathbb{R}^k$, where the noise vector $z$ satisfies

$$
z_i \sim(1-\varepsilon) \delta_0+\varepsilon Q_i,
$$
independently for all $i \in[\mathrm{m}]$. And $b \in \mathbb{R}^k$ is an arbitrary bias vector. Then, the estimator $\hat{v} = \mathop{argmin}\limits_{v \in \mathbb{R}^{n}}|| \eta - Av||_{1}$ satisfies the following theoretical guarantee.

\begin{lemma}\label{Lemma A.8.} Assume the design matrix A satisfies \ref{sublemma: condition_a} , \ref{sublemma: condition_b} and \ref{sublemma: condition_c}. Then, as long as $\frac{\overline{\lambda}\sqrt{\frac{mn}{k}{\log}(\frac{ek}{mn})} + \epsilon \sigma \sqrt{\frac{mn}{k}}}{{\lambda}(1-\epsilon)}$
is sufficiently small and $\frac{8\frac{1}{k}\sum_{i=1}^k|b_i|}{{\lambda}(1-\epsilon)} < 1$, we have
$$\Vert \hat{v} - v^*\Vert \leq \frac{4\frac{1}{k}\sum_{i=1}^k\lvert b_i\rvert }{{\lambda}(1-\epsilon)}$$
with high probability.
\end{lemma}

We first analyze $\hat{u_1},..,\hat{u_p}$. The idea is to apply the result of lemma \ref{4.1} to each of the $p$ robust regression problems. Thus, it suffices to check if the conditions of lemma  \ref{4.1} hold for the $p$ regression problems simultaneously. 

\end{proof}

\vfill
\end{document}